\pdfoutput=1

\documentclass{article}

\usepackage[preprint]{neurips_2026}

\usepackage[utf8]{inputenc} 
\usepackage[T1]{fontenc}    

\usepackage{amsthm}
\usepackage{amsmath}
\usepackage{amsfonts}       
\usepackage{amssymb}
\usepackage{booktabs}       
\usepackage{enumitem}
\usepackage{graphicx}
\usepackage{hyperref}       
\hypersetup{hidelinks}
\usepackage{microtype}      
\usepackage{multirow}
\usepackage{natbib}
\usepackage{nicefrac}       
\usepackage{subcaption}
\usepackage{url}            
\usepackage[table]{xcolor}
\usepackage{makecell}

\usepackage{thmtools}
\usepackage{thm-restate}

\newtheorem{definition}{Definition}

\newtheorem{corollary}{Corollary}
\newtheorem{proposition}{Proposition}
\newtheorem{remark}{Remark}

\DeclareMathOperator*{\argmax}{arg\,max}

\title{Jacobian-Guided Anisotropic Noise Reshaping for Enhancing Representation Utility under \\Local Differential Privacy}

\author{
  Youngmok Ha$^{1,2}$\thanks{Corresponding author.}\hspace{0.4em}\thanks{This work was conducted at Imperial College London during the author's academic leave from ETRI for doctoral studies.}
  ~~~Viktor Schlegel$^{1,3,4}$~ Yidan Sun$^{3}$~ Anil Anthony Bharath$^{1,3}$ \\
  $^1$Department of Bioengineering, Imperial College London, United Kingdom \\
  $^2$Artificial Intelligence Computing Research Laboratory, \\ Electronics and Telecommunications Research Institute (ETRI), Daejeon, Republic of Korea\\
  $^3$Imperial Global Singapore, Imperial College London, Singapore\\
  $^4$Department of Computer Science, University of Manchester, United Kingdom\\
  \texttt{\{y.ha25, v.schlegel, y.sun1, a.bharath\}@imperial.ac.uk}
}

\begin{document}

\maketitle

\begin{abstract}
While Local Differential Privacy (LDP) serves as a foundational primitive for distributed data collection, its stringent randomization requirements often lead to severe degradation in data representation utility.
This degradation stems from the task-agnostic nature of conventional LDP mechanisms, which perturb all dimensions without accounting for their relative importance to the downstream objective.
To address this issue, we propose a novel approach that mitigates noise in task-relevant subspaces of the data representation.
Our method identifies task-critical subspaces via the Jacobian of a public downstream model, selectively attenuates noise along these directions, and reshapes the isotropic noise of standard LDP mechanisms into an anisotropic distribution.
The resulting mechanism preserves the privacy guarantee of the underlying LDP randomizer while heterogeneously modulating the impact of noise across task directions, thereby substantially enhancing data utility.
The approach is applicable to both linear and nonlinear models and can be seamlessly integrated with existing LDP mechanisms.
Extensive experiments on CIFAR-10-C under brightness corruption at the highest severity level demonstrate that integrating our approach improves classification accuracy by approximately 8 percentage points for Laplace and 20 percentage points for PrivUnit variants at $\epsilon=7.5$.
The source code is available at \url{https://github.com/ymha/jacobian-anr-ldp}.
\end{abstract}

\section{Introduction}

Local Differential Privacy (LDP)~\citep{kasiviswanathan2011what} has become a foundational primitive for privacy-preserving distributed data collection. 
Unlike centralized approaches~\citep{dwork2006calibrating}, LDP eliminates the need for a trusted data curator by enabling each data owner to locally randomize their data prior to sharing. 
Owing to this decentralized trust model, it has seen widespread adoption in practice. 
Prominent examples include Google's RAPPOR for collecting statistics from Chrome clients~\citep{erlingsson2014rappor}, 
Apple's deployment in iOS for collecting new word suggestions, emoji usage frequencies, and health data statistics~\citep{apple2017learning}, 
and Microsoft's integration into Windows telemetry for application usage tracking~\citep{ding2017collecting}.

A well-known drawback of LDP is the severe degradation of data represenation utility due to its heavy randomization~\citep{duchi2013local,kairouz2016discrete}.
To provide rigorous privacy guarantees, standard LDP mechanisms enforce a pessimistic randomization process.
Specifically, since the privacy guarantee must hold for any arbitrary points over the entire private data space, the magnitude of randomness is governed by the worst-case sensitivity. 
Although this prevents inferring any individual's true input, it significantly lowers the signal-to-noise ratio and distorts the original characteristics of the data. 
As a result, downstream tasks that require fine-grained feature preservation remain challenging under the standard LDP paradigm~\citep{duchi2018minimax,duan2022utility}.

To mitigate this privacy--utility trade-off, a substantial body of literature has explored various directions. For example, prior work has focused on optimizing privacy mechanisms~\citep{duchi2013local,duchi2018minimax,wang2019multidimensional,bhowmick2018privunit,asi2022privunitg}. 
To allocate privacy budgets and noise more effectively, researchers have also proposed task-aware~\citep{cheng2022task}, relevance-aware~\citep{phan2017adaptive}, coordinate-wise~\citep{alaggan2016heterogeneous,muthukrishnan2025DP}, and correlation-aware approaches~\citep{kifer2014pufferfish,aumuller2024plan,dagan2024dimensionfree}. 
Finally, another line of work has investigated post-processing methods~\citep{erlingsson2014rappor,apple2017learning,ding2017collecting,hay2010boosting,wang2017locally,jia2019calibrate,cormode2019answering,wang2020locally,sajadmanesh2021locally,fang2023convolution}, as well as the incorporation of public data, models, and prior knowledge~\citep{kurakin2022toward,nasr2023effectively,hou2024pretext}, to improve utility while preserving privacy.

However, enhancing utility through the analysis of downstream models remains underexplored. 
Although Cheng et al.~\citep{cheng2022task} investigated this direction, their methodology is predominantly centered around linear models and requires the direct use of private data. 
Moreover, it does not fully utilize geometric principles. 
For example, the data representation space can be partitioned into a task-relevant row space and a task-irrelevant null space. 
These subspaces, which are distinct from the subspace reflecting the primary data correlations~\citep{aumuller2024plan}, dictate the directions in which the model is sensitive or robust to perturbations. 
While identifying and exploiting these subspaces can enhance utility, the existing LDP literature has largely overlooked these structural properties.

In this paper, we propose a novel approach that enhances data representation utility under LDP by reducing the amount of noise injected into task-critical subspaces defined by public downstream models. 
Our approach is motivated by the geometric insight that Jacobian-based analysis can identify task-sensitive and task-insensitive subspaces, thereby enabling fine-grained control over the noise injected into each subspace. 
We further exploit the observation that, under an isotropically calibrated base randomizer, sensitivity modulation can effectively reduce the noise scale applied to selected subspaces. 
Building on these principles, our approach transforms the standard isotropic noise introduced by the base randomizer into an anisotropic noise distribution.

To the best of our knowledge, this work is the first to investigate improving representation utility under LDP through Jacobian-based identification of task-relevant subspaces and anisotropic noise reshaping. 
The main advantages of our approach are as follows: 
(i) it reduces the noise injected into task-critical subspaces, thereby substantially improving representation utility; 
(ii) it enables mechanisms satisfying pure LDP to outperform approximate mechanisms in terms of utility; 
(iii) it is applicable to both linear and nonlinear downstream models; 
(iv) it can be seamlessly integrated with advanced privacy mechanisms by operating as a pre- and post-processing wrapper around a base mechanism; 
and (v) it preserves the privacy guarantee of the base mechanism without incurring any additional privacy cost, as its transformations are determined solely by public information.

\section{Related Work}
\textbf{Task-Aware LDP.}
The work by Cheng et al.~\citep{cheng2022task} is most closely related to ours in that it utilizes task model information and serves as a wrapper compatible with other mechanisms: they propose an encoding-decoding framework that exploits downstream model information to apply LDP to high-dimensional data, injecting noise between the encoding and decoding steps.
The most significant difference emerges when dealing with nonlinear downstream models. 
For these models, they propose training the encoder and decoder using unprotected data. 
The encoder and decoder are trained by minimizing the discrepancy between the downstream model's outputs when fed with protected versus unprotected private data. 
In contrast, we utilize a publicly available Jacobian matrix to identify the local row and null spaces. 
By relying on this public information, we avoid any risk of privacy leakage.

\textbf{Coordinate-Wise DP.}
Our work is also aligned with \citep{muthukrishnan2025DP}, which investigates the determination of non-identical noise scales for independent coordinates with varying sensitivities under DP. 
Their primary objective is to minimize the MSE between the original and privatized data. 
In contrast, our method prioritizes downstream utility over data MSE (DMSE) minimization. Specifically, our objective is to minimize a first-order surrogate for the risk caused by noise injection. 
These two objectives are often misaligned; for instance, injecting substantial noise into the null space may significantly degrade the DMSE while leaving downstream performance unaffected. 
Another key difference is that their method does not account for the task geometry. 
Finally, while they directly inject anisotropic noise into the private data, we inject isotropic noise, which is then distributed anisotropically through reshaping.

\textbf{Variance-Aware DP.} 
In terms of utilizing prior information, PLAN \citep{aumuller2024plan} shares similarities with our work. 
Assuming a known covariance and data sampled from distributions with bounded coordinate-wise standard deviations, PLAN calibrates noise by disproportionately allocating the privacy budget to coordinates exhibiting higher variance, guided by the data geometry inherent in the covariance matrix. 
In contrast, our approach determines noise scales based on the task-specific importance of subspaces utilized by the downstream model, rather than relying on data correlations.

Table~\ref{tab:comparison} summarizes these comparisons.
We provide discussions on relevance-aware and feature-wise approaches, alongside post-processing, and DP assisted by public data and models in Appendix~\ref{appendix:1}.

\begin{table}[t]
  \caption{Comparison between our approach and prior methods for nonlinear models.}
  \label{tab:comparison}
  \centering
  \resizebox{\textwidth}{!}{
  \begin{tabular}{@{}lllccc@{}}
    \toprule
    \textbf{Method} & \textbf{Focus}    & \textbf{Objective}             & \textbf{Use Private Data} & \textbf{Use Jacobian Guidance} & \textbf{Mechanism-Agnostic} \\
    \midrule
    Task-Aware~\citep{cheng2022task} & Task model        & Prediction matching            & Yes                       & No  & Yes \\
    CW~\cite{muthukrishnan2025DP}         & Query             & Data-space MSE                 & No                        & No  & No  \\
    PLAN~\citep{aumuller2024plan}       & Data distribution & Mean-estimation MSE            & No                        & No  & No  \\
    \textbf{Our Approach}   & Task model        & First-order noise risk         & No                        & Yes & Yes \\
    \bottomrule
  \end{tabular}
  }
\end{table}

\section{Preliminaries} \label{sec:pre}
\subsection{Jacobian-Guided Subspace Identification}\label{sec:row_null}
Suppose that we have an $m$-dimensional data representation $\boldsymbol{z} \in \mathbb{R}^m$ and a linear transformation matrix $\boldsymbol{W} \in \mathbb{R}^{k \times m}$, where $m, k \in \mathbb{Z}^+$.
The row space and null space of $\boldsymbol{W}$ are two subspaces that form an orthogonal decomposition of $\mathbb{R}^m$. 
The row space, $\text{Row}(\boldsymbol{W})$, is spanned by the row vectors of $\boldsymbol{W}$ and captures the directions along which the transformation has a non-trivial effect. 
An orthonormal basis for $\text{Row}(\boldsymbol{W})$, $\boldsymbol{u}_r$, is derived from the Jacobian, $\boldsymbol{J}_{\mathcal{T}}(\boldsymbol{z}) = \partial \mathcal{T}(\boldsymbol{z})/\partial \boldsymbol{z} = \boldsymbol{W}$, which is constant and equals the weight matrix. 
In contrast, the null space, $\text{Null}(\boldsymbol{W}) = \{\boldsymbol{z} \in \mathbb{R}^m \mid \boldsymbol{W}\boldsymbol{z} = \mathbf{0}\}$, consists of all vectors annihilated by the transformation. 
A basis for $\text{Null}(\boldsymbol{W})$, $\boldsymbol{u}_n$, can be found by computing an orthonormal set of vectors that span the orthogonal complement of the row space.
These subspaces characterize which components of a representation are task-relevant  and task-irrelevant, respectively.

For a nonlinear task function, unlike the linear case, a single global weight matrix that fully characterizes the transformation does not exist. 
Instead, we can analyze the local behavior of $\mathcal{T}$ around a specific representation $\boldsymbol{z}_0$ using a first-order Taylor expansion: $\mathcal{T}(\boldsymbol{z}_0 + \Delta \boldsymbol{z}) \approx \mathcal{T}(\boldsymbol{z}_0) + \boldsymbol{J}_{\mathcal{T}}(\boldsymbol{z}_0)\Delta \boldsymbol{z}$, where $\Delta \boldsymbol{z}$ is a sufficiently small perturbation.\footnote{This local linear approximation is well-suited for modern deep learning architectures, e.g. neural networks composed of parallel ReLU-equipped units that function as piecewise linear approximators.}
Locally, the Jacobian $\boldsymbol{J}_{\mathcal{T}}(\boldsymbol{z}_0)$ acts analogously to the linear weight matrix $\boldsymbol{W}$, allowing us to define a local row space and a local null space specific to the point $\boldsymbol{z}_0$.
If the perturbation $\Delta \boldsymbol{z}$ lies within the local null space of $\boldsymbol{J}_{\mathcal{T}}(\boldsymbol{z}_0)$, then $\boldsymbol{J}_{\mathcal{T}}(\boldsymbol{z}_0)\Delta \boldsymbol{z} = \mathbf{0}$, and consequently $\mathcal{T}(\boldsymbol{z}_0 + \Delta \boldsymbol{z}) \approx \mathcal{T}(\boldsymbol{z}_0)$, meaning that the nonlinear transformation remains locally invariant to perturbations along the null space directions.

\subsection{Local Differential Privacy}\label{sec:LDP}
Local Differential Privacy (LDP)~\citep{kasiviswanathan2011what} applies a randomization mechanism to each data representation $\boldsymbol{z} \in \mathcal{Z}$, within the representation space $\mathcal{Z} \subset \mathbb{R}^{m}$.
Let $\epsilon > 0$ denote the privacy budget and $\delta \geq 0$ the failure probability, where $\delta=0$ corresponds to pure $\epsilon$-LDP and $\delta>0$ to $(\epsilon, \delta)$-LDP.
A mechanism $\mathcal{M}:\mathbb{R}^{m} \rightarrow \mathbb{R}^{m}$ satisfies $(\epsilon, \delta)$-LDP if, for any pair of representations $\boldsymbol{z}, \boldsymbol{z}' \in \mathcal{Z}$ and any measurable subset $\mathcal{S} \subseteq \text{Range}(\mathcal{M})$, it holds that $\Pr[\mathcal{M}(\boldsymbol{z}) \in \mathcal{S}] \leq e^{\epsilon} \Pr[\mathcal{M}(\boldsymbol{z}') \in \mathcal{S}] + \delta$.
A canonical class of mechanisms satisfying these guarantees is that of additive noise mechanisms~\citep{dwork2014algorithmic} of the form
\begin{equation} \label{eq:1}
\tilde{\boldsymbol{z}} = \mathcal{M}(\boldsymbol{z}) = \boldsymbol{z} + \boldsymbol{\xi}.
\end{equation}
For $(\epsilon,0)$-LDP (i.e., pure $\epsilon$-LDP), the Laplace mechanism $\mathcal{M}_L$ is a natural choice, where the noise scale $b$ is directly proportional to the $\ell_1$-sensitivity, defined as $\Delta_1 = \max_{\boldsymbol{z}, \boldsymbol{z}' \in \mathcal{Z}} \|\boldsymbol{z} - \boldsymbol{z}'\|_1$, and each element of $\boldsymbol{\xi}$ is sampled i.i.d.\ from $\mathrm{Lap}(0, \Delta_1/\epsilon)$~\citep{dwork2014algorithmic}.
For $(\epsilon, \delta)$-LDP with $\delta>0$, the Gaussian mechanism $\mathcal{M}_G$ is a natural choice, where $\boldsymbol{\xi} \sim \mathcal{N}(\boldsymbol{0}, \sigma^2 \mathbf{I})$ and the noise scale $\sigma$ is proportional to the $\ell_2$-sensitivity, $\Delta_2 = \max_{\boldsymbol{z}, \boldsymbol{z}' \in \mathcal{Z}} \|\boldsymbol{z} - \boldsymbol{z}'\|_2$, calibrated via the Analytic Gaussian Mechanism~\citep{balle2018improving}.

As a local instantiation of DP, LDP inherits its fundamental strengths, including worst-case bounds on privacy leakage and immunity to post-processing~\citep{dwork2014algorithmic}.
These theoretical bounds hold even against an adversary with full knowledge of the randomization mechanism $\mathcal{M}$ and its parameters, including the privacy budget, sensitivity, and noise distribution settings.
Furthermore, the privacy guarantees of LDP are preserved under arbitrary post-processing independent of (private) data representation.
Once $\tilde{\boldsymbol{z}}$ is produced by a mechanism satisfying $(\epsilon, \delta)$-LDP, any measurable downstream computation or transformation $g(\tilde{\boldsymbol{z}})$ independent of the private data maintains the same privacy guarantee without consuming additional privacy budget.

\section{Motivation}

Our primary objective is to enhance the utility of randomized representations generated by an LDP mechanism $\mathcal{M}$. 
We focus on pure $\epsilon$-LDP ($\delta=0$), while our approach also extends to $(\epsilon,\delta)$-LDP with a non-zero probability of failure (i.e., $\delta > 0$).
Nonetheless, we explore a mechanism-agnostic approach that can be extended to $(\epsilon, \delta)$-LDP without loss of generality.

A fundamental obstacle to this goal is the severe utility degradation caused by the heavy randomization. 
Consider a two-dimensional representation space ($m = 2$) where each coordinate is bounded within $[B_L, B_U]$ ($B_L, B_U \in \mathbb{R}, B_L < B_U$), yielding an $\ell_1$-sensitivity of $\Delta_1 = 2(B_U-B_L)$.
Under a relatively loose practical privacy regime with $\epsilon = 10$, the Laplace mechanism $\mathcal{M}_L$ requires a noise scale of $b = (B_U-B_L)/5$.
Since the maximum span of each dimension is $B_U-B_L$, the injected noise often perturbs representations by a magnitude comparable to, or exceeding, the scale of the representation space itself. 
This leads to severe semantic distortion of the randomized representation $\tilde{\boldsymbol{z}}$.
In addition, this degradation is exacerbated in high-dimensional spaces: as the sensitivities scale with the dimensionality (i.e., $\Delta_1 = \mathcal{O}(m)$ and $\Delta_2 = \mathcal{O}(\sqrt{m})$).

To mitigate this limitation, we explore a novel approach that mitigates the noise injected into the task-relevant subspaces.
Our approach is motivated by the observation that a classifier exhibits anisotropic sensitivity to perturbations in the representation space.
Consider a linear binary classifier, $\mathcal{T}(\boldsymbol{z})=\boldsymbol{w}^\top\boldsymbol{z}+b$, with a single decision boundary in $\mathbb{R}^2$.
The decision boundary is defined by a normal vector $\boldsymbol{w} \in \mathbb{R}^2$, and the space can be spanned by two basis vectors $\boldsymbol{u}_{r}$ and $\boldsymbol{u}_{n}$, aligned parallel and perpendicular to $\boldsymbol{w}$, respectively.
$\boldsymbol{u}_{r}$ and $\boldsymbol{u}_{n}$ span the row space and null space of the classifier, respectively.
Crucially, perturbing a representation $\boldsymbol{z}$ along $\boldsymbol{u}_{r}$ can alter the classification outcome, whereas perturbations of any arbitrary magnitude along $\boldsymbol{u}_{n}$ leave the classification result unchanged.
We provide an illustration of this motivating example in Appendix~\ref{appdx:motivating_example}.
Based on this observation, our approach anisotropically reshapes the isotropic noise $\boldsymbol{\xi}$ to allocate less noise along the task-sensitive direction $\boldsymbol{u}_{r}$ and more along the task-insensitive direction $\boldsymbol{u}_{n}$. 
\section{Proposed Approach}
We propose to reduce the noise injected into the task-critical subspaces defined by the public downstream model. 
Our approach comprises a pre-processing function $f:\mathcal{Z} \rightarrow \bar{\mathcal{Z}}$, where $\bar{\mathcal{Z}}$ is an intermediate bounded space in $\mathbb{R}^m$, and a post-processing function $g: \mathbb{R}^{m} \rightarrow \mathbb{R}^{m}$.
A randomization mechanism $\mathcal{M}$ is applied to the intermediate representation $\bar{\boldsymbol{z}} = f(\boldsymbol{z})$, followed by the application of the post-processing function $g$.
As a result, this process yields a randomized version of $\boldsymbol{z}$ injected with anisotropic noise $\boldsymbol{\xi}_a \in \mathbb{R}^m$:
\begin{equation}\label{eq:2}
\hat{\boldsymbol{z}}= g(\mathcal{M}(f(\boldsymbol{z}))) = g(\mathcal{M}(\bar{\boldsymbol{z}})) = g(\bar{\boldsymbol{z}} + \boldsymbol{\xi}) \approx \boldsymbol{z} + \boldsymbol{\xi}_a,
\end{equation}
where the covariance matrix $\boldsymbol{\Sigma} \in \mathbb{R}^{m \times m}$ of $\boldsymbol{\xi}_a$ is publicly known and positive-definite.
This matrix is designed to rotate and scale the noise, attenuating it along the row space of the downstream model, while amplifying it along the null space.
A more detailed description is provided in Appendix~\ref{appdx:pa_description}.

\subsection{Pre-processing and Post-processing} \label{sec:pre_post}
Suppose that the eigendecomposition of $\boldsymbol{\Sigma}$ is given by $\boldsymbol{\Sigma} = \boldsymbol{U}\boldsymbol{\Lambda}\boldsymbol{U}^\top$ where $\boldsymbol{U} \in \mathbb{R}^{m \times m}$ is an orthogonal matrix representing rotation, and $\boldsymbol{\Lambda} \in \mathbb{R}^{m \times m}$ is a diagonal matrix that scales each coordinate in $\mathbb{R}^{m}$.
We derive our pre-processing and post-processing functions by modifying \eqref{eq:3}
\begin{equation}\label{eq:3}
\boldsymbol{z} + \boldsymbol{\xi}_a = \boldsymbol{L}(\boldsymbol{L}^{-1}(\boldsymbol{z} - \boldsymbol{\mu})+\boldsymbol{\xi}) + \boldsymbol{\mu}
\end{equation}
where $\boldsymbol{\mu} \in \mathbb{R}^{m}$ represents an offset and $\boldsymbol{L}=\boldsymbol{U}\boldsymbol{\Lambda}^{1/2} \in \mathbb{R}^{m \times m}$ denotes a reshaping factor.

\textbf{Pre-processing.}
Define the configuration tuple as $\boldsymbol{\phi} = (\boldsymbol{U}, \boldsymbol{\Lambda}, \boldsymbol{\mu})$, and let $\rho > 0$ and $p \in [1,\infty]$.
Incorporating clipping, we define our pre-processing function $f_{\boldsymbol{\phi}}:\mathcal{Z} \rightarrow \bar{\mathcal{Z}}$ as \eqref{eq:4}:
\begin{equation}\label{eq:4}
f_{\boldsymbol{\phi}}(\boldsymbol{z};p,\rho) = \theta_p\left(\boldsymbol{L}^{-1}(\boldsymbol{z}-\boldsymbol{\mu});\rho\right)\boldsymbol{L}^{-1}(\boldsymbol{z}-\boldsymbol{\mu})
\end{equation}
where $\theta_p(\cdot;\rho)$ denotes $\ell_p$-norm-based radial clipping with a threshold $\rho$.
The detailed construction of $\boldsymbol{L} = \boldsymbol{U}\boldsymbol{\Lambda}^{1/2}$ via Jacobian computation is presented in Section~\ref{sec:cov}.
Using $\boldsymbol{L}^{-1}$, this pre-processing maps the representation $\boldsymbol{z}$ from the space $\mathcal{Z} \subset \mathbb{R}^m$ to another bounded space $\bar{\mathcal{Z}} \subset \mathbb{R}^m$. 
Specifically, it first offsets $\boldsymbol{z}$ by $\boldsymbol{\mu}$, applies an inverse rotation and rescaling via $\boldsymbol{L}^{-1}$, and subsequently bounds the result to ensure that $\bar{\boldsymbol{z}} \in \bar{\mathcal{Z}}$ strictly resides within a bounded space.

\textbf{Post-processing.}
Furthermore, we define the post-processing function $g_{\boldsymbol{\phi}}: \mathbb{R}^{m} \rightarrow \mathbb{R}^{m}$ as \eqref{eq:5}:
\begin{equation}\label{eq:5}
g_{\boldsymbol{\phi}}(\mathcal{M}(f_{\boldsymbol{\phi}}(\boldsymbol{z};p,\rho))) = \boldsymbol{L} (f_{\boldsymbol{\phi}}(\boldsymbol{z};p,\rho)+\boldsymbol{\xi}) + \boldsymbol{\mu} \approx \boldsymbol{z} + \boldsymbol{\xi}_a.
\end{equation}
Our post-processing function $g$ reverses the transformation $\boldsymbol{L}^{-1}$ and the offset $\boldsymbol{\mu}$ applied in $f$, while simultaneously rotating and rescaling the isotropic noise $\boldsymbol{\xi}$ injected into $\bar{\mathcal{Z}}$. 
This yields a new anisotropic noise vector $\boldsymbol{\xi}_a$ designed to enhance downstream utility. 
Notably, by virtue of post-processing immunity, this approach preserves the privacy guarantees because $g$ is applied to the output of the $\epsilon$-LDP mechanism $\mathcal{M}$.

\subsection{Covariance Matrix for Rotation and Rescaling}\label{sec:cov}
To exploit post-processing immunity, our approach relies on a covariance matrix, $\boldsymbol{\Sigma}$, constructed from public knowledge. 
In typical LDP settings, although downstream models cannot access the (private) representation $\boldsymbol{z}$, they are generally expected to perform interpolation or be trained on data following a similar distribution. 
We therefore assume access to known downstream models or those pre-trained on public datasets, and estimate $\boldsymbol{\Sigma}$ by analyzing the Jacobian of these public models.

The construction of $\boldsymbol{\Sigma}$ consists of two components: the \emph{rotation} and the \emph{scaling}, as mentioned in Section~\ref{sec:pre_post}. 
We first determine the rotation matrix $\boldsymbol{U}$, which aligns the axes of $\boldsymbol{\xi}_a$ with the basis vectors of the task-sensitive and task-insensitive subspaces. 
Then, we determine the scaling matrix $\boldsymbol{\Lambda}=\operatorname{diag}(\lambda_1, \cdots, \lambda_m) \succ 0$, which assigns noise scales to each coordinate to enhance data utility.

\textbf{Rotation.} To identify the relevant basis vectors for rotation, we characterize the global geometric behavior of the downstream models. 
First, we evaluate the local Jacobian matrices $\boldsymbol{J}_{i} \in \mathbb{R}^{k \times m}$ for $i \in \{1, \dots, N\}$ at $N$ samples from public datasets. 
Next, we construct an aggregated Jacobian matrix $\boldsymbol{J} = \frac{1}{\sqrt{N}} \begin{bmatrix} \boldsymbol{J}_1^\top & \cdots & \boldsymbol{J}_N^\top \end{bmatrix}^\top \in \mathbb{R}^{Nk \times m}$ by vertically stacking $\boldsymbol{J}_{i}$. 
Subsequently, we obtain the Gram matrix of this empirically aggregated matrix, $\widehat{\boldsymbol{\Omega}}=\boldsymbol{J}^\top \boldsymbol{J}$, and perform eigendecomposition on $\widehat{\boldsymbol{\Omega}}$ such that $\widehat{\boldsymbol{\Omega}} = \boldsymbol{V}\operatorname{diag} (\gamma_1,\dots,\gamma_m)\boldsymbol{V}^\top$, where $\operatorname{diag}(\gamma_1, \dots, \gamma_m) \in \mathbb{R}^{m \times m}$ is a positive semi-definite diagonal matrix containing the eigenvalues, and $\boldsymbol{V} \in \mathbb{R}^{m \times m}$ represents the eigenvectors of $\widehat{\boldsymbol{\Omega}}$ used for our rotation ($\boldsymbol{U}=\boldsymbol{V}$). 
The eigenvectors in $\boldsymbol{U}$ associated with significant eigenvalues span the global active subspace \citep{constantine2014active}, representing the task-sensitive directions. 
Conversely, the vectors corresponding to zero eigenvalues span the inactive subspace, representing the task-insensitive directions.

\textbf{Scaling.} 
The objective of scaling is to assign a tailored noise scale factor $\sigma_i$ to each coordinate $i \in \{1, 2, \dots, m\}$ of the noise $\boldsymbol{\xi}_a = \boldsymbol{L}\boldsymbol{\xi}$, where these coordinates are aligned with the basis vectors of the task-sensitive and task-insensitive subspaces. 
To determine these individual scale factors, we use the eigenvalues $\gamma_i$ of the Jacobian Gram matrix, as they represent the relative importance of the $i$-th coordinate. 
Additionally, we impose $p=1$ and a gauge constraint such that the sum of the inverse noise scales remains constant before and after rescaling; that is, $m = \sum_{i=1}^m (\lambda_i)^{-1/2}$, where $\lambda_i = \left(\frac{\sigma_i}{\sigma} \right)^2$ and $0 <\sigma_i < \infty$. 
The isotropic noise scale $\sigma$ is determined once $\rho$ is set. 
Furthermore, to prevent the risk from diverging due to the empirical Jacobian, the distribution discrepancy between public and private data, and the nonlinear remainder, we assign a finite scale cap $1 < \lambda_{\max} < \infty$ to the task-insensitive coordinates associated with zero or very small eigenvalues. 
Let $C$ and $F$ denote the sets of capped and free coordinates, respectively. 
The method of Lagrange multipliers yields the scale factors as
\begin{equation}
    \lambda_i^\star =
    \begin{cases}
        \lambda_{\max}, & i\in C, \\[10pt]
        \displaystyle \left( \frac{ m-|C|\lambda_{\max}^{-1/2} }{ \sum_{j\in F}\gamma_j^{1/3} } \gamma_i^{1/3} \right)^{-2}, & i\in F. 
    \end{cases}
    \label{eq:scale}
\end{equation}
Eq.~\eqref{eq:scale} is obtained by jointly optimizing the Jacobian basis and the scale allocation, using the noise component of the first-order risk surrogate as the objective function. 
This choice of objective is inspired by our observation that the noise scale, rather than the distortion induced by clipping, is the primary bottleneck under LDP constraints (see Section~\ref{sec:ablation} for details). 
A detailed analysis of the derivation and performance bounds is provided in Appendix~\ref{append:analysis}.

\subsection{Sensitivity Bounding}
We employ a clipping operation to control sensitivity and reduce the noise injected into task-sensitive subspaces. Applying $\boldsymbol{L}^{-1}$ during the pre-processing step stretches the task-sensitive subspaces and shrinks the task-insensitive ones. 
This stretching amplifies the overall sensitivity in $\mathbb{R}^{m}$ (see Appendix~\ref{appendix:3}), necessitating the injection of a large amount of noise to guarantee $\epsilon$-LDP. 
If $\boldsymbol{L}$ is applied during post-processing without clipping, the task-sensitive subspaces receive the same amount of noise as they would if isotropic noise were directly injected into the original space $\mathcal{Z}$, while the task-insensitive subspaces suffer from excessive noise amplification. 
However, by regulating the sensitivity through clipping in the transformed space $\boldsymbol{L}^{-1}\mathcal{Z}$, we can effectively reduce the noise injected into the task-sensitive subspaces while satisfying $\epsilon$-LDP. 
The task-insensitive subspaces also remain protected by $\epsilon$-LDP-compliant noise, although this noise is often excessive.

\subsection{Privacy Guarantee}
Our configuration parameters $\boldsymbol U$, $\boldsymbol\Lambda$, $\boldsymbol\mu$, $p$, and $\rho$ are selected using public information and fixed before privatization.
Privacy then follows from the bounded-domain property of $f_{\boldsymbol\phi}$, the privacy guarantee of the base randomizer $\mathcal{M}$ on that domain, and post-processing by $g_{\boldsymbol\phi}$.

\begin{restatable}[Privacy Guarantee of Jacobian-Guided Anisotropic Noise Reshaping]
    {theorem}{privacytheorem}
\label{thm:privacy}

Let $\mathcal{Z} \subseteq \mathbb{R}^m$ be the domain of a single private representation, which may possibly be unbounded. Assume: 
\begin{itemize}[leftmargin=*]
    \item \textbf{(A1) Public configuration.} Conditional on fixed public side information, let $\boldsymbol{U}$ be orthogonal and $\boldsymbol{\Lambda} = \operatorname{diag}(\lambda_1,\ldots,\lambda_m)$ satisfy $0 < \lambda_{\min} \le \lambda_i \le \lambda_{\max} < \infty$.
    Define the configuration tuple $\boldsymbol{\phi} = (\boldsymbol{U}, \boldsymbol{\Lambda}, \boldsymbol{\mu})$, and let $\boldsymbol{L} = \boldsymbol{U}\boldsymbol{\Lambda}^{1/2}$ (which is invertible and has a finite condition number). 
    Let $\rho > 0$ and $p \in [1,\infty]$. $\boldsymbol{\phi}$, $p$, and $\rho$ may be chosen analytically or empirically using only the public task model, public data, and public validation results. Once selected, they are fixed before any input is privatized, and are neither computed from nor adapted to the protected representation. 
    \item \textbf{(A2) Bounded intermediate domain.} With $\theta_p(\boldsymbol{\bar{z}};\rho) = \min(1, \rho/\|\boldsymbol{\bar{z}}\|_p)$ for $\boldsymbol{\bar{z}} \neq \boldsymbol{0}$ and $1$ otherwise, define $f_\phi(\boldsymbol z) = \theta_p\!\left( \boldsymbol L^{-1}(\boldsymbol z-\boldsymbol\mu);\rho \right) \boldsymbol L^{-1}(\boldsymbol z-\boldsymbol\mu)$ and the bounded domain $\bar{\mathcal{Z}}_{p,\rho} = \{\boldsymbol{\bar{z}} \in \mathbb{R}^m : \|\boldsymbol{\bar{z}}\|_p \le \rho\}$. Radial clipping ensures $f_{\boldsymbol{\phi}}(\mathcal{Z}) \subseteq \bar{\mathcal{Z}}_{p,\rho}$ for every $\mathcal{Z}$, regardless of whether it is bounded or not. 
    \item \textbf{(A3) Underlying randomizer.} The base randomizer $\mathcal{M}$ is instantiated with the domain and calibration required by its original privacy theorem. Under the corresponding conditions, $\mathcal{M}$ satisfies $(\epsilon,\delta)$-LDP uniformly over $\bar{\mathcal{Z}}_{p,\rho} = \{ \boldsymbol{\bar{z}}\in\mathbb{R}^m: \|\boldsymbol{\bar{z}}\|_p\le \rho \}$. 
    \item \textbf{(A4) Post-processing.} $g_{\boldsymbol{\phi}}(\boldsymbol{v}) = \boldsymbol{L}\boldsymbol{v} + \boldsymbol{\mu}$ is deterministic, measurable, and fixed solely from public information. 
\end{itemize}

Then for every admissible configuration $\boldsymbol{\phi}$, the complete mechanism $\Phi_{\boldsymbol{\phi}} = g_{\boldsymbol{\phi}} \circ \mathcal{M} \circ f_{\boldsymbol{\phi}}$ satisfies $(\epsilon,\delta)$-LDP on $\mathcal{Z}$.
\end{restatable}

We provide the proof and detailed instantiations of several randomizers, including Laplace~\citep{dwork2014algorithmic}, AGM~\citep{balle2018improving}, and PrivUnit2~\citep{bhowmick2018privunit,asi2022privunitg} in Appendix~\ref{appendix:privacy_guarantee}.

Additional properties of our approach are detailed in Appendices~\ref{sec:apdx_comp}--\ref{sec:apdx_fl}.
\section{Evaluation}\label{sec:eval}
We evaluate the utility of our method by measuring predictive performance across downstream inference tasks (one regression and two image classification tasks) under varying representation dimensions ($m \in \{16, 32, 64\}$) and privacy budgets ($\epsilon \in [0.5, 10]$).

\textbf{Datasets.} We apply user-level LDP to time-series regression using London household smart meter (LHSM) data~\citep{ukpowernetworks_lcl,jeanmidev2019smartmeters}, which records the half-hourly electricity consumption (kWh/half-hour) of 5,547 London households over approximately 27 months. 
Additionally, we apply item-level LDP (see Corollary~\ref{corollary:item-level}) to image classification on the CIFAR-10-C, CIFAR-10 datasets~\cite{krizhevsky2009cifar}, and MNIST~\citep{lecun1998mnist}.

We split the LHSM dataset into public and private sets based on both household and chronological criteria. 
The public dataset comprises the initial 14 months of data from 3,000 randomly sampled households, used for training the LR model. 
The private dataset consists of the remaining 13 months from the other households, 
where it is independent and out-of-sample for the pre-trained model.

The training sets from CIFAR-10 and MNIST are treated as public. 
These are partitioned into training and validation sets with an 80:20 split to pre-train the public feature extractors (VAE~\cite{kingma2013auto} and ResNet-20~\cite{he2016deep}) and classifiers. 
To assess the utility of the mechanisms, the test sets of CIFAR-10-C~\cite{hendrycks2019benchmarking} (e.g., brightness, fog, defocus blur, and Gaussian noise), CIFAR-10, and MNIST are treated as private. 
Note that the CIFAR-10-C test set exhibits a distribution shift compared to the CIFAR-10 training set.

\textbf{Downstream Tasks.}
Both the regression and classification models are trained on the public dataset and subsequently evaluated by feeding them the randomized private dataset to measure accuracy and agreement. 
Specifically, the regression task involves predicting the next day's \emph{energy\_mean} using the previous 16 days ($m=16$) of \emph{energy\_mean} values as input. 
For the classification task, an $m$-dimensional feature representation ($m \in \{32, 64\}$) is first extracted using a public feature extractor and then fed into the classifier.

\textbf{Downstream Models.} We employ linear regression (LR), linear classifiers (LC), and multi-layer perceptron (MLP) classifiers with ReLU, GELU, and Tanh activations. 

\textbf{Baselines.} 
Our evaluation includes foundational mechanisms, such as the standard Laplace~\citep{dwork2014algorithmic}, AGM~\citep{balle2018improving}, and PrivUnit2~\citep{bhowmick2018privunit}, and state-of-the-art approaches, including Instance Optimal (IO)~\citep{huang2021instance}, PrivUnitG~\citep{asi2022privunitg}, Task-Aware~\citep{cheng2022task}, PLAN~\citep{aumuller2024plan}, and Coordinate-wise (CW)~\citep{muthukrishnan2025DP}.

\subsection{Results: Privacy-Utility Trade-offs}
We evaluate utility using Root Mean Squared Error (RMSE) for regression and average accuracy for classification, where lower RMSE and higher accuracy indicate better performance, respectively.
Tables~\ref{tb:lr} and \ref{tb:nonlinear} report the conservative results of the proposed approach (PA) over 20 different seeds: RMSE across 2,447 households for regression, and average accuracy over 10,000 samples for nonlinear classification using CIFAR-10-C (brightness). 
Further details regarding our experimental setup, evaluation protocol and dataset descriptions are provided in Appendix~\ref{appendix:exp}. 
Due to the page limit, we here present the primary results for downstream models, specifically LR and an MLP classifier with ReLU activations.

\textbf{Linear Regression.}
Mechanisms without PA exhibited high RMSE values under low privacy budgets. 
At $\epsilon=0.5$, Laplace recorded 17.63 kWh/half-hour, while PrivUnit2 and PrivUnitG reached 3.97 and 3.96 kWh/half-hour, respectively. 
Although RMSE decreased gradually as $\epsilon$ increased, even at $\epsilon=10.0$, Laplace remained at 0.89 kWh/half-hour and PrivUnit2 at 0.32 kWh/half-hour, indicating limited utility without PA.
Task-Aware approach~\citep{cheng2022task} maintained its performance largely regardless of the budget; yet, in tasks that deviate from their MSE optimization objective, such as classification, we observed collapse, with performance falling below that of Laplace (Table 11 in Appendix~\ref{apdx_activation}).

Mechanisms augmented with PA achieved lower RMSE across all $\epsilon$ values. 
Laplace+PA recorded 1.05 kWh/half-hour at $\epsilon=0.5$ and 0.54 kWh/half-hour at $\epsilon=1.0$, reducing RMSE by factors of approximately 16.8 and 16.4, respectively, compared with the base Laplace mechanism.
PrivUnit2+PA and PrivUnitG+PA demonstrated comparable improvements, and they outperformed Laplace+PA.
CW+PA w/o bounding also exhibited competitive performance. 
This may show that CW’s DMSE optimization translates to linear problems computed in the form of $\boldsymbol{Wz}$.
It recorded 0.14 kWh/half-hour at $\epsilon=5.0$ and 0.10 kWh/half-hour at $\epsilon=7.5$, outperforming all other mechanisms and approaching the no-randomization baseline of 0.07 kWh/half-hour most closely. 
While approximate mechanisms, AGM, IO, and PLAN, outperform the standard Laplace mechanism, Laplace+PA surpasses them all. 
Furthermore, these approximate mechanisms also fall short of CW+PA w/o bounding. 
Considering that approximate mechanisms are specifically designed to improve utility by tolerating a privacy failure probability, it is significant that PA-integrated approaches substantially outperform them.

\begin{table}[t]
\caption{Main results for \textbf{Linear Regression} ($m=16$) on the \textbf{LHSM}. The Root Mean Square Error (RMSE) is used. PA denotes the proposed approach.}
\centering
\resizebox{\textwidth}{!}{
\begin{tabular}{l cccccccccccc}
\toprule
& \multicolumn{12}{c}{Privacy Budget ($\epsilon$)} \\
 \cmidrule(l){2-13}
 \textbf{Mechanism} & 0.5 & 1.0 & 1.5 & 2.0 & 2.5 & 3.0 & 3.5 & 4.0 & 4.5 & 5.0 & 7.5 & 10.0 \\
 \midrule
 No randomization & 0.0691 & 0.0691 & 0.0691 & 0.0691 & 0.0691 & 0.0691 & 0.0691 & 0.0691 & 0.0691 & 0.0691 & 0.0691 & 0.0691 \\
 \midrule
 Laplace ($\ell_1$) &
17.6293 & 8.8154 & 5.8778 & 4.4093 & 3.5284 & 2.9413 & 2.5221 & 2.2078 & 1.9635 & 1.7682 & 1.1832 & 0.8921 \\
 \midrule
 Laplace+PA &
\textbf{1.0520} & \textbf{0.5387} & \textbf{0.3729} & \textbf{0.2935} & \textbf{0.2482} & \textbf{0.2197} & \textbf{0.2006} & \textbf{0.1872} & \textbf{0.1774} & \textbf{0.1700} & \textbf{0.1511} & \textbf{0.1439} \\
 \midrule
 PrivUnit2 &
3.9709 & 2.0035 & 1.3492 & 1.0252 & 0.8348 & 0.7097 & 0.6222 & 0.5577 & 0.5083 & 0.4710 & 0.3619 & 0.3155 \\
 \midrule
 PrivUnit2+PA &
\textbf{0.9048} & \textbf{0.4776} & \textbf{0.3437} & \textbf{0.2822} & \textbf{0.2490} & \textbf{0.2282} & \textbf{0.2150} & \textbf{0.2058} & \textbf{0.1991} & \textbf{0.1945} & \textbf{0.1826} & \textbf{0.1780} \\
 \midrule
 PrivUnitG &
3.9645 & 1.9924 & 1.3471 & 1.0290 & 0.8425 & 0.7192 & 0.6333 & 0.5707 & 0.5235 & 0.4865 & 0.3831 & 0.3365 \\
 \midrule
 PrivUnitG+PA & 
\textbf{0.8996} & \textbf{0.4736} & \textbf{0.3406} & \textbf{0.2801} & \textbf{0.2471} & \textbf{0.2269} & \textbf{0.2136} & \textbf{0.2044} & \textbf{0.1981} & \textbf{0.1936} & \textbf{0.1822} & \textbf{0.1765} \\
 \midrule
 CW+PA w/o bounding  &
\textbf{1.1897} & \textbf{0.5976} & \textbf{0.4016} & \textbf{0.3045} & \textbf{0.2470} & \textbf{0.2093} & \textbf{0.1828} & \textbf{0.1634} & \textbf{0.1486} & \textbf{0.1371} & \textbf{0.1048} & \textbf{0.0908} \\
 \midrule
 Task-Aware & 
 0.1920 & 0.1920 & 0.1918 & 0.1916 & 0.1913 & 0.1910 & 0.1906 & 0.1901 & 0.1896 & 0.1890 & 0.1852 & 0.1802 \\
  \midrule
  \midrule
 AGM ($\ell_2$, $\delta=10^{-5}$) &
 11.6732 & 6.1941 & 4.2890 & 3.3123 & 2.7156 & 2.3121 & 2.0205 & 1.7995 & 1.6261 & 1.4863 & 1.0593 & 0.8400 \\
  \midrule
 IO ($\delta=10^{-5}$) &
 8.2244 & 5.9637 & 4.8461 & 4.1808 & 3.7330 & 3.3761 & 3.0977 & 2.8843 & 2.7001 & 2.5527 & 2.0428 & 1.7520 \\
  \midrule
 PLAN ($\delta=10^{-5}$) &
 8.4207 & 4.2118 & 2.8095 & 2.1088 & 1.6888 & 1.4091 & 1.2095 & 1.0601 & 0.9441 & 0.8514 & 0.5753 & 0.4395 \\
 \bottomrule
\end{tabular}
}
\vspace{-15pt}
\label{tb:lr}
\end{table}

\begin{table}[t]
\caption{Main results for the \textbf{Nonlinear MLP Classifier with ReLU} ($m=64$) on the \textbf{CIFAR-10-C (Brightness, Severity 5)}. Average accuracy is used. PA denotes the proposed approach.}
\centering
\resizebox{\textwidth}{!}{
\begin{tabular}{l cccccccccccc}
\toprule
& \multicolumn{12}{c}{Privacy Budget ($\epsilon$)} \\
 \cmidrule(l){2-13}
 \textbf{Mechanism} & 0.5 & 1.0 & 1.5 & 2.0 & 2.5 & 3.0 & 3.5 & 4.0 & 4.5 & 5.0 & 7.5 & 10.0 \\
 \midrule
 No randomization & 0.8365 & 0.8365 & 0.8365 & 0.8365 & 0.8365 & 0.8365 & 0.8365 & 0.8365 & 0.8365 & 0.8365 & 0.8365 & 0.8365 \\
 \midrule
 Laplace ($\ell_1$) &
0.1001 & 0.1013 & 0.1025 & 0.1038 & 0.1050 & 0.1063 & 0.1074 & 0.1087 & 0.1099 & 0.1111 & 0.1177 & 0.1251 \\
 \midrule
 Laplace+PA &
\textbf{0.1048} & \textbf{0.1103} & \textbf{0.1161} & \textbf{0.1219} & \textbf{0.1280} & \textbf{0.1342} & \textbf{0.1405} & \textbf{0.1467} & \textbf{0.1538} & \textbf{0.1609} & \textbf{0.1989} & \textbf{0.2416} \\
 \midrule
 PrivUnit2 &
0.1099 & 0.1203 & 0.1316 & 0.1434 & 0.1560 & 0.1686 & 0.1820 & 0.1950 & 0.2092 & 0.2227 & 0.2933 & 0.3615 \\
 \midrule
 PrivUnit2+PA &
\textbf{0.1194} & \textbf{0.1410} & \textbf{0.1660} & \textbf{0.1936} & \textbf{0.2223} & \textbf{0.2536} & \textbf{0.2840} & \textbf{0.3152} & \textbf{0.3453} & \textbf{0.3754} & \textbf{0.4978} & \textbf{0.5840} \\
 \midrule
 PrivUnitG &
0.1083 & 0.1190 & 0.1299 & 0.1415 & 0.1541 & 0.1665 & 0.1798 & 0.1931 & 0.2061 & 0.2197 & 0.2858 & 0.3375 \\
 \midrule
 PrivUnitG+PA & 
\textbf{0.1201} & \textbf{0.1425} & \textbf{0.1672} & \textbf{0.1936} & \textbf{0.2227} & \textbf{0.2524} & \textbf{0.2822} & \textbf{0.3121} & \textbf{0.3414} & \textbf{0.3688} & \textbf{0.4860} & \textbf{0.5648} \\
 \midrule
 CW+PA w/o bounding &
\textbf{0.1025} & \textbf{0.1058} & \textbf{0.1092} & \textbf{0.1128} & \textbf{0.1164} & \textbf{0.1201} & \textbf{0.1241} & \textbf{0.1282} & \textbf{0.1323} & \textbf{0.1366} & \textbf{0.1581} & \textbf{0.1828} \\
 \midrule
 \midrule
AGM ($\ell_2$, $\delta=10^{-5}$) &
0.1023 & 0.1046 & 0.1068  & 0.1090  & 0.1116 & 0.1139 & 0.1157 & 0.1181 & 0.1202 & 0.1225 & 0.1321 & 0.1423 \\
 \midrule
 IO ($\delta=10^{-5}$)&
0.1045 & 0.1080 & 0.1112 & 0.1142 & 0.1171 & 0.1198 & 0.1226 & 0.1253 & 0.1280 & 0.1307 & 0.1433 & 0.1561 \\
 \midrule
 PLAN ($\delta=10^{-5}$)&
0.1022 & 0.1049 & 0.1080 & 0.1107 & 0.1134 & 0.1163 & 0.1194 & 0.1224 & 0.1260 & 0.1290 & 0.1464 & 0.1654 \\
 \bottomrule
\end{tabular}
}
\vspace{-15pt}
\label{tb:nonlinear}
\end{table}

\textbf{Nonlinear Classification.}
The nonlinear classification experiments were conducted on the CIFAR-10-C (Brightness corruption at the highest severity level 5) test dataset using a ResNet-20-based nonlinear MLP classifier with ReLU activations, where 64-dimensional feature vectors serve as input.
Utility is evaluated by classification accuracy, with the no-randomization baseline achieving 0.8365. 
Baseline mechanisms without PA exhibited severely degraded performance across all privacy budgets.
Laplace $(\ell_1)$ yielded accuracies near 0.1, equivalent to random guessing on a 10-class task, while PrivUnit2 and PrivUnitG achieved results of 0.2933 and 0.2858 at $\epsilon=7.5$, respectively.

The task-aware approach by \citep{cheng2022task} is focused on MNIST and does not offer guidance for more complex datasets such as CIFAR-10(-C). Crucially, their method relies on direct access to private data during the training phase. Despite avoiding such direct access to private data, our approach demonstrates superior results. As detailed in Appendix~\ref{apdx:task-aware}, integrating PA with PrivUnit2 and PrivUnitG significantly outperforms their method under its optimal configuration, while the Laplace+PA mechanism yields comparable or superior performance.

The integration of PA consistently and significantly improved classification accuracy across all mechanisms and $\epsilon$ values. 
Laplace+PA outperformed its base counterpart at every operating point, achieving an accuracy of 0.1989 at $\epsilon=7.5$, compared to 0.1177 without PA. 
Furthermore, Laplace+PA surpassed all evaluated approximate mechanisms.
More notably, PrivUnit2+PA and PrivUnitG+PA delivered substantially stronger performance, achieving accuracies of 0.4978 and 0.4860 at $\epsilon=7.5$ and 0.3754 and 0.3688 at $\epsilon=5.0$, respectively. 
This represents an improvement of approximately 20 and 15 percentage points for each respective budget compared to their non-PA counterparts, underscoring the effectiveness of PA in recovering a significant portion of the utility lost to privacy noise.
CW+PA w/o bounding, however, underperformed relative to other PA-augmented mechanisms in this setting, recording only 0.1581 at $\epsilon=7.5$ and 0.1366 at $\epsilon=5.0$. 
This performance was even lower than that of Laplace+PA. 
The distinction between these two approaches lies in the use of sensitivity bounding versus a DMSE-optimized noise scale. 
These results suggest that our clipping-based approach is more effective under conditions where DMSE optimization is suboptimal.

\subsection{Ablation Study}
\label{sec:ablation}

\begin{table}[t]
\caption{Ablation study on \textbf{Pre-Processing} and \textbf{Post-Processing} for \textbf{Laplace + PA}, \textbf{PrivUnit2+PA}, and \textbf{PrivUnitG+PA} on \textbf{CIFAR-10-C (Brightness, Severity 5)}. Average accuracy is used.}
\centering
\resizebox{\textwidth}{!}{
\begin{tabular}{l cccccccccccc}
\toprule
\multirow{2}{*}{\textbf{Mechanism}} & \multicolumn{12}{c}{Privacy Budget ($\epsilon$)} \\
\cmidrule(l){2-13}
 & 0.5 & 1.0 & 1.5 & 2.0 & 2.5 & 3.0 & 3.5 & 4.0 & 4.5 & 5.0 & 7.5 & 10.0 \\
\midrule
 Laplace+PA ($\boldsymbol{\Lambda}_{\mathrm{Jac}}^{-1/2}\boldsymbol{U}_{\mathrm{Jac}}^{-1},\ell_1, \boldsymbol{U}_{\mathrm{Jac}}\boldsymbol{\Lambda}_{\mathrm{Jac}}^{1/2}$) 
 & \textbf{0.1045} & \textbf{0.1100} & \textbf{0.1161} & \textbf{0.1217} & \textbf{0.1279} & \textbf{0.1344} & \textbf{0.1411} & \textbf{0.1478} & \textbf{0.1543} & \textbf{0.1615} & \textbf{0.1996} & \textbf{0.2428} \\
 Laplace+($\boldsymbol{\Lambda}_{\mathrm{Jac}}^{-1/2}\boldsymbol{U}_{\mathrm{Jac}}^{-1},\ell_1,\boldsymbol{U}_{\mathrm{Jac}}$)
 & 0.1030 & 0.1048 & 0.1070 & 0.1087 & 0.1109 & 0.1130 & 0.1148 & 0.1171 & 0.1188 & 0.1210 & 0.1326 & 0.1447 \\
 Laplace+$(\ell_1)$
 & 0.1001 & 0.1015 & 0.1028 & 0.1042 & 0.1056 & 0.1068 & 0.1081 & 0.1096 & 0.1109 & 0.1124 & 0.1196 & 0.1276 \\
  Laplace+($\boldsymbol{\Lambda}_{\mathrm{Jac}}^{-1/2}\boldsymbol{U}_{\mathrm{PCA}}^{-1},\ell_1, \boldsymbol{U}_{\mathrm{PCA}}\boldsymbol{\Lambda}_{\mathrm{Jac}}^{1/2}$)  
 & 0.1014 & 0.1015 & 0.1016 & 0.1016 & 0.1017 & 0.1017 & 0.1019 & 0.1019 & 0.1019 & 0.1020 & 0.1022 & 0.1024 \\
 Laplace+($\boldsymbol{\Lambda}_{\mathrm{Jac}}^{-1/2}\boldsymbol{U}_{\mathrm{Rand}}^{-1},\ell_1, \boldsymbol{U}_{\mathrm{Rand}}\boldsymbol{\Lambda}_{\mathrm{Jac}}^{1/2}$)  
 & 0.1008 & 0.1008 & 0.1008 & 0.1008 & 0.1008 & 0.1008 & 0.1009 & 0.1009 & 0.1009 & 0.1009 & 0.1010 & 0.1011 \\
 Laplace+PA w/o Pre-processing & 0.1015 & 0.1014 & 0.1014 & 0.1014 & 0.1009 & 0.1009 & 0.1010 & 0.1009 & 0.1007 & 0.1008 & 0.1006 & 0.1001 \\
 Laplace+PA w/o Post-processing & 0.0984 & 0.0982 & 0.0979 & 0.0977 & 0.0976 & 0.0974 & 0.0971 & 0.0969 & 0.0966 & 0.0966 & 0.0954 & 0.0945 \\
 \midrule
 PrivUnit2+PA ($\boldsymbol{\Lambda}_{\mathrm{Jac}}^{-1/2}\boldsymbol{U}_{\mathrm{Jac}}^{-1},\ell_1, \boldsymbol{U}_{\mathrm{Jac}}\boldsymbol{\Lambda}_{\mathrm{Jac}}^{1/2}$)
 & \textbf{0.1196} & \textbf{0.1410} & \textbf{0.1662} & \textbf{0.1945} & \textbf{0.2244} & \textbf{0.2553} & \textbf{0.2856} & \textbf{0.3171} & \textbf{0.3471} & \textbf{0.3760} & \textbf{0.4986} & \textbf{0.5829} \\
 PrivUnit2+($\boldsymbol{\Lambda}_{\mathrm{Jac}}^{-1/2}\boldsymbol{U}_{\mathrm{Jac}}^{-1},\ell_1,\boldsymbol{U}_{\mathrm{Jac}}$)
 & 0.1155 & 0.1319 & 0.1509 & 0.1707 & 0.1926 & 0.2148 & 0.2381 & 0.2614 & 0.2856 & 0.3104 & 0.4209 & 0.5117 \\
 PrivUnit2+($\ell_1$)
 & 0.1111 & 0.1214 & 0.1328 & 0.1443 & 0.1567 & 0.1687 & 0.1820 & 0.1944 & 0.2092 & 0.2225 & 0.2942 & 0.3620 \\
 PrivUnit2+($\boldsymbol{\Lambda}_{\mathrm{Jac}}^{-1/2}\boldsymbol{U}_{\mathrm{PCA}}^{-1},\ell_1, \boldsymbol{U}_{\mathrm{PCA}}\boldsymbol{\Lambda}_{\mathrm{Jac}}^{1/2}$)
 & 0.1000 & 0.1001 & 0.1002 & 0.1003 & 0.1003 & 0.1002 & 0.1005 & 0.1008 & 0.1011 & 0.1013 & 0.1015 & 0.1019 \\
 PrivUnit2+($\boldsymbol{\Lambda}_{\mathrm{Jac}}^{-1/2}\boldsymbol{U}_{\mathrm{Rand}}^{-1},\ell_1, \boldsymbol{U}_{\mathrm{Rand}}\boldsymbol{\Lambda}_{\mathrm{Jac}}^{1/2}$)
 & 0.0993 & 0.0995 & 0.0997 & 0.0999 & 0.1001 & 0.1004 & 0.1007 & 0.1012 & 0.1019 & 0.1030 & 0.1043 & 0.1039 \\
 PrivUnit2+PA w/o Pre-processing & 0.1000 & 0.0994 & 0.0999 & 0.0997 & 0.0997 & 0.0994 & 0.0990 & 0.0986 & 0.0986 & 0.0984 & 0.0976 & 0.0955 \\
 PrivUnit2+PA w/o Post-processing 
 & 0.0995 & 0.0987 & 0.0979 & 0.0971 & 0.0961 & 0.0954 & 0.0947 & 0.0940 & 0.0935 & 0.0931 & 0.0890 & 0.0875 \\
 \midrule
 PrivUnitG+PA ($\boldsymbol{\Lambda}_{\mathrm{Jac}}^{-1/2}\boldsymbol{U}_{\mathrm{Jac}}^{-1},\ell_1, \boldsymbol{U}_{\mathrm{Jac}}\boldsymbol{\Lambda}_{\mathrm{Jac}}^{1/2}$)
 & \textbf{0.1205} & \textbf{0.1421} & \textbf{0.1666} & \textbf{0.1931} & \textbf{0.2230} & \textbf{0.2539} & \textbf{0.2838} & \textbf{0.3147} & \textbf{0.3440} & \textbf{0.3725} & \textbf{0.4887} & \textbf{0.5669} \\
 PrivUnitG+($\boldsymbol{\Lambda}_{\mathrm{Jac}}^{-1/2}\boldsymbol{U}_{\mathrm{Jac}}^{-1},\ell_1,\boldsymbol{U}_{\mathrm{Jac}}$)
 & 0.1146 & 0.1312 & 0.1501 & 0.1700 & 0.1910 & 0.2142 & 0.2373 & 0.2607 & 0.2839 & 0.3078 & 0.4121 & 0.4869 \\
 PrivUnitG+($\ell_1$)
 & 0.1089 & 0.1196 & 0.1303 & 0.1420 & 0.1551 & 0.1679 & 0.1812 & 0.1942 & 0.2071 & 0.2210 & 0.2864 & 0.3378 \\
 PrivUnitG+($\boldsymbol{\Lambda}_{\mathrm{Jac}}^{-1/2}\boldsymbol{U}_{\mathrm{PCA}}^{-1},\ell_1, \boldsymbol{U}_{\mathrm{PCA}}\boldsymbol{\Lambda}_{\mathrm{Jac}}^{1/2}$)
 & 0.0992 & 0.0993 & 0.0994 & 0.0995 & 0.0995 & 0.0996 & 0.0998 & 0.0998 & 0.1000 & 0.1001 & 0.1005 & 0.1009 \\
 PrivUnitG+($\boldsymbol{\Lambda}_{\mathrm{Jac}}^{-1/2}\boldsymbol{U}_{\mathrm{Rand}}^{-1},\ell_1, \boldsymbol{U}_{\mathrm{Rand}}\boldsymbol{\Lambda}_{\mathrm{Jac}}^{1/2}$)
 & 0.1012 & 0.1014 & 0.1016 & 0.1018 & 0.1020 & 0.1022 & 0.1025 & 0.1027 & 0.1030 & 0.1032 & 0.1041 & 0.1049 \\
 PrivUnitG+PA w/o Pre-processing 
 & 0.0986 & 0.0987 & 0.0985 & 0.0979 & 0.0977 & 0.0979 & 0.0976 & 0.0975 & 0.0975 & 0.0978 & 0.0966 & 0.0958 \\
 PrivUnitG+PA w/o Post-processing 
 & 0.0986 & 0.0983 & 0.0973 & 0.0964 & 0.0957 & 0.0951 & 0.0940 & 0.0932 & 0.0924 & 0.0919 & 0.0889 & 0.0869 \\
\bottomrule
\end{tabular}
}
\vspace{-15pt}
\label{tb:pre_post}
\end{table}

\paragraph{Pre-processing and Post-processing.}
We conduct an ablation study to investigate the individual contributions of the pre-processing and post-processing components within PA.

In summary, the ablation results demonstrate that both pre-processing and post-processing components are indispensable. 
Table~\ref{tb:pre_post} provides the results across all three mechanisms on CIFAR-10-C under brightness corruption at severity level 5.
Removing either component causes all three mechanisms to collapse to near-random performance (10\%), regardless of the privacy budget. 
A notable instability is observed when post-processing is ablated: the standard deviation of accuracy grows substantially with increasing $\epsilon$, reaching up to 0.026 for PrivUnit2+PA and 0.025 for PrivUnitG+PA at $\epsilon=10$. 
This suggests that without post-processing, the representations become erratic as the privacy budget relaxes, undermining the reliability of downstream classification.

To disentangle the internal mechanics of the post-processing phase, we evaluate a variant, which retains the noise rotation but omits the scaling procedure: ($\boldsymbol{\Lambda}_{\mathrm{Jac}}^{-1/2}\boldsymbol{U}_{\mathrm{Jac}}^{-1},\ell_1,\boldsymbol{U}_{\mathrm{Jac}}$). 
Across Laplace, PrivUnit2, and PrivUnitG, the performance of this variant falls between the fully ablated post-processing baseline and the full PA integration. 
For instance, with PrivUnit2 at $\epsilon=10$, the variant recovers substantial utility (0.5117) compared to the complete absence of post-processing (0.0875), yet still falls short of the full PA (0.5829). 
This result implies that while directional alignment (rotation) alone successfully mitigates a portion of the noise impact, structural scaling (reshaping) is necessary to fully utilize the proposed approach.

A further distinction emerges when comparing the variant without scaling against the direct $\ell_1$-clipping baseline applied in the original space $\mathcal{Z}$. 
Although both configurations omit anisotropic noise scaling, they differ in one fundamental respect. 
The variant without scaling applies $\ell_1$-clipping within the transformed space $\boldsymbol{\Lambda}_{\mathrm{Jac}}^{-1/2}\boldsymbol{U}_{\mathrm{Jac}}^{-1}(\boldsymbol{z} - \boldsymbol{\mu})$, where the axes are not only aligned with the task-critical subspaces but also inversely scaled according to their respective importance. 
In contrast, the direct $\ell_1$-clipping baseline operates in the raw space $\mathcal{Z}$ without any such alignment. 
The consistent performance gap between these two variants across all three mechanisms demonstrates that clipping in a task-aligned space is inherently beneficial even without rescaling. 
This highlights that the utility of PA stems from the interaction of three key components: space alignment, sensitivity bounding in the transformed space, and anisotropic rescaling.

We also experimented with replacing the Jacobian-guided rotation $\boldsymbol{U}_{\mathrm{Jac}}$ with PCA- or random permutation-based rotations: $\boldsymbol{U}_{\mathrm{PCA}}$ and $\boldsymbol{U}_{\mathrm{Rand}}$. 
These variants resulted in near-random performance, highlighting the necessity of the Jacobian-guided rotation.
Furthermore, compared to the $\ell_1$-clipping baseline, the full PA integration achieves up to 90\% relative improvement (Laplace, $\epsilon=10$), demonstrating that the utility gains stem not from clipping alone, but from the coordinated interaction of the pre- and post-processing pipeline introduced by PA.

\begin{figure}[t]
    \centering
    \includegraphics[width=\textwidth]{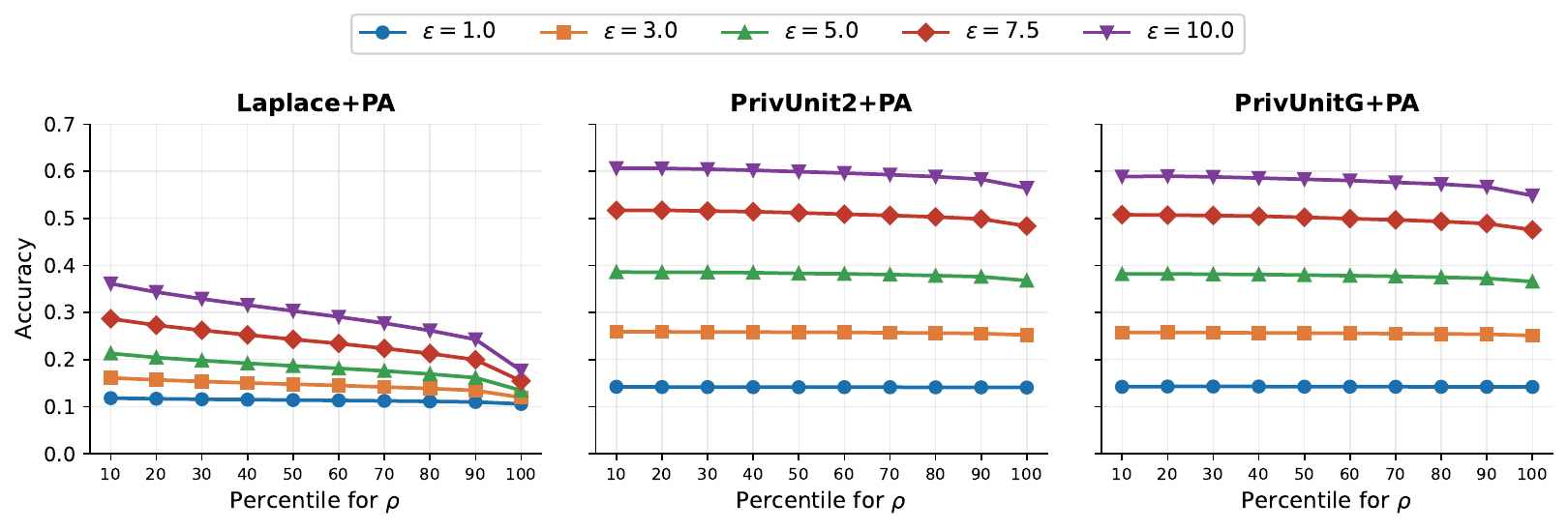} 
    \caption{Effect of the clipping threshold $\rho$ on accuracy across privacy budgets $\epsilon \in \{1, 3, 5, 7.5, 10\}$. }
\vspace{-15pt}
\label{fig_clip}
\end{figure}

\paragraph{Clipping Threshold.} 
We investigate the impact of varying clipping thresholds $\rho$ on ReLU-based nonlinear MLP classifiers, evaluated on the CIFAR-10-C dataset under maximum brightness corruption (Level 5). The clipping threshold $\rho$ is determined based on specific percentiles of the public data (e.g., the original CIFAR-10 training set).

In summary, despite the increased distortion of individual representations caused by clipping, the substantial reduction in the injected noise scale outweighs this effect. 
This suggests that the noise scale, rather than clipping-induced distortion, is the primary bottleneck under LDP constraints.
Figure~\ref{fig_clip} presents the results of integrating PA with the Laplace~\citep{dwork2014algorithmic}, PrivUnit2~\citep{bhowmick2018privunit}, and PrivUnitG~\citep{asi2022privunitg} mechanisms, across clipping thresholds $\rho$ corresponding to the 10th through 100th percentiles. 
Overall, the mechanisms integrated with PA achieve higher accuracy when a lower $\rho$ is applied. 
Tighter clipping thresholds (e.g., at the 10th percentile) consistently yield superior accuracy. 
Although omitted here for brevity, our approach reached peak performance at a $\rho$ corresponding to the 30th percentile for the ReLU-based MLP classifier operating on VAE-encoded MNIST features.

\subsection{Distribution Shift and Additional Baseline}
We provide results of our distribution shift analysis (across 20 corruption-severity combinations) and alternative activations (e.g., GELU, Tanh) in Appendices~\ref{apdx_corruption} and \ref{apdx_activation}, respectively.
In these evaluations, we found that i) PA-integrated approaches demonstrate robust performance across both the type and severity of practical distribution shifts, and ii) PA remains effective for nonlinear classifiers that employ various activation functions within hidden layers.
\section{Limitations}
Our theoretical analysis establishes the optimal basis and scale allocation for the noise-induced component of the first-order utility surrogate under the considered geometric constraint, and further characterizes the relation between the first-order surrogate and nonlinear output distortion and prediction stability.
For nonlinear downstream tasks, the end-to-end privacy guarantee remains exact, whereas the utility improvement over the isotropic baseline is established empirically rather than through a global guarantee on the nonlinear task loss.
A remaining limitation is that the transformation-dependent clipping term and the resulting nonlinear task loss are not jointly optimized with the noise allocation.
Since these effects depend on both the downstream model geometry and the representation distribution, deriving a tractable end-to-end objective that simultaneously accounts for clipping and nonlinear utility remains an important direction for future work.

\clearpage

\bibliographystyle{unsrt}
\bibliography{ref}

@article{kasiviswanathan2011what,
author  = {Kasiviswanathan, Shiva Prasad and Lee, Homin K and Nissim, Kobbi and Raskhodnikova, Sofya and Smith, Adam},
title   = {{What Can We Learn Privately?}},
journal = {SIAM Journal on Computing},
volume  = {40},
number  = {3},
pages   = {793--826},
year    = {2011}
}

@inproceedings{dwork2006calibrating,
author    = {Dwork, Cynthia and McSherry, Frank and Nissim, Kobbi and Smith, Adam},
title     = {{Calibrating noise to sensitivity in private data analysis}},
booktitle = {Proceedings of the Third Conference on Theory of Cryptography (TCC)},
pages     = {265--284},
year      = {2006},
series    = {Lecture Notes in Computer Science},
publisher = {Springer-Verlag}
}

@inproceedings{erlingsson2014rappor,
author    = {Erlingsson, {\'U}lfar and Pihur, Vasyl and Korolova, Aleksandra},
title     = {{RAPPOR: Randomized Aggregatable Privacy-Preserving Ordinal Response}},
booktitle = {Proceedings of the ACM SIGSAC Conference on Computer and Communications Security (CCS)},
pages     = {1054--1067},
year      = {2014}
}

@techreport{apple2017learning,
author      = {{Apple Differential Privacy Team}},
title       = {{Learning with Privacy at Scale}},
institution = {Apple Machine Learning Journal},
year        = {2017},
url         = {https://machinelearning.apple.com/research/learning-with-privacy-at-scale}
}

@inproceedings{ding2017collecting,
author    = {Ding, Bolin and Kulkarni, Janardhan and Yekhanin, Sergey},
title     = {{Collecting telemetry data privately}},
booktitle = {Advances in Neural Information Processing Systems (NIPS)},
volume    = {30},
pages     = {3574--3583},
year      = {2017}
}

@inproceedings{duchi2013local,
author    = {Duchi, John C. and Jordan, Michael I. and Wainwright, Martin J.},
title     = {{Local Privacy and Statistical Minimax Rates}}, 
booktitle = {2013 IEEE 54th Annual Symposium on Foundations of Computer Science (FOCS)}, 
pages     = {429--438},
year      = {2013}
}

@inproceedings{kairouz2016discrete,
author    = {Kairouz, Peter and Bonawitz, Keith and Ramage, Daniel},
title     = {{Discrete Distribution Estimation under Local Privacy}},
booktitle = {Proceedings of The 33rd International Conference on Machine Learning (ICML)},
pages     = {2436--2444},
volume    = {48},
year      = {2016}
}

@article{duchi2018minimax,
author    = {Duchi, John C. and Jordan, Michael I. and Wainwright, Martin J.},
title     = {{Minimax Optimal Procedures for Locally Private Estimation}},
journal   = {Journal of the American Statistical Association},
volume    = {113},
number    = {521},
pages     = {182--201},
year      = {2018}
}

@inproceedings{duan2022utility,
author    = {Duan, Jiawei and Ye, Qingqing and Hu, Haibo},
title     = {{Utility Analysis and Enhancement of LDP Mechanisms in High-Dimensional Space}}, 
booktitle = {2022 IEEE 38th International Conference on Data Engineering (ICDE)}, 
pages     = {407--419},
year      = {2022},
}

@inproceedings{wang2019multidimensional,
author    = {Wang, Ning and Xiao, Xiaokui and Yang, Yin and Zhao, Jun and Hui, Siu Cheung and Shin, Hyejin and Shin, Junbum and Yu, Ge},
title     = {{Collecting and Analyzing Multidimensional Data with Local Differential Privacy}}, 
booktitle = {2019 IEEE 35th International Conference on Data Engineering (ICDE)}, 
pages     = {638--649},
year      = {2019}
}

@article{bhowmick2018privunit,
author    = {Bhowmick, Abhishek and Duchi, John C. and Freudiger, Julien and Kapoor, Gaurav and Rogers, Ryan},
title     = {{Protection Against Reconstruction and Its Applications in Private Federated Learning}},
journal   = {arXiv preprint arXiv:1812.00984},
year      = {2018}
}

@inproceedings{asi2022privunitg,
author    = {Asi, Hilal and Feldman, Vitaly and Koren, Tomer and Talwar, Kunal},
title     = {{Optimal Algorithms for Mean Estimation under Local Differential Privacy}},
booktitle = {Proceedings of the 39th International Conference on Machine Learning (ICML)},
volume    = {162},
pages     = {1046--1056},
year      = {2022}
}

@inproceedings{cheng2022task,
author    = {Cheng, Jiangnan and Tang, Ao and Chinchali, Sandeep},
title     = {{Task-aware Privacy Preservation for Multi-dimensional Data}},
booktitle = {Proceedings of the 39th International Conference on Machine Learning (ICML)},
pages     = {3761--3782},
year      = {2022},
volume    = {162}
}

@inproceedings{phan2017adaptive,
author    = {Phan, NhatHai and Wu, Xintao and Hu, Han and Dou, Dejing},
title     = {{Adaptive Laplace Mechanism: Differential Privacy Preservation in Deep Learning}}, 
booktitle = {2017 IEEE International Conference on Data Mining (ICDM)}, 
year      = {2017},
pages     = {385--394}
}

@article{alaggan2016heterogeneous,
author    = {Alaggan, Mohammad and Gambs, S{\'e}bastien and Kermarrec, Anne-Marie},
title     = {{Heterogeneous Differential Privacy}},
journal   = {Journal of Privacy and Confidentiality},
volume    = {7},
number    = {2},
pages     = {127--158},
year      = {2016}
}

@article{muthukrishnan2025DP,
author    = {Muthukrishnan, Gokularam and Kalyani, Sheetal},
title     = {{Differential Privacy With Higher Utility by Exploiting Coordinate-Wise Disparity: Laplace Mechanism Can Beat Gaussian in High Dimensions}}, 
journal   = {IEEE Transactions on Information Forensics and Security}, 
volume    = {20},
pages     = {2836--2851},
year      = {2025}
}

@article{kifer2014pufferfish,
author    = {Kifer, Daniel and Machanavajjhala, Ashwin},
title     = {{Pufferfish: A Framework for Mathematical Privacy Definitions}},
journal   = {ACM Transactions on Database Systems},
volume    = {39},
number    = {1},
pages     = {3:1--3:36},
year      = {2014}
}

@article{aumuller2024plan,
author    = {Aum{\"u}ller, Martin and Lebeda, Christian Janos and Nelson, Boel and Pagh, Rasmus},
title     = {{PLAN: Variance-Aware Private Mean Estimation}},
journal   = {Proceedings on Privacy Enhancing Technologies},
volume    = {2024},
number    = {3},
pages     = {606--625},
year      = {2024}
}

@inproceedings{dagan2024dimensionfree,
author    = {Dagan, Yuval and Jordan, Michael I. and Yang, Xuelin and Zakynthinou, Lydia and Zhivotovskiy, Nikita},
title     = {{Dimension-free Private Mean Estimation for Anisotropic Distributions}},
booktitle = {Advances in Neural Information Processing Systems (NeurIPS)},
volume    = {37},
number    = {3835},
pages     = {120667--120698},
year      = {2024}
}

@article{hay2010boosting,
author    = {Hay, Michael and Rastogi, Vibhor and Miklau, Gerome and Suciu, Dan},
title     = {{Boosting the Accuracy of Differentially Private Histograms Through Consistency}},
journal   = {Proceedings of the VLDB Endowment},
volume    = {3},
number    = {1},
pages     = {1021--1032},
year      = {2010},
publisher = {VLDB Endowment}
}

@inproceedings{wang2017locally,
author    = {Tianhao Wang and Jeremiah Blocki and Ninghui Li and Somesh Jha},
title     = {{Locally Differentially Private Protocols for Frequency Estimation}},
booktitle = {26th USENIX Security Symposium (USENIX Security 17)},
pages     = {729--745},
year      = {2017}
}

@inproceedings{jia2019calibrate,
author    = {Jia, Jinyuan and Gong, Neil Zhenqiang},
title     = {{Calibrate: Frequency Estimation and Heavy Hitter Identification with Local Differential Privacy via Incorporating Prior Knowledge}}, 
booktitle = {IEEE INFOCOM 2019 - IEEE Conference on Computer Communications}, 
pages     = {2008--2016},
year      = {2019}
}

@article{cormode2019answering,
author    = {Cormode, Graham and Kulkarni, Tejas and Srivastava, Divesh},
title     = {{Answering Range Queries Under Local Differential Privacy}},
journal   = {Proceedings of the VLDB Endowment},
volume    = {12},
number    = {10},
pages     = {1126--1138},
year      = {2019}
}

@inproceedings{wang2020locally,
author    = {Wang, Tianhao and Lopuha{\"{a}}{-}Zwakenberg, Milan and Li, Zitao and Skoric, Boris and Li, Ninghui},
title     = {{Locally Differentially Private Frequency Estimation with Consistency}},
booktitle = {Proceedings of the 27th Network and Distributed System Security Symposium (NDSS)},
year      = {2020}
}

@inproceedings{sajadmanesh2021locally,
author    = {Sajadmanesh, Sina and Gatica-Perez, Daniel},
title     = {{Locally Private Graph Neural Networks}},
booktitle = {Proceedings of the 2021 ACM SIGSAC Conference on Computer and Communications Security (CCS)},
pages     = {2130--2145},
year      = {2021}
}

@inproceedings{fang2023convolution,
author    = {Fang, Huiyu and Chen, Liquan and Liu, Yali and Gao, Yuan},
title     = {{Locally Differentially Private Frequency Estimation Based on Convolution Framework}}, 
booktitle = {2023 IEEE Symposium on Security and Privacy (SP)}, 
pages     = {2208-2222},
year      = {2023}
}

@article{kurakin2022toward,
author    = {Kurakin, Alexey and Song, Shuang and Chien, Steve and Geambasu, Roxana and Terzis, Andreas and Thakurta, Abhradeep},
title     = {{Toward Training at ImageNet Scale with Differential Privacy}},
journal   = {arXiv preprint arXiv:2201.12328},
year      = {2022}
}

@inproceedings{nasr2023effectively,
author    = {Nasr, Milad and Mahloujifar, Saeed and Tang, Xinyu and Mittal, Prateek and Houmansadr, Amir},
title     = {{Effectively Using Public Data in Privacy Preserving Machine Learning}},
booktitle = {Proceedings of the 40th International Conference on Machine Learning (ICML)},
volume    = {202},
pages     = {25718--25732},
year      = {2023},
}

@inproceedings{hou2024pretext,
author    = {Hou, Charlie and Shrivastava, Akshat and Zhan, Hongyuan and Conway, Rylan and Le, Trang and Sagar, Adithya and Fanti, Giulia and Lazar, Daniel},
title     = {{PrE-Text: Training Language Models on Private Federated Data in the Age of LLMs}},
booktitle = {Proceedings of the 41st International Conference on Machine Learning (ICML)},
volume    = {235},
pages     = {19043--19061},
year      = {2024}
}

@article{dwork2014algorithmic,
author    = {Dwork, Cynthia and Roth, Aaron},
title     = {The Algorithmic Foundations of Differential Privacy},
journal   = {Foundations and Trends in Theoretical Computer Science},
volume    = {9},
number    = {3--4},
pages     = {211--407},
year      = {2014}
}

@inproceedings{balle2018improving,
author    = {Balle, Borja and Wang, Yu-Xiang},
title     = {{Improving the Gaussian Mechanism for Differential Privacy: Analytical Calibration and Optimal Denoising}},
booktitle = {Proceedings of the 35th International Conference on Machine Learning (ICML)},
volume    = {80},
pages     = {394--403},
year      = {2018}
}

@article{constantine2014active,
author    = {Constantine, Paul G. and Dow, Eric and Wang, Qiqi},
title     = {{Active Subspace Methods in Theory and Practice: Applications to Kriging Surfaces}},
journal   = {SIAM Journal on Scientific Computing},
volume    = {36},
number    = {4},
pages     = {A1500--A1524},
year      = {2014}
}

@misc{ukpowernetworks_lcl,
author    = {{UK Power Networks}},
title     = {SmartMeter Energy Consumption Data in London Households},
url       = {https://data.london.gov.uk/dataset/smartmeter-energy-use-data-in-london-households},
note      = {Low Carbon London project data}
}

@misc{jeanmidev2019smartmeters,
author       = {Jean-Michel D.},
title        = {{Smart Meters in London}},
year         = {2019},
publisher    = {Kaggle},
howpublished = {\url{https://www.kaggle.com/datasets/jeanmidev/smart-meters-in-london}},
note         = {Accessed: 2026}
}

@techreport{krizhevsky2009cifar,
author      = {Krizhevsky, Alex},
title       = {{Learning Multiple Layers of Features from Tiny Images}},
institution = {University of Toronto},
year        = {2009}
}

@article{lecun1998mnist,
author    = {Lecun, Y. and Bottou, L. and Bengio, Y. and Haffner, P.},
journal   = {Proceedings of the IEEE}, 
title     = {{Gradient-Based Learning Applied to Document Recognition}}, 
volume    = {86},
number    = {11},
pages     = {2278--2324},
year      = {1998}
}

@inproceedings{kingma2013auto,
author    = {Kingma, Diederik P and Welling, Max},
title     = {{Auto-Encoding Variational Bayes}},
booktitle = {International Conference on Learning Representations (ICLR)},
year      = {2014}
}

@inproceedings{he2016deep,
title     = {{Deep Residual Learning For Image Recognition}},
author    = {He, Kaiming and Zhang, Xiangyu and Ren, Shaoqing and Sun, Jian},
booktitle = {{Proceedings of the IEEE Conference on Computer Vision and Pattern Recognition (CVPR)}},
pages     = {770--778},
year      = {2016}
}

@inproceedings{hendrycks2019benchmarking,
title     = {{Benchmarking Neural Network Robustness to Common Corruptions and Perturbations}},
author    = {Dan Hendrycks and Thomas Dietterich},
booktitle = {International Conference on Learning Representations (ICLR)},
year      = {2019}
}

@inproceedings{huang2021instance,
title     = {{Instance-Optimal Mean Estimation Under Differential Privacy}},
author    = {Huang, Ziyue and Liang, Yuting and Yi, Ke},
booktitle = {Advances in Neural Information Processing Systems (NIPS)},
volume    = {34},
number    = {1990},
pages     = {25993--26004},
year      = {2021}
}

\clearpage

\appendix

\section{Appendix}

\subsection{Related Work} \label{appendix:1}
\textbf{Relevance-Aware.} A line of related work dynamically allocates privacy budgets based on feature importance. 
For instance, the Adaptive Laplace Mechanism (AdLM)~\citep{phan2017adaptive} employs Layer-wise Relevance Propagation (LRP) to assign larger privacy budgets to features that contribute more to the final prediction, thereby reducing noise on highly relevant dimensions. 
While both AdLM and our method share a task-aware intuition, they differ fundamentally in their mathematical foundations. 
AdLM relies on LRP to compute scalar attribution scores for individual features. 
In contrast, our approach leverages the Jacobian matrix to evaluate the sensitivity of the output with respect to the input. 
This mathematically grounded approach allows us to analytically decompose the input space into geometrically meaningful structures, i.e., the local row and null spaces, rather than merely assigning scalar weights to individual pixels.

\textbf{Correlation-Aware.} 
Several studies address DP vulnerabilities when the assumption of record independence is violated. 
For instance, Pufferfish~\citep{kifer2014pufferfish} incorporates joint probability distributions directly into the privacy definition to account for data dependencies. 
Similarly, PLAN~\citep{aumuller2024plan} injects anisotropic noise by allocating the privacy budget based on the data geometry inherent in a known covariance matrix, a method later extended to unknown covariances by \citep{dagan2024dimensionfree}. 
While these works~\citep{kifer2014pufferfish, aumuller2024plan, dagan2024dimensionfree} explore anisotropic noise injection guided solely by intrinsic data correlations, our approach fundamentally differs by scaling noise based on the task-dependent importance of subspaces utilized by the downstream model.

\textbf{Feature-Wise DP.}
Alaggan et al.~\citep{alaggan2016heterogeneous} introduced the Heterogeneous Differential Privacy (HDP), which captures the non-uniformity of privacy expectations both among different users and across various data items belonging to the same individual.
HDP enables item-grained privacy, allowing for the adaptive calibration of noise levels based on the specific sensitivity of each piece of information. 
Our approach primarily differs from HDP in its treatment of information sensitivity. 
Unlike HDP, which accounts for non-uniform sensitivity, we use an isotropically calibrated base randomizer and reshape its effective perturbation across task directions. 
Furthermore, we enhance overall utility by applying a bijective transformation that reduces the noise injected into key dimensions.

\textbf{Post-Processing.}
Enhancing utility under LDP often relies on post-processing techniques such as data aggregation, denoising, and statistical inference. 
Foundational mechanisms—including industry deployments by Google~\citep{erlingsson2014rappor}, Apple~\citep{apple2017learning}, and Microsoft~\citep{ding2017collecting}, as well as variance-optimized encodings like OUE/OLH~\citep{wang2017locally}—established standard aggregation pipelines. 
To further reduce estimation errors, recent works have incorporated structural priors, such as hierarchical consistency~\citep{hay2010boosting, cormode2019answering}, probability simplex constraints~\citep{wang2020locally}, iterative calibration~\citep{jia2019calibrate}, and signal-processing filters~\citep{fang2023convolution}. 
Among post-processing efforts for high-dimensional data, \citep{sajadmanesh2021locally} is the closest to our work, proposing KProp to denoise private node features in GNNs via neighborhood averaging. 
However, while they focus on structural aggregation during the model training phase, our research targets the inference phase, uniquely focusing on reducing the noise injected into the task-sensitive subspace.

\textbf{Public Data \& Model Assisted DP.}
Various studies have leveraged public data or pre-trained models to alleviate the utility loss inherent in Differential Privacy (DP). 
For instance, Kurakin et al.~\citep{kurakin2022toward} investigated the use of ResNet-18 models pre-trained on the Places365 dataset, demonstrating that such priors are vital for navigating the loss landscape under noise injection. 
Their findings underscore that leveraging feature representations learned from public data is a key factor in improving both convergence speed and final classification accuracy. 
Similarly, DOPE-SGD \citep{nasr2023effectively} integrates advanced data augmentation with a re-centered gradient clipping strategy informed by public priors. 
This method effectively mitigates the bias in private gradient estimation, leading to substantial utility gains. 
Furthermore, \citep{hou2024pretext} introduced a framework for generating DP synthetic text, showing that models trained on such data achieve superior performance with $100\times$ less communication overhead than on-device baselines. 
By establishing that synthetic data training provides a more favorable privacy-utility trade-off, their research validates the efficacy of using generative models to create private proxies for downstream tasks.
While these works predominantly center on optimizing the model training process, our approach diverges by focusing on enhancing utility during inference under LDP.

\newpage

\subsection{Motivating Example}
\label{appdx:motivating_example}
\begin{figure}[h]
     \centering
     \begin{subfigure}[b]{0.45\textwidth}
         \centering
         \includegraphics[width=\textwidth]{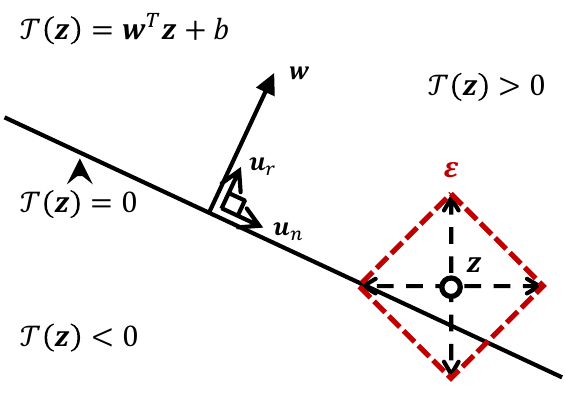}
         \caption{Linear Binary Classification in $\mathbb{R}^2$}
         \label{fig:fig_1a}
     \end{subfigure}
     \hfill
     \begin{subfigure}[b]{0.45\textwidth}
         \centering
         \includegraphics[width=\textwidth]{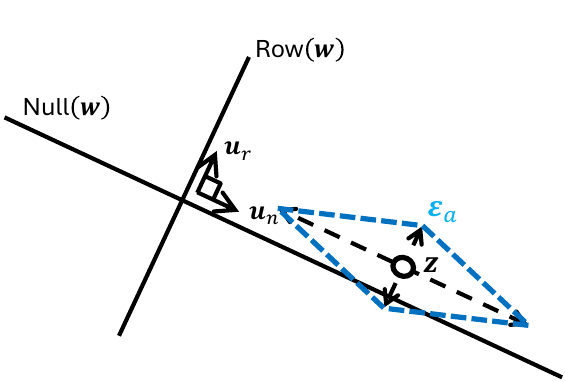}
         \caption{Anisotropically Reshaping Noise in $\mathbb{R}^2$}
         \label{fig:fig_1b}
     \end{subfigure}
     \caption{\textbf{Motivating Example.} $\mathcal{T}(\boldsymbol{z})$ is a linear classifier with normal vector $\boldsymbol{w}$ orthogonal to the decision boundary ($\mathcal{T}(\boldsymbol{z})=0$). $\boldsymbol{u}_{r}$ and $\boldsymbol{u}_{n}$ are basis vectors for the row and null spaces, respectively.
     Dotted arrows denote perturbations around $\boldsymbol{z}$.
     The red and blue dashed lines describe the shape of the Laplace noise.
     Notably, perturbations along $\boldsymbol{u}_{r}$ may alter the prediction, while those along $\boldsymbol{u}_{n}$ do not. 
     \textbf{(a) Isotropic noise.} 
     Isotropic noise $\boldsymbol{\xi}$ is generated. 
     \textbf{(b) Anisotropic Reshaping.} 
     Isotropic noise $\boldsymbol{\xi}$ is reshaped into $\boldsymbol{\xi}_a$ by attenuating the $\boldsymbol{u}_{r}$ axis and amplifying the $\boldsymbol{u}_{n}$ axis.     
     }
     \label{fig:fig_1}
\end{figure}

\subsection{Proposed Approach: Description}
\label{appdx:pa_description}

\begin{figure}[h]
  \centering
  \includegraphics[width=1.0\textwidth]{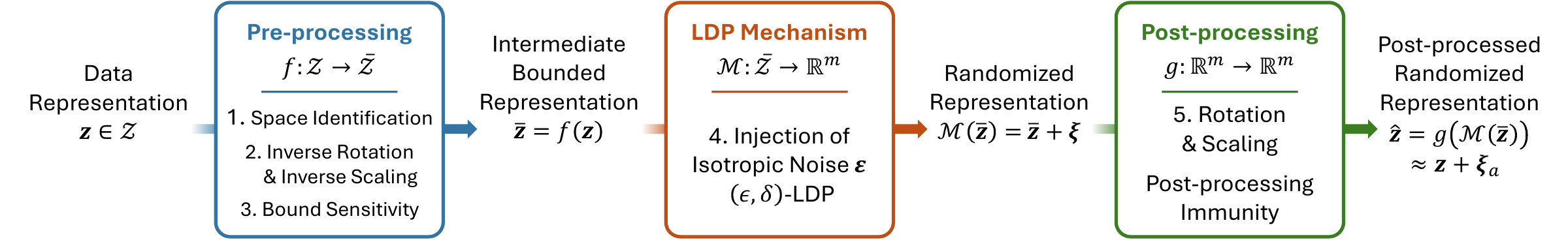} 
  \caption{Overview of the proposed approach.}
  \label{fig:overview}
\end{figure}

We provide a more detailed explanation of the procedure.
Here, we primarily focus on the row space for clarity.
Let $\lambda_r \in (0, 1]$ be a parameter that governs the relationship between the desired variance $\sigma_r^2$ of the noise projected into the row space via anisotropic noise $\boldsymbol{\xi}_a$, and the variance $\sigma^2$ of the isotropic noise $\boldsymbol{\xi}$, such that $\lambda_r = (\frac{\sigma_r}{\sigma})^2$.
The procedure consists of five steps: 
\begin{enumerate}
\item Our approach identifies the row and null spaces based on Jacobian matrices, 
\item It inversely rotates and expands the representation in the row space by a factor of $\frac{1}{\sqrt{\lambda_r}}$, 
\item It bounds the row space sensitivity, 
\item It injects isotropic noise guaranteeing $\epsilon$-LDP, and 
\item It rotates and scales the randomized representation down by a factor of $\sqrt{\lambda_r}$, thereby reshaping the isotropic noise into anisotropic noise and restoring the original scale and orientation of the representation.
\end{enumerate}

The pre-processing function $f$ handles the first three steps, and the post-processing function $g$ manages the final step. 
In the pre-processing step, while space identification and representation alignment are important, bounding sensitivity also plays a crucial role in reducing the noise injected into the row space.
Although bounding incurs utility loss, we assume that it is outweighed by the utility degradation caused by $\epsilon$-LDP randomization under high sensitivity (see Section~\ref{sec:ablation} for details). 
In the fourth step, we inject noise via the conventional mechanism to guarantee $\epsilon$-LDP for the predetermined sensitivity, which is proportional to the bounding threshold. 
In the final step, $g$ applies a fixed invertible affine transformation.
The privacy guarantee is preserved due to the post-processing immunity of differential privacy~\citep{dwork2014algorithmic}.

Note that in our approach, noise control is achieved through two key parameters: the bounding threshold $\rho \in \mathbb{R}^+$ and the ratio $\lambda_r$ of the coordinate-wise noise to the isotropic noise. 
The magnitude of isotropic noise is determined by the privacy budget, the threshold, and the chosen randomization mechanism. 
Finally, through the linear bijective rotation and scaling applied during post-processing, noise satisfying $\epsilon$-LDP is automatically distributed across all coordinates.

\begin{figure}[t]
    \centering
    \includegraphics[width=\textwidth]{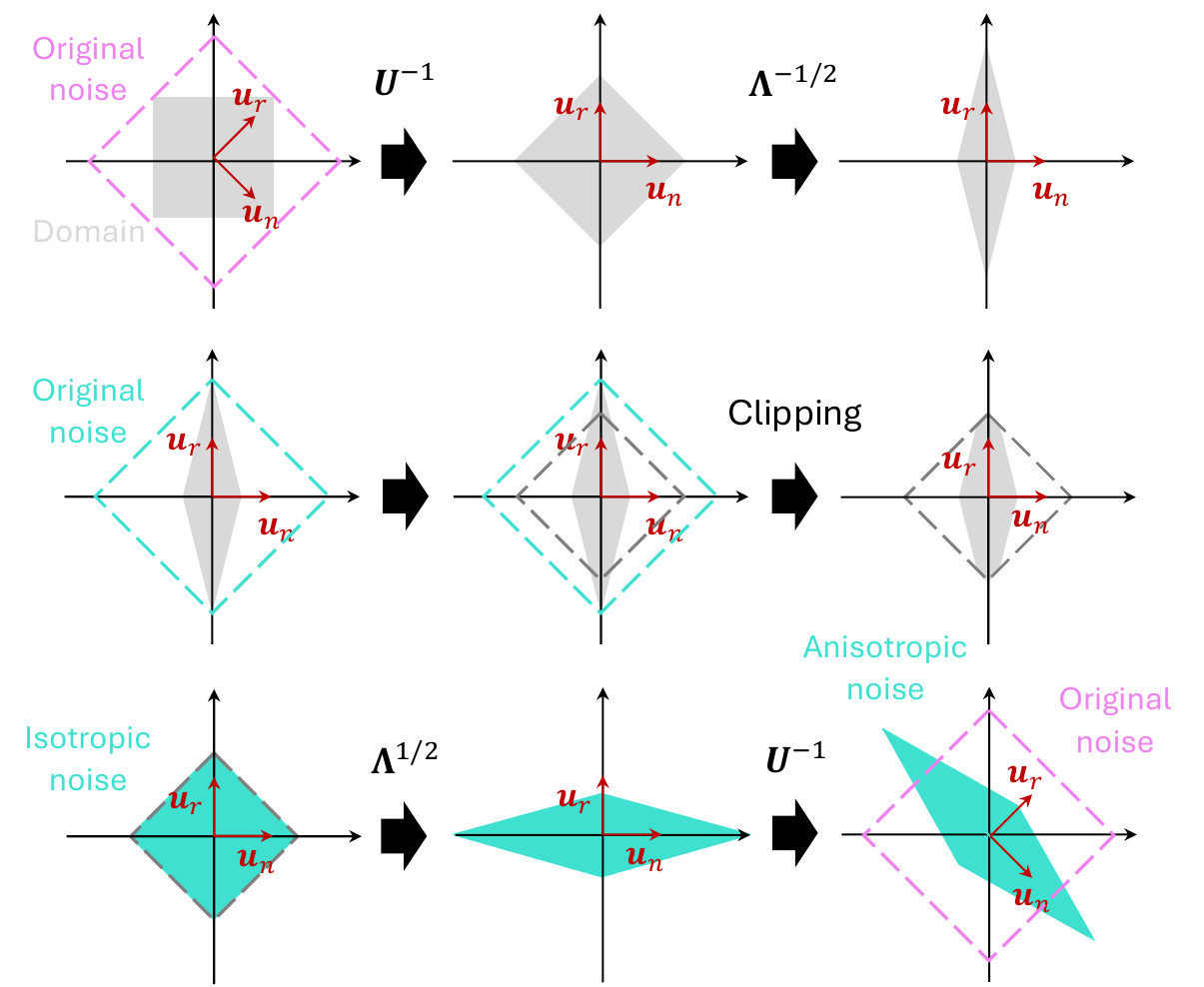} 
    \caption{Processing steps of the proposed approach.}
\end{figure}

\clearpage

\subsection{Theoretical Analysis}
\label{append:analysis}
We provide our analysis under the assumption of $\ell_1$ radial clipping and its corresponding geometric gauge $\sum_{i=1}^m \lambda_i^{-1/2} = m$.
The extension to general $\ell_p$ clipping and the gauge $\sum_{i=1}^m \lambda_{i}^{-p/2} = m$ can be derived analogously.

\begin{definition}[End-to-end perturbation for additive mechanisms]
\label{definition:perturbation}
Radial $\ell_1$ clipping in the transformed space scales $\boldsymbol{L}^{-1}(\boldsymbol{z}-\boldsymbol{\mu})$ by
\begin{equation}
\theta_1(\boldsymbol{z}) =
\begin{cases}
\displaystyle \min\left\{ 1,\, \frac{\rho}{\left\| \boldsymbol{L}^{-1} (\boldsymbol{z}-\boldsymbol{\mu}) \right\|_1} \right\}, & \boldsymbol{L}^{-1} (\boldsymbol{z}-\boldsymbol{\mu}) \neq \boldsymbol{0}, \\[10pt]
1, & \boldsymbol{L}^{-1} (\boldsymbol{z}-\boldsymbol{\mu}) = \boldsymbol{0}.
\end{cases} \nonumber
\end{equation}
This operation shrinks the centered representation toward $\boldsymbol{\mu}$ while preserving the direction of $\boldsymbol{z}-\boldsymbol{\mu}$.

For an additive mechanism, the $i$-th randomized representation can be written as
\begin{align}
\hat{\boldsymbol{z}}_i &= \boldsymbol{L}\left[\theta_1(\boldsymbol{z}_i)\left\{\boldsymbol{L}^{-1}(\boldsymbol{z}_i-\boldsymbol{\mu})\right\} + \boldsymbol{\xi} \right] + \boldsymbol{\mu} \notag \\
&= \theta_1(\boldsymbol{z}_i)\boldsymbol{L}\left\{\boldsymbol{L}^{-1}(\boldsymbol{z}_i-\boldsymbol{\mu})\right\} + \boldsymbol{L}\boldsymbol{\xi} + \boldsymbol{\mu} \notag \\
&= \theta_1(\boldsymbol{z}_i)\left(\boldsymbol{z}_i-\boldsymbol{\mu}\right) + \boldsymbol{L}\boldsymbol{\xi} + \boldsymbol{\mu} \nonumber
\end{align}
where $\boldsymbol{L}=\boldsymbol{U}\boldsymbol{\Lambda}^{1/2}$.

Note that the scaling factor of the $\ell_1$-ball clipping operation evaluates to a scalar given a representation $\boldsymbol{z}$ and the parameters $\{\boldsymbol{U},\boldsymbol{\Lambda},\boldsymbol{\mu},p,\rho\}$.

Therefore, the end-to-end perturbation is
\begin{equation}
\boldsymbol{h}_i = \hat{\boldsymbol{z}}_i-\boldsymbol{z}_i = \boldsymbol{d}_i+\boldsymbol{L}\boldsymbol{\xi}, \nonumber
\end{equation}
where
\begin{equation}
\boldsymbol{d}_i = \bigl(\theta_1(\boldsymbol{z}_i)-1\bigr) (\boldsymbol{z}_i-\boldsymbol{\mu}). \nonumber
\end{equation}
Here, $\boldsymbol{d}_i=\boldsymbol{0}$ for an unclipped sample and points toward $\boldsymbol{\mu}$ otherwise.
\end{definition}

\begin{definition}[Empirical Jacobian Gram second-moment matrix]
Define
\begin{equation}
\widehat{\boldsymbol{\Omega}} = \boldsymbol{J}^\top\boldsymbol{J} = \frac{1}{N} \sum_{i=1}^N \boldsymbol{J}_i^\top\boldsymbol{J}_i. \nonumber
\end{equation}
Let its eigendecomposition be
\begin{equation}
\widehat{\boldsymbol{\Omega}} = \boldsymbol{V} \operatorname{diag} (\gamma_1,\ldots,\gamma_m) \boldsymbol{V}^\top, \nonumber
\end{equation}
where $\gamma_1 \ge \cdots \ge \gamma_m \ge 0$.
\end{definition}

\begin{definition}[Empirical first-order risk surrogate]
Define
\begin{equation}
\widehat{R}_{\mathrm{FO}} = \frac{1}{N} \sum_{i=1}^N \mathbb{E}_{\boldsymbol{\xi}} \left[ \left\| \boldsymbol{J}_i\boldsymbol{h}_i \right\|_2^2 \right]. \nonumber
\end{equation}
Using $\boldsymbol{h}_i = \boldsymbol{d}_i + \boldsymbol{L}\boldsymbol{\xi}$, we obtain
\begin{equation}
\widehat{R}_{\mathrm{FO}} = \frac{1}{N} \sum_{i=1}^N \mathbb{E}_{\boldsymbol{\xi}} \left[ \left\| \boldsymbol{J}_i\boldsymbol{d}_i + \boldsymbol{J}_i \boldsymbol{L}\boldsymbol{\xi} \right\|_2^2 \right]. \nonumber
\end{equation}
Because $\boldsymbol{d}_i$ is fixed conditional on $\boldsymbol{z}_i$ and $\mathbb{E}[\boldsymbol{\xi}]=\boldsymbol{0}$, $\operatorname{Cov}(\boldsymbol{\xi}) = \sigma^2\boldsymbol{I}$, the cross term vanishes, yielding
\begin{equation}
\widehat{R}_{\mathrm{FO}} = \underbrace{ \sigma^2 \operatorname{Tr} \left( \boldsymbol{\Lambda} \boldsymbol{U}^\top \widehat{\boldsymbol{\Omega}} \boldsymbol{U} \right) }_{\widehat{R}_{\mathrm{FO,Noise}}} + \underbrace{ \frac{1}{N} \sum_{i=1}^N \left\| \boldsymbol{J}_i\boldsymbol{d}_i \right\|_2^2 }_{\widehat{R}_{\mathrm{FO,Clip}}}. \nonumber
\end{equation}
The following propositions optimize $\widehat{R}_{\mathrm{FO,Noise}}$, not the clipping term.
\end{definition}

\begin{definition}[Empirical nonlinear risk]
For a nonlinear task model $\mathcal{T}$, we define the empirical nonlinear (ENL) risk as
\begin{equation*}
\widehat{R}_{\mathrm{NL}} = \frac{1}{N} \sum_{i=1}^{N} \mathbb{E}_{\boldsymbol{\xi}} \left[ \left\| \mathcal{T}(\boldsymbol{z}_i+\boldsymbol{h}_i) - \mathcal{T}(\boldsymbol{z}_i) \right\|_2^2 \right].
\end{equation*}
\end{definition}

\begin{proposition}[Optimality of the Jacobian basis for the noise component]
\label{prop:optimality_of_the_Jacobian}
Let
\begin{equation}
\widehat{\boldsymbol{\Omega}} = \boldsymbol{V} \operatorname{diag} (\gamma_1,\ldots,\gamma_m) \boldsymbol{V}^\top, \qquad \gamma_1 \ge \cdots \ge \gamma_m \ge 0.\nonumber
\end{equation}
Fix an ordered positive scale spectrum $0 < \lambda_1 \le \cdots \le \lambda_m$ and define $\boldsymbol{\Lambda} = \operatorname{diag} (\lambda_1,\ldots,\lambda_m)$. Then
\begin{equation}
\min_{\boldsymbol{U}^\top\boldsymbol{U}=\boldsymbol{I}} \operatorname{Tr} \left( \boldsymbol{\Lambda} \boldsymbol{U}^\top \widehat{\boldsymbol{\Omega}} \boldsymbol{U} \right) = \sum_{i=1}^m \lambda_i\gamma_i. \nonumber
\end{equation}
One minimizer is $\boldsymbol{U}=\boldsymbol{V}$, with the largest Jacobian-Gram eigenvalues paired with the smallest scales. Equivalently, for a fixed ordered spectrum, the Jacobian right-singular-vector basis is optimal over all orthogonal bases for the noise-induced first-order objective. Non-uniqueness may arise from sign changes and transformations within repeated-eigenvalue or equal-scale blocks.
\end{proposition}

\begin{proof}
Let $\boldsymbol{H} = \boldsymbol{V}^\top\boldsymbol{U}$, which is orthogonal. Then
\begin{equation}
\operatorname{Tr} \left( \boldsymbol{\Lambda} \boldsymbol{U}^\top \widehat{\boldsymbol{\Omega}} \boldsymbol{U} \right) = \operatorname{Tr} \left( \boldsymbol{\Lambda} \boldsymbol{H}^\top \operatorname{diag} (\gamma_1,\ldots,\gamma_m) \boldsymbol{H} \right) = \sum_{i,j} \lambda_i\gamma_j H_{ji}^2. \nonumber
\end{equation}
Define $G_{ij}=H_{ji}^2$. Because $\boldsymbol{H}$ is orthogonal, $\boldsymbol{G}$ is doubly stochastic. Relaxing the feasible set of such orthostochastic matrices to the full Birkhoff polytope therefore provides a lower bound.

The relaxed objective is linear in $\boldsymbol{G}$, so its minimum is attained at an extreme point of the Birkhoff polytope, namely a permutation matrix. By the rearrangement inequality, the minimizing permutation pairs the largest $\gamma_i$ with the smallest $\lambda_i$, giving $\sum_{i=1}^m\lambda_i\gamma_i$.

This lower bound is attainable in the original feasible set because permutation matrices are orthogonal. Under the stated ordering, $\boldsymbol{H}=\boldsymbol{I}$, equivalently $\boldsymbol{U}=\boldsymbol{V}$, attains the bound. Hence the relaxation is tight.
\end{proof}

\begin{remark}[Rank deficiency and the finite cap]
When $\gamma_i=0$, the noise objective does not depend directly on $\lambda_i$. If the scales are optimized under $\sum_{j=1}^m\lambda_j^{-1/2}=m$ without an upper bound, the infimum sends the corresponding scale toward $\lambda_i\rightarrow\infty$. Equivalently, $t_i=\lambda_i^{-1/2}\rightarrow 0$, which leaves more normalization budget for directions with positive $\gamma_i$. A finite optimizer therefore need not exist in an exact null direction.

The cap $\lambda_i\le\lambda_{\max}$ removes this degeneracy. For a near-zero but positive $\gamma_i$, the uncapped optimal scale is finite but may be very large and sensitive to estimation error. In that case, the cap additionally limits extreme anisotropy and improves numerical robustness.
\end{remark}

\begin{proposition}[Joint optimality of the Jacobian basis and scale allocation]
\label{prop:joint_opt_Jac_scale}
Assume $1 < \lambda_{\max} < \infty$ and $\widehat{\boldsymbol{\Omega}} \neq \boldsymbol{0}$. 
Consider
\begin{equation}
\begin{aligned}
\min_{\substack{ \boldsymbol{U}^\top\boldsymbol{U}=\boldsymbol{I}\\ 0<\lambda_i\le\lambda_{\max} }} \quad& \sigma^2 \operatorname{Tr} \left( \boldsymbol{\Lambda} \boldsymbol{U}^\top \widehat{\boldsymbol{\Omega}} \boldsymbol{U} \right) \qquad
\text{subject to}\quad& \sum_{i=1}^m \lambda_i^{-1/2} = m. \nonumber
\end{aligned}
\end{equation}
A simultaneous permutation of the columns of $\boldsymbol{U}$ and the diagonal entries of $\boldsymbol{\Lambda}$ does not change the optimization problem. 
We may therefore assume, without loss of generality, that $0 < \lambda_1 \le \cdots \le \lambda_m$. 
By Proposition~\ref{prop:optimality_of_the_Jacobian}, the minimization over $\boldsymbol{U}$ is attained by the Jacobian basis and reduces the joint problem to
\begin{equation}
\min_{\boldsymbol{\Lambda}} \quad \sigma^2 \sum_{i=1}^m \gamma_i\lambda_i. \nonumber
\end{equation}
Set $t_i=\lambda_i^{-1/2}$ and $t_{\min}=\lambda_{\max}^{-1/2}$. 
The problem becomes
\begin{equation}
\begin{aligned}
\min_{\boldsymbol{t}} \quad& \sum_{i=1}^m \gamma_i t_i^{-2} \qquad
\text{subject to}\quad& \sum_{i=1}^m t_i=m, \qquad t_i \ge t_{\min}. \nonumber
\end{aligned}
\end{equation}
This is a convex optimization problem. Its solution has the water-filling form
\begin{equation}
t_i^\star = \max \left\{ t_{\min}, \alpha \gamma_i^{1/3} \right\}, \qquad \lambda_i^\star = (t_i^\star)^{-2}, \nonumber
\end{equation}
where $\alpha>0$ is chosen so that
\begin{equation}
\sum_{i=1}^m t_i^\star = \sum_{i=1}^m \max \left\{ t_{\min}, \alpha \gamma_i^{1/3} \right\} = m. \nonumber
\end{equation}
Let the resulting capped and free sets be
\begin{equation}
C = \left\{ i: \alpha \gamma_i^{1/3} \le t_{\min} \right\}, \qquad F=C^c. \nonumber
\end{equation}
The normalization constraint then gives $|C|t_{\min} + \alpha\sum_{j\in F}\gamma_j^{1/3} = m$. Therefore,
\begin{equation}
\alpha = \frac{ m-|C|t_{\min} }{ \sum_{j\in F}\gamma_j^{1/3} } = \frac{ m-|C|\lambda_{\max}^{-1/2} }{ \sum_{j\in F}\gamma_j^{1/3} }. \nonumber
\end{equation}
Consequently, the optimal scales can be written as
\begin{equation}
\lambda_i^\star =
\begin{cases}
\lambda_{\max}, & i\in C, \\[10pt]
\displaystyle \left( \frac{ m-|C|\lambda_{\max}^{-1/2} }{ \sum_{j\in F}\gamma_j^{1/3} } \gamma_i^{1/3} \right)^{-2}, & i\in F. \nonumber
\end{cases}
\end{equation}
Thus, on the free coordinates, $\lambda_i^\star \propto \gamma_i^{-2/3}$, whereas zero and sufficiently small singular-value directions receive the finite cap. The resulting scale spectrum is ordered oppositely to the Jacobian-Gram eigenvalues, as required by Proposition~\ref{prop:optimality_of_the_Jacobian}.

When no coordinate is capped, $C=\varnothing$ and $F=\{1,\ldots,m\}$, so $\alpha = \frac{m}{\sum_{j=1}^m \gamma_j^{1/3}}$ and
\begin{equation}
\lambda_i^\star = \left( \frac{ m\,\gamma_i^{1/3} }{ \sum_{j=1}^m \gamma_j^{1/3} } \right)^{-2}. \nonumber
\end{equation}
\end{proposition}

\begin{proof}
The Lagrangian is
\begin{equation}
\mathcal{L} = \sum_i \gamma_i t_i^{-2} + \eta \left( \sum_i t_i-m \right) + \sum_i\nu_i(t_{\min}-t_i), \nonumber
\end{equation}
where $\nu_i\ge 0$. The stationarity condition is
\begin{equation}
-2\gamma_i t_i^{-3} + \eta - \nu_i = 0. \nonumber
\end{equation}
For a free coordinate, $\nu_i=0$, so
\begin{equation}
t_i = \left( \frac{2\gamma_i}{\eta} \right)^{1/3} = \alpha \gamma_i^{1/3}. \nonumber
\end{equation}
For a capped coordinate, $t_i=t_{\min}$, and complementary slackness gives
\begin{equation}
\nu_i = \eta - 2\gamma_i t_{\min}^{-3} \ge 0. \nonumber
\end{equation}
Hence capping occurs for sufficiently small $\gamma_i$. Enforcing the normalization constraint gives
\begin{equation}
\alpha = \frac{ m-|C|t_{\min} }{ \sum_{j\in F}\gamma_j^{1/3} }. \nonumber
\end{equation}
Because $\lambda_{\max}>1$ and at least one $\gamma_i>0$, the equation determining $\alpha$ has a unique solution. Strict convexity gives uniqueness on coordinates with $\gamma_i>0$. For $\gamma_i=0$, setting $t_i=t_{\min}$, equivalently $\lambda_i=\lambda_{\max}$, leaves the largest possible normalization budget for the positive-singular-value coordinates and is therefore optimal.
\end{proof}

\begin{proposition}[Remainder bound for the empirical first-order risk surrogate]
\label{prop:remainder_bound}
Let $\mathcal{T}$ be a piecewise-affine ReLU network that is globally $K$-Lipschitz and differentiable at each $\boldsymbol{z}_i$. For each $i$, let $\mathcal{R}_i$ be a closed convex polyhedral activation cell containing $\boldsymbol{z}_i$, on which $\mathcal{T}$ is affine with Jacobian $\boldsymbol{J}_i$, and define the nonlinear remainder $\boldsymbol{\zeta}_i$ by
\begin{equation*}
\mathcal{T}(\boldsymbol{z}_i+\boldsymbol{h}_i) - \mathcal{T}(\boldsymbol{z}_i) = \boldsymbol{J}_i\boldsymbol{h}_i + \boldsymbol{\zeta}_i.
\end{equation*}

Let $\mathcal{A}_i$ denote the event
\begin{equation*}
\mathcal{A}_i = \left\{ \boldsymbol{z}_i+t\boldsymbol{h}_i\in\mathcal{R}_i \ \text{ for all } t\in[0,1] \right\},
\end{equation*}
and define
\begin{equation*}
\bar{P} = \frac{1}{N} \sum_{i=1}^{N} \mathrm{Pr}_{\boldsymbol{\xi}}(\mathcal{A}_i^c).
\end{equation*}

Since $\mathcal{R}_i$ is convex and $\boldsymbol{z}_i\in\mathcal{R}_i$, the event $\mathcal{A}_i$ is equivalently characterized by its endpoint, $\mathcal{A}_i=\{\boldsymbol{z}_i+\boldsymbol{h}_i\in\mathcal{R}_i\}$. On $\mathcal{A}_i$ the map $\mathcal{T}$ agrees with its affine restriction on $\mathcal{R}_i$, so that $\boldsymbol{\zeta}_i=\boldsymbol{0}$; the first-order expression is exact whenever the perturbation stays within $\mathcal{R}_i$.

Assume that
\begin{equation*}
\mathbb{E}_{\boldsymbol{\xi}} \left[ \|\boldsymbol{h}_i\|_2^4 \right] <\infty \qquad \text{for all } i\in\{1,\ldots,N\}.
\end{equation*}

Then
\begin{equation*}
\left| \sqrt{\widehat{R}_{\mathrm{NL}}} - \sqrt{\widehat{R}_{\mathrm{FO}}} \right| \le 2K \left( \frac{1}{N} \sum_{i=1}^{N} \mathbb{E}_{\boldsymbol{\xi}} \left[ \|\boldsymbol{h}_i\|_2^2 \mathbf{1}_{\mathcal{A}_i^c} \right] \right)^{1/2},
\end{equation*}
and consequently,
\begin{equation*}
\left| \sqrt{\widehat{R}_{\mathrm{NL}}} - \sqrt{\widehat{R}_{\mathrm{FO}}} \right| \le 2K \left( \frac{1}{N} \sum_{i=1}^{N} \mathbb{E}_{\boldsymbol{\xi}} \left[ \|\boldsymbol{h}_i\|_2^4 \right] \right)^{1/4} \bar{P}^{1/4}.
\end{equation*}
\end{proposition}

\begin{proof}
For a sequence of random vectors $\{\boldsymbol{v}_i\}_{i=1}^{N}$, define
\begin{equation*}
\left\| \{\boldsymbol{v}_i\}_{i=1}^{N} \right\|_{\mathcal{H}} = \left( \frac{1}{N} \sum_{i=1}^{N} \mathbb{E}_{\boldsymbol{\xi}} \left[ \|\boldsymbol{v}_i\|_2^2 \right] \right)^{1/2}.
\end{equation*}

This is the root-mean-square norm with respect to the joint average over the empirical index $i$ (uniform on $\{1,\ldots,N\}$) and the randomness of $\boldsymbol{\xi}$. Under this norm,
\begin{equation*}
\sqrt{\widehat{R}_{\mathrm{NL}}} = \left\| \left\{ \mathcal{T}(\boldsymbol{z}_i+\boldsymbol{h}_i) - \mathcal{T}(\boldsymbol{z}_i) \right\}_{i=1}^{N} \right\|_{\mathcal{H}},
\qquad
\sqrt{\widehat{R}_{\mathrm{FO}}} = \left\| \left\{ \boldsymbol{J}_i\boldsymbol{h}_i \right\}_{i=1}^{N} \right\|_{\mathcal{H}}.
\end{equation*}

Since $\mathcal{T}$ is globally $K$-Lipschitz and differentiable at $\boldsymbol{z}_i$, its Jacobian satisfies
\begin{equation*}
\|\boldsymbol{J}_i\|_{\mathrm{op}} = \sup_{\boldsymbol{u}\neq\boldsymbol{0}} \frac{\|\boldsymbol{J}_i\boldsymbol{u}\|_2}{\|\boldsymbol{u}\|_2} \le K,
\end{equation*}
so both $\|\mathcal{T}(\boldsymbol{z}_i+\boldsymbol{h}_i)-\mathcal{T}(\boldsymbol{z}_i)\|_2$ and $\|\boldsymbol{J}_i\boldsymbol{h}_i\|_2$ are bounded by $K\|\boldsymbol{h}_i\|_2$. Since $\mathbb{E}_{\boldsymbol{\xi}}[\|\boldsymbol{h}_i\|_2^4]<\infty$ implies $\mathbb{E}_{\boldsymbol{\xi}}[\|\boldsymbol{h}_i\|_2^2]<\infty$, both $\widehat{R}_{\mathrm{NL}}$ and $\widehat{R}_{\mathrm{FO}}$ are finite.

By the reverse triangle inequality,
\begin{equation}
\left| \sqrt{\widehat{R}_{\mathrm{NL}}} - \sqrt{\widehat{R}_{\mathrm{FO}}} \right| \le \left\| \left\{ \boldsymbol{\zeta}_i \right\}_{i=1}^{N} \right\|_{\mathcal{H}} = \left( \frac{1}{N} \sum_{i=1}^{N} \mathbb{E}_{\boldsymbol{\xi}} \left[ \|\boldsymbol{\zeta}_i\|_2^2 \right] \right)^{1/2}.
\nonumber
\end{equation}

For the remainder itself, the triangle inequality together with the two bounds above gives the deterministic worst case
\begin{equation}
\|\boldsymbol{\zeta}_i\|_2 = \left\| \mathcal{T}(\boldsymbol{z}_i+\boldsymbol{h}_i) - \mathcal{T}(\boldsymbol{z}_i) - \boldsymbol{J}_i\boldsymbol{h}_i \right\|_2 \le \left\| \mathcal{T}(\boldsymbol{z}_i+\boldsymbol{h}_i) - \mathcal{T}(\boldsymbol{z}_i) \right\|_2 + \|\boldsymbol{J}_i\boldsymbol{h}_i\|_2 \le 2K\|\boldsymbol{h}_i\|_2.\nonumber
\end{equation}

Since $\boldsymbol{\zeta}_i=\boldsymbol{0}$ on $\mathcal{A}_i$, this worst case is only active on the complement:
\begin{equation*}
\|\boldsymbol{\zeta}_i\|_2^2 \le 4K^2 \|\boldsymbol{h}_i\|_2^2 \mathbf{1}_{\mathcal{A}_i^c}.
\end{equation*}

Substituting into the preceding reverse-triangle bound gives
\begin{equation*}
\left| \sqrt{\widehat{R}_{\mathrm{NL}}} - \sqrt{\widehat{R}_{\mathrm{FO}}} \right| \le 2K \left( \frac{1}{N} \sum_{i=1}^{N} \mathbb{E}_{\boldsymbol{\xi}} \left[ \|\boldsymbol{h}_i\|_2^2 \mathbf{1}_{\mathcal{A}_i^c} \right] \right)^{1/2}.
\end{equation*}

Finally, since $\|\cdot\|_{\mathcal{H}}$ is the root-mean-square norm for this joint average, the Cauchy–Schwarz inequality applies to the pair $\|\boldsymbol{h}_i\|_2^2$ and $\mathbf{1}_{\mathcal{A}_i^c}$ over the same average; using $\mathbf{1}_{\mathcal{A}_i^c}^2=\mathbf{1}_{\mathcal{A}_i^c}$,
\begin{align*}
\frac{1}{N} \sum_{i=1}^{N} \mathbb{E}_{\boldsymbol{\xi}} \left[ \|\boldsymbol{h}_i\|_2^2 \mathbf{1}_{\mathcal{A}_i^c} \right] &\le \left( \frac{1}{N} \sum_{i=1}^{N} \mathbb{E}_{\boldsymbol{\xi}} \left[ \|\boldsymbol{h}_i\|_2^4 \right] \right)^{1/2} \left( \frac{1}{N} \sum_{i=1}^{N} \mathbb{E}_{\boldsymbol{\xi}} \left[ \mathbf{1}_{\mathcal{A}_i^c} \right] \right)^{1/2} \\
&= \left( \frac{1}{N} \sum_{i=1}^{N} \mathbb{E}_{\boldsymbol{\xi}} \left[ \|\boldsymbol{h}_i\|_2^4 \right] \right)^{1/2} \bar{P}^{1/2}.
\end{align*}

Taking the outer square root yields
\begin{equation*}
\left| \sqrt{\widehat{R}_{\mathrm{NL}}} - \sqrt{\widehat{R}_{\mathrm{FO}}} \right| \le 2K \left( \frac{1}{N} \sum_{i=1}^{N} \mathbb{E}_{\boldsymbol{\xi}} \left[ \|\boldsymbol{h}_i\|_2^4 \right] \right)^{1/4} \bar{P}^{1/4},
\end{equation*}
which completes the proof.
\end{proof}

\begin{corollary}[Direct bound on the empirical nonlinear-risk discrepancy]
Under the assumptions of Proposition~\ref{prop:remainder_bound}, define
\begin{equation}\nonumber
    B = 2K \left( \frac{1}{N} \sum_{i=1}^{N} \mathbb{E}_{\boldsymbol{\xi}} \left[ \|\boldsymbol{h}_i\|_2^4 \right] \right)^{1/4} \bar{P}^{1/4}.
\end{equation}
Then
\begin{equation}\nonumber
    \left| \widehat{R}_{\mathrm{NL}} - \widehat{R}_{\mathrm{FO}} \right| \le 2B\sqrt{\widehat{R}_{\mathrm{FO}}} + B^2.
\end{equation}
\end{corollary}

\begin{proof}
By Proposition~\ref{prop:remainder_bound},
\begin{equation}\nonumber
    \left| \sqrt{\widehat{R}_{\mathrm{NL}}} - \sqrt{\widehat{R}_{\mathrm{FO}}} \right| \le B.
\end{equation}
Hence,
\begin{equation}
    \left| \widehat{R}_{\mathrm{NL}} - \widehat{R}_{\mathrm{FO}} \right| 
    = \left| \sqrt{\widehat{R}_{\mathrm{NL}}} - \sqrt{\widehat{R}_{\mathrm{FO}}} \right| \left( \sqrt{\widehat{R}_{\mathrm{NL}}} + \sqrt{\widehat{R}_{\mathrm{FO}}} \right) \le B \left( 2\sqrt{\widehat{R}_{\mathrm{FO}}} + B \right), \nonumber
\end{equation}
where the last inequality follows from
\begin{equation}\nonumber
    \sqrt{\widehat{R}_{\mathrm{NL}}} \le \sqrt{\widehat{R}_{\mathrm{FO}}} + B.
\end{equation}
This proves the result.
\end{proof}

\begin{proposition}[Deterministic bound for the empirical noise-risk discrepancy]
Let $\boldsymbol{\Omega}$ denote the population Jacobian second-moment matrix and let $\widehat{\boldsymbol{\Omega}}$ denote its empirical estimate obtained from $N$ public samples. 

Define
\begin{equation*}
R_{\mathrm{FO,Noise}} = \sigma^2 \operatorname{Tr} \left( \boldsymbol{\Lambda} \boldsymbol{U}^{\top} \boldsymbol{\Omega} \boldsymbol{U} \right)
\end{equation*}
and
\begin{equation*}
\widehat{R}_{\mathrm{FO,Noise}} = \sigma^2 \operatorname{Tr} \left( \boldsymbol{\Lambda} \boldsymbol{U}^{\top} \widehat{\boldsymbol{\Omega}} \boldsymbol{U} \right).
\end{equation*}

Then,
\begin{align*}
\left| R_{\mathrm{FO,Noise}} - \widehat{R}_{\mathrm{FO,Noise}} \right| = \sigma^2 \left| \operatorname{Tr} \left( \boldsymbol{\Lambda} \boldsymbol{U}^{\top} \left( \boldsymbol{\Omega} - \widehat{\boldsymbol{\Omega}} \right) \boldsymbol{U} \right) \right| \le \sigma^2 \left\| \boldsymbol{\Omega} - \widehat{\boldsymbol{\Omega}} \right\|_{\mathrm{op}} \operatorname{Tr}(\boldsymbol{\Lambda}).
\end{align*}
\end{proposition}

\begin{proof}
Using cyclic invariance of the trace,
\begin{equation*}
\operatorname{Tr} \left( \boldsymbol{\Lambda} \boldsymbol{U}^{\top} \left( \boldsymbol{\Omega} - \widehat{\boldsymbol{\Omega}} \right) \boldsymbol{U} \right) = \operatorname{Tr} \left( \boldsymbol{U} \boldsymbol{\Lambda} \boldsymbol{U}^{\top} \left( \boldsymbol{\Omega} - \widehat{\boldsymbol{\Omega}} \right) \right).
\end{equation*}

By the duality between the nuclear and operator norms,
\begin{equation*}
\left| \operatorname{Tr} \left( \boldsymbol{U} \boldsymbol{\Lambda} \boldsymbol{U}^{\top} \left( \boldsymbol{\Omega} - \widehat{\boldsymbol{\Omega}} \right) \right) \right| \le \left\| \boldsymbol{U} \boldsymbol{\Lambda} \boldsymbol{U}^{\top} \right\|_{*} \|\left( \boldsymbol{\Omega} - \widehat{\boldsymbol{\Omega}} \right)\|_{\mathrm{op}}.
\end{equation*}

Since $\boldsymbol{U}$ is orthogonal and $\boldsymbol{\Lambda}\succeq 0$,
\begin{equation*}
\left\| \boldsymbol{U} \boldsymbol{\Lambda} \boldsymbol{U}^{\top} \right\|_{*} = \operatorname{Tr}(\boldsymbol{\Lambda}).
\end{equation*}

Multiplying by $\sigma^2$ gives the result.
\end{proof}

\begin{corollary}[Finite-sample convergence of the empirical noise risk]
Suppose that the public samples are independently and identically distributed and that the Jacobian second-moment estimator satisfies
\begin{equation*}
\left\| \boldsymbol{\Omega} - \widehat{\boldsymbol{\Omega}} \right\|_{\mathrm{op}} = \mathcal{O}_{P} \left( \sqrt{\frac{\log m}{N}} \right),
\end{equation*}
as obtained, for example, under appropriate boundedness or sub-Gaussian assumptions through a matrix concentration inequality.

Then,
\begin{equation*}
\left| R_{\mathrm{FO,Noise}} - \widehat{R}_{\mathrm{FO,Noise}} \right| = \mathcal{O}_{P} \left( \sigma^2 \operatorname{Tr}(\boldsymbol{\Lambda}) \sqrt{\frac{\log m}{N}} \right).
\end{equation*}
\end{corollary}

\begin{proposition}[Public–private discrepancy bound for the empirical first-order noise risk]
\label{prop:public_private_discrepancy}
Let $\boldsymbol{\Omega}_{\mathrm{pub/priv}}$ denote the Jacobian second-moment matrix under the public/private population.

Let
\begin{equation*}
\widehat{\boldsymbol{\Omega}}_{\mathrm{pub}} = \boldsymbol{J}_{\mathrm{pub}}^{\top} \boldsymbol{J}_{\mathrm{pub}} = \frac{1}{N}\sum_{i=1}^{N} \boldsymbol{J}_{\mathrm{pub},i}^{\top} \boldsymbol{J}_{\mathrm{pub},i}
\end{equation*}
denote the empirical public Jacobian second-moment matrix. 

Define the empirical public first-order noise risk as
\begin{equation*}
\widehat{R}_{\mathrm{pub,FO,Noise}} = \sigma^{2} \operatorname{Tr} \left( \boldsymbol{\Lambda}_{\mathrm{pub}} \boldsymbol{U}_{\mathrm{pub}}^{\top} \widehat{\boldsymbol{\Omega}}_{\mathrm{pub}} \boldsymbol{U}_{\mathrm{pub}} \right),
\end{equation*}
and define the corresponding private-population noise risk, evaluated using the public basis and scale allocation, as
\begin{equation*}
R_{\mathrm{priv,FO,Noise}} = \sigma^{2} \operatorname{Tr} \left( \boldsymbol{\Lambda}_{\mathrm{pub}} \boldsymbol{U}_{\mathrm{pub}}^{\top} \boldsymbol{\Omega}_{\mathrm{priv}} \boldsymbol{U}_{\mathrm{pub}} \right).
\end{equation*}

Then the public–private discrepancy bound for the empirical first-order noise risk is given by 
\begin{equation*}
\left| R_{\mathrm{priv,FO,Noise}} - \widehat{R}_{\mathrm{pub,FO,Noise}} \right| \le \sigma^{2} \left\| \boldsymbol{\Omega}_{\mathrm{priv}} - \widehat{\boldsymbol{\Omega}}_{\mathrm{pub}} \right\|_{\mathrm{op}} \operatorname{Tr} \left( \boldsymbol{\Lambda}_{\mathrm{pub}} \right).
\end{equation*}
where 
\begin{align*}
\left\| \boldsymbol{\Omega}_{\mathrm{priv}} - \widehat{\boldsymbol{\Omega}}_{\mathrm{pub}} \right\|_{\mathrm{op}} &\le \left\| \boldsymbol{\Omega}_{\mathrm{priv}} - \boldsymbol{\Omega}_{\mathrm{pub}} \right\|_{\mathrm{op}} + \left\| \boldsymbol{\Omega}_{\mathrm{pub}} - \widehat{\boldsymbol{\Omega}}_{\mathrm{pub}} \right\|_{\mathrm{op}}.
\end{align*}
\end{proposition}

\begin{proof}
By the definitions above,
\begin{equation*}
\left| R_{\mathrm{priv,FO,Noise}} - \widehat{R}_{\mathrm{pub,FO,Noise}} \right| = \sigma^{2} \left| \operatorname{Tr} \left( \boldsymbol{\Lambda}_{\mathrm{pub}} \boldsymbol{U}_{\mathrm{pub}}^{\top} \left( \boldsymbol{\Omega}_{\mathrm{priv}} - \widehat{\boldsymbol{\Omega}}_{\mathrm{pub}} \right) \boldsymbol{U}_{\mathrm{pub}} \right) \right|.
\end{equation*}

Using cyclic invariance of the trace and the duality between the nuclear and operator norms,
\begin{align*}
\left| \operatorname{Tr} \left( \boldsymbol{\Lambda}_{\mathrm{pub}} \boldsymbol{U}_{\mathrm{pub}}^{\top} \left( \boldsymbol{\Omega}_{\mathrm{priv}} - \widehat{\boldsymbol{\Omega}}_{\mathrm{pub}} \right) \boldsymbol{U}_{\mathrm{pub}} \right) \right| &\le \left\| \boldsymbol{U}_{\mathrm{pub}} \boldsymbol{\Lambda}_{\mathrm{pub}} \boldsymbol{U}_{\mathrm{pub}}^{\top} \right\|_{*} \left\| \boldsymbol{\Omega}_{\mathrm{priv}} - \widehat{\boldsymbol{\Omega}}_{\mathrm{pub}} \right\|_{\mathrm{op}}.
\end{align*}

Since $\boldsymbol{U}_{\mathrm{pub}}$ is orthogonal and $\boldsymbol{\Lambda}_{\mathrm{pub}}\succ\boldsymbol{0}$,
\begin{equation*}
\left\| \boldsymbol{U}_{\mathrm{pub}} \boldsymbol{\Lambda}_{\mathrm{pub}} \boldsymbol{U}_{\mathrm{pub}}^{\top} \right\|_{*} = \operatorname{Tr} \left( \boldsymbol{\Lambda}_{\mathrm{pub}} \right).
\end{equation*}
This proves the first inequality. The second follows directly from the triangle inequality for the operator norm.
\end{proof}

\begin{corollary}[Finite-sample convergence of the empirical public noise risk]
Suppose that the public samples are independently and identically distributed and that, under the stated boundedness or sub-Gaussian Jacobian assumptions,
\begin{equation*}
\left\| \boldsymbol{\Omega}_{\mathrm{pub}} - \widehat{\boldsymbol{\Omega}}_{\mathrm{pub}} \right\|_{\mathrm{op}} = \mathcal{O}_{P} \left( \sqrt{\frac{\log m}{N}} \right).
\end{equation*}

Then
\begin{align*}
\left| R_{\mathrm{priv,FO,Noise}} - \widehat{R}_{\mathrm{pub,FO,Noise}} \right| &\le \sigma^{2} \operatorname{Tr} \left( \boldsymbol{\Lambda}_{\mathrm{pub}} \right) \left\| \boldsymbol{\Omega}_{\mathrm{priv}} - \boldsymbol{\Omega}_{\mathrm{pub}} \right\|_{\mathrm{op}} \\
&\quad+ \mathcal{O}_{P} \left( \sigma^{2} \operatorname{Tr} \left( \boldsymbol{\Lambda}_{\mathrm{pub}} \right) \sqrt{\frac{\log m}{N}} \right).
\end{align*}

Hence, the finite-sample estimation component decreases at an $N^{-1/2}$-type rate, up to logarithmic dependence on $m$, whereas the population-level public–private mismatch need not vanish as $N$ increases. In the matched-population case ($\boldsymbol{\Omega}_{\mathrm{priv}}=\boldsymbol{\Omega}_{\mathrm{pub}}$), the bound reduces to
\begin{equation*}
\left| R_{\mathrm{priv,FO,Noise}} - \widehat{R}_{\mathrm{pub,FO,Noise}} \right| = \mathcal{O}_{P} \left( \sigma^{2} \operatorname{Tr} \left( \boldsymbol{\Lambda}_{\mathrm{pub}} \right) \sqrt{\frac{\log m}{N}} \right).
\end{equation*}
\end{corollary}

\begin{proof}
The result follows by combining Proposition~\ref{prop:public_private_discrepancy} with the assumed operator-norm concentration rate for $\widehat{\boldsymbol{\Omega}}_{\mathrm{pub}}$.
\end{proof}

\paragraph{Prediction stability for nonlinear $\argmax$ tasks.}
We further relate the first-order perturbation analysis to prediction stability for nonlinear classification tasks. 
For piecewise-affine ReLU networks, the first-order score perturbation is exact whenever the perturbed representation remains in the same activation region.
This yields a bound on the prediction-disagreement probability in terms of activation-region crossing and pairwise margin flipping.

For notational convenience, several symbols in this paragraph are reused from other parts of the paper. 
Such reuse is local to this paragraph, and the symbols below should be interpreted according to the definitions given here.

Consider a differentiable nonlinear score function $\mathcal{T} : \mathbb{R}^{m} \rightarrow \mathbb{R}^{K}$, whose final prediction is obtained by
\begin{equation}
\mathcal{C}(\boldsymbol{z}) = \argmax_{k\in\{1,\ldots,K\}} \mathcal{T}_k(\boldsymbol{z}).
\nonumber
\end{equation}
For each sample $\boldsymbol{z}_i$, let $c_i = \mathcal{C}(\boldsymbol{z}_i)$ denote its clean prediction. 
Let $\boldsymbol{J}_i\in\mathbb{R}^{K\times m}$ denote the Jacobian of $\mathcal{T}$ at $\boldsymbol{z}_i$. 
Using the end-to-end perturbation of Definition~\ref{definition:perturbation},
\begin{equation}
\boldsymbol{h}_i = \boldsymbol{d}_i + \boldsymbol{L}\boldsymbol{\xi}, \qquad \boldsymbol{d}_i = \bigl(\theta_1(\boldsymbol{z}_i)-1\bigr) (\boldsymbol{z}_i-\boldsymbol{\mu}), \qquad \boldsymbol{L} = \boldsymbol{U}\boldsymbol{\Lambda}^{1/2},
\nonumber
\end{equation}
the first-order approximation of the pairwise score gap between the clean predicted class $c_i$ and a competing class $k$ is defined by
\begin{align}
\mathcal{T}_{c_i}(\boldsymbol{z}_i+\boldsymbol{h}_i) - \mathcal{T}_{k}(\boldsymbol{z}_i+\boldsymbol{h}_i) 
&\approx \left( \mathcal{T}_{c_i}(\boldsymbol{z}_i) + \nabla\mathcal{T}_{c_i}(\boldsymbol{z}_i)^\top \boldsymbol{h}_i \right) - \left( \mathcal{T}_{k}(\boldsymbol{z}_i) + \nabla\mathcal{T}_{k}(\boldsymbol{z}_i)^\top \boldsymbol{h}_i \right) \nonumber \\
&= \left( \mathcal{T}_{c_i}(\boldsymbol{z}_i) - \mathcal{T}_{k}(\boldsymbol{z}_i) \right) + \left(\nabla\mathcal{T}_{c_i}(\boldsymbol{z}_i)^\top - \nabla\mathcal{T}_{k}(\boldsymbol{z}_i)^\top\right) \boldsymbol{h}_i. \nonumber
\end{align}

\begin{definition}[Clean pairwise margin]
For every competing class $k \neq c_i$, we define the clean pairwise margin as
\begin{equation}
\Delta_{ik} = \mathcal{T}_{c_i}(\boldsymbol{z}_i) - \mathcal{T}_{k}(\boldsymbol{z}_i).
\nonumber
\end{equation}
\end{definition}

\begin{definition}[Pairwise margin gradient]
For every competing class $k \neq c_i$, the pairwise margin gradient $\boldsymbol{a}_{ik} \in \mathbb{R}^{m}$ is defined as
\begin{equation}
\boldsymbol{a}_{ik} = \nabla\mathcal{T}_{c_i}(\boldsymbol{z}_i) - \nabla\mathcal{T}_{k}(\boldsymbol{z}_i) = \boldsymbol{J}_i^{\top} (\boldsymbol{e}_{c_i}-\boldsymbol{e}_k),
\nonumber
\end{equation}
where $\boldsymbol{e}_k$ denotes the $k$-th canonical basis vector of $\mathbb{R}^{K}$. Equivalently, $\boldsymbol{a}_{ik}^{\top}$ is obtained by subtracting the $k$-th row of $\boldsymbol{J}_i$ from its $c_i$-th row.
\end{definition}

\begin{definition}[Deterministic first-order margin]
Substituting the perturbation $\boldsymbol{h}_i = \boldsymbol{d}_i + \boldsymbol{L}\boldsymbol{\xi}$ into the first-order approximation of the pairwise score gap yields
\begin{align*}
\mathcal{T}_{c_i}(\boldsymbol{z}_i+\boldsymbol{h}_i) - \mathcal{T}_{k}(\boldsymbol{z}_i+\boldsymbol{h}_i) 
&\approx \left( \mathcal{T}_{c_i}(\boldsymbol{z}_i) - \mathcal{T}_{k}(\boldsymbol{z}_i) \right) + \left(\nabla\mathcal{T}_{c_i}(\boldsymbol{z}_i)^\top - \nabla\mathcal{T}_{k}(\boldsymbol{z}_i)^\top\right) \boldsymbol{h}_i \\
&= \Delta_{ik} + \boldsymbol{a}_{ik}^{\top}\boldsymbol{d}_i + \boldsymbol{a}_{ik}^{\top}\boldsymbol{L}\boldsymbol{\xi}.
\end{align*}
Isolating the non-random components, we define the deterministic first-order margin after clipping as
\begin{equation}
\bar{\Delta}_{ik} = \Delta_{ik} + \boldsymbol{a}_{ik}^{\top}\boldsymbol{d}_i.
\nonumber
\end{equation}
\end{definition}

\begin{definition}[Noise-induced pairwise perturbation]
Define the noise-induced pairwise perturbation by
\begin{equation}
\Xi_{ik} = \boldsymbol{a}_{ik}^{\top}\boldsymbol{L}\boldsymbol{\xi}.
\nonumber
\end{equation}
Define its variance by
\begin{equation}
\operatorname{Var}(\Xi_{ik}) = \sigma^2 \sum_{j=1}^{m} \lambda_j \left( \boldsymbol{a}_{ik}^{\top}\boldsymbol{u}_j \right)^2.
\nonumber
\end{equation}
\end{definition}

Define the largest coordinate-wise projected noise scale by
\begin{equation}
M_{ik} = \max_{1\le j\le m} \sqrt{\lambda_j} \left| \boldsymbol{a}_{ik}^{\top}\boldsymbol{u}_j \right|,
\nonumber
\end{equation}

\begin{proposition}[Prediction disagreement bound for nonlinear ReLU classifiers]
\label{prop:nonlinear_argmax_disagreement}

Suppose that $\mathcal{T}:\mathbb{R}^m\to\mathbb{R}^K$ is a piecewise-affine ReLU network and is differentiable at $\boldsymbol{z}_i$.
Let $\mathcal{R}_i$ be the activation cell containing $\boldsymbol{z}_i$ on which $\mathcal{T}$ is affine with Jacobian $\boldsymbol{J}_i$, and let
$\mathcal{A}_i = \left\{ \boldsymbol{z}_i+t\boldsymbol{h}_i\in\mathcal{R}_i \text{ for all } t\in[0,1] \right\}$ denote the event that the perturbation remains in the same activation cell.

For each competing class $k\neq c_i$, suppose that the deterministic first-order margin after clipping satisfies
\begin{equation}\nonumber
    \bar{\Delta}_{ik} = \Delta_{ik} + \boldsymbol{a}_{ik}^{\top}\boldsymbol{d}_i > 0.
\end{equation}

For each pair satisfying $\operatorname{Var}(\Xi_{ik})>0$, define
\begin{equation}\nonumber
    \phi_{ik} = \frac{\bar{\Delta}_{ik}^{2}} {4\operatorname{Var}(\Xi_{ik})},
    \qquad
    \psi_{ik} = \frac{\bar{\Delta}_{ik}} {2\sigma M_{ik}}.
\nonumber
\end{equation}

Then the prediction-disagreement probability of the nonlinear
classifier satisfies
\begin{equation}
    \mathrm{Pr}_{\boldsymbol{\xi}} \left[ \mathcal{C}(\hat{\boldsymbol{z}}_i) \neq \mathcal{C}(\boldsymbol{z}_i) \right]
    \le \mathrm{Pr}_{\boldsymbol{\xi}}(\mathcal{A}_i^c) + \sum_{\substack{k\neq c_i\\ \operatorname{Var}(\Xi_{ik})>0}} \exp\left( -\min\{\phi_{ik},\psi_{ik}\} \right).
    \nonumber
\end{equation}

If $\operatorname{Var}(\Xi_{ik})=0$, then all projected
coefficients vanish and hence $\Xi_{ik}\equiv0$.
Since $\bar{\Delta}_{ik}>0$, the corresponding pair contributes
zero probability.

\end{proposition}

\begin{proof}
On the event $\mathcal{A}_i$, the perturbation remains inside the activation cell $\mathcal{R}_i$. Since $\mathcal{T}$ is affine on $\mathcal{R}_i$, the first-order expression is exact:
\begin{equation}\nonumber
    \mathcal{T}(\boldsymbol{z}_i+\boldsymbol{h}_i) - \mathcal{T}(\boldsymbol{z}_i) = \boldsymbol{J}_i\boldsymbol{h}_i.
\end{equation}

Therefore, for every competing class $k\neq c_i$, on $\mathcal{A}_i$ the actual pairwise score gap satisfies
\begin{align}
    \mathcal{T}_{c_i}(\boldsymbol{z}_i+\boldsymbol{h}_i) - \mathcal{T}_{k}(\boldsymbol{z}_i+\boldsymbol{h}_i) = \Delta_{ik} + \boldsymbol{a}_{ik}^{\top}\boldsymbol{h}_i = \Delta_{ik} + \boldsymbol{a}_{ik}^{\top}\boldsymbol{d}_i + \boldsymbol{a}_{ik}^{\top}\boldsymbol{L}\boldsymbol{\xi} = \bar{\Delta}_{ik} + \Xi_{ik}.
    \nonumber
\end{align}

If the nonlinear prediction changes while $\mathcal{A}_i$ occurs, then the score of at least one competing class must reach or exceed the score of the clean predicted class. Thus,
\begin{align}
    \left\{ \mathcal{C}(\hat{\boldsymbol{z}}_i) \neq \mathcal{C}(\boldsymbol{z}_i) \right\} 
    &\subseteq \mathcal{A}_i^c \cup \bigcup_{k\neq c_i} \left\{ \Xi_{ik} \le -\bar{\Delta}_{ik} \right\}. \nonumber
\end{align}

Applying the union bound gives
\begin{equation}
    \mathrm{Pr}_{\boldsymbol{\xi}} \left[ \mathcal{C}(\hat{\boldsymbol{z}}_i) \neq \mathcal{C}(\boldsymbol{z}_i) \right]
    \le \mathrm{Pr}_{\boldsymbol{\xi}} (\mathcal{A}_i^c) + \sum_{k\neq c_i} \mathrm{Pr}_{\boldsymbol{\xi}} \left[ \Xi_{ik} \le -\bar{\Delta}_{ik} \right].
    \nonumber
\end{equation}

Define $\kappa_{ikj} = \sqrt{\lambda_j} \left( \boldsymbol{a}_{ik}^{\top}\boldsymbol{u}_j \right)$, so that $\Xi_{ik} = \sum_{j=1}^{m}\kappa_{ikj}\xi_j$ and $M_{ik} = \max_j|\kappa_{ikj}|.$

For $\xi_j\sim\operatorname{Lap}(0,\sigma/\sqrt{2})$,
\begin{equation}\nonumber
    \mathbb{E}[e^{t\xi_j}] = \left( 1-\frac{\sigma^2t^2}{2} \right)^{-1}, \qquad |t|<\frac{\sqrt{2}}{\sigma}.
\end{equation}
By independence,
\begin{equation}\nonumber
    \log \mathbb{E} \left[ e^{t\Xi_{ik}} \right] = \sum_{j=1}^{m} -\log \left( 1 - \frac{\sigma^2\kappa_{ikj}^{2}t^2}{2} \right).
\end{equation}

For $|t| \le t_{\max} = \frac{1}{\sigma M_{ik}}$, we have $\sigma^2\kappa_{ikj}^2t^2/2\le 1/2$ for every $j$. Using $-\log(1-x)\le 2x$ for $x\in[0,1/2]$ yields
\begin{equation}
    \log \mathbb{E} \left[ e^{t\Xi_{ik}} \right] \le \sigma^2t^2 \sum_{j=1}^{m}\kappa_{ikj}^2 = t^2\operatorname{Var}(\Xi_{ik}). \nonumber
\end{equation}

Thus, for any $\tau>0$ and $0<t\le t_{\max}$, the Chernoff bound gives
\begin{equation}\nonumber
    \Pr[\Xi_{ik}\ge\tau] \le \exp \left( -t\tau + t^2\operatorname{Var}(\Xi_{ik}) \right).
\end{equation}

If $\tau \le \frac{ 2\operatorname{Var}(\Xi_{ik}) }{ \sigma M_{ik} }$, the unconstrained minimizer $t^\star = \frac{\tau}{2\operatorname{Var}(\Xi_{ik})}$ is admissible, which gives
\begin{equation}\nonumber
    \Pr[\Xi_{ik}\ge\tau] \le \exp \left( -\frac{\tau^2}{4\operatorname{Var}(\Xi_{ik})} \right).
\end{equation}

Otherwise, taking $t=t_{\max}$ gives
\begin{align}
    -t_{\max}\tau + t_{\max}^2\operatorname{Var}(\Xi_{ik})
    = -\frac{\tau}{\sigma M_{ik}} + \frac{ \operatorname{Var}(\Xi_{ik}) }{ \sigma^2M_{ik}^2 } < -\frac{\tau}{2\sigma M_{ik}}, \nonumber
\end{align}
and 
\begin{equation}\nonumber
    \Pr[\Xi_{ik}\ge\tau] \le \exp \left( -\frac{\tau}{2\sigma M_{ik}} \right).
\end{equation}

Combining the two regimes,
\begin{equation}\nonumber
    \Pr[\Xi_{ik}\ge\tau] \le \exp \left[ -\min \left\{ \frac{\tau^2}{4\operatorname{Var}(\Xi_{ik})}, \frac{\tau}{2\sigma M_{ik}} \right\} \right].
\end{equation}

Since $\Xi_{ik}$ is symmetric about zero, the same bound holds for the left tail:
\begin{equation}\nonumber
    \Pr[\Xi_{ik}\le-\tau] \le \exp \left[ -\min \left\{ \frac{\tau^2}{4\operatorname{Var}(\Xi_{ik})}, \frac{\tau}{2\sigma M_{ik}} \right\} \right].
\end{equation}

Substituting $\tau=\bar{\Delta}_{ik}>0$ gives
\begin{equation}\nonumber
    \Pr \left[ \Xi_{ik} \le -\bar{\Delta}_{ik} \right] \le \exp \left( -\min\{\phi_{ik},\psi_{ik}\} \right).
\end{equation}

For $\operatorname{Var}(\Xi_{ik})=0$, all projected coefficients vanish, and hence $\Xi_{ik}\equiv 0$. Since $\bar{\Delta}_{ik}>0$, it follows that
\begin{equation}
    \Pr \left[ \Xi_{ik} \le -\bar{\Delta}_{ik} \right] = 0. \nonumber
\end{equation}

Combining the positive-variance tail bounds with the zero-variance case and applying the union bound above proves the result.
\end{proof}

\begin{remark}
The proposition separates nonlinear prediction disagreement into two
sources: activation-region crossing and within-region pairwise margin
flipping. The latter is controlled by the projected noise variance and
the largest coordinate-wise projected noise scale. We use this result
to characterize prediction stability of the proposed transformation;
we do not claim that the resulting disagreement bound is jointly
optimized with respect to clipping and noise allocation.
\end{remark}

\subsection{Sensitivity Amplification in Transformed Space} \label{appendix:3}
Let $\boldsymbol{\phi} = \boldsymbol{z} - \boldsymbol{z}'$ be a perturbation in the original space $\mathcal{Z}$, and $\boldsymbol{\Sigma} = \boldsymbol{L}\boldsymbol{L}^\top$ be the covariance matrix. 
The worst-case sensitivity of the transformed space under the inverse Cholesky factor $\boldsymbol{L}^{-1}$, defined by the induced matrix norm $\|\boldsymbol{L}^{-1}\|_p = \max_{\boldsymbol{\phi} \neq \mathbf{0}} \frac{\|\boldsymbol{L}^{-1}\boldsymbol{\phi}\|_p}{\|\boldsymbol{\phi}\|_p}$, is amplified ($\|\boldsymbol{L}^{-1}\|_p > 1$) if the transformation reduces variance along any principal coordinate, for both $p \in \{1, 2\}$.

\textbf{Proof.} The maximum perturbation bound is fundamentally determined by the induced $\ell_p$-norm of the transformation matrix: $\|\boldsymbol{L}^{-1}\boldsymbol{\phi}\|_p \leq \|\boldsymbol{L}^{-1}\|_p \|\boldsymbol{\phi}\|_p$. 
We prove the sensitivity amplification for each norm independently.

\textbf{Case 1: $\ell_2$-norm ($p=2$).} 
The induced $\ell_2$-norm (spectral norm) of a matrix is its maximum singular value. 
For $\boldsymbol{L}^{-1}$, this is related to the eigenvalues of the precision matrix $\boldsymbol{\Sigma}^{-1}$:$$\|\boldsymbol{L}^{-1}\|_2 = \sqrt{\lambda_{\max}((\boldsymbol{L}^{-1})^\top \boldsymbol{L}^{-1})} = \sqrt{\lambda_{\max}(\boldsymbol{\Sigma}^{-1})} = \frac{1}{\sqrt{\lambda_{\min}(\boldsymbol{\Sigma})}}$$
If $\boldsymbol{L}$ applies a pure rotation without rescaling ($\boldsymbol{\Sigma} = \boldsymbol{I}$), then $\lambda_{\min}(\boldsymbol{\Sigma}) = 1$, yielding $\|\boldsymbol{L}^{-1}\|_2 = 1$ (distance is perfectly preserved). 
However, if the transformation shrinks the space such that the minimum variance is less than 1 ($\lambda_{\min}(\boldsymbol{\Sigma}) < 1$), the bound becomes strictly greater than 1:$$\|\boldsymbol{L}^{-1}\|_2 > 1 \implies \max_{\boldsymbol{\phi} \neq \mathbf{0}} \frac{\|\boldsymbol{L}^{-1}\boldsymbol{\phi}\|_2}{\|\boldsymbol{\phi}\|_2} > 1$$
Thus, the worst-case $\ell_2$ sensitivity increases under variance reduction.

\textbf{Case 2: $\ell_1$-norm ($p=1$).}
The induced $\ell_1$-norm of a matrix is its maximum absolute column sum: $\|\boldsymbol{L}^{-1}\|_1 = \max_{j} \sum_{i} |(\boldsymbol{L}^{-1})_{ij}|$.
Unlike the $\ell_2$-norm, the $\ell_1$-norm is coordinate-dependent. 
For a pure rotation matrix $\boldsymbol{U}$ (where no coordinate scaling occurs), the rotation misaligns the canonical axes, generally resulting in $\|\boldsymbol{U}^{-1}\|_1 \geq 1$.
Furthermore, if $\boldsymbol{L}$ involves shrinking (rescaling), we utilize the submultiplicative property of induced norms on the identity matrix $\boldsymbol{I} = \boldsymbol{L}\boldsymbol{L}^{-1}$:$$1 = \|\boldsymbol{I}\|_1 \leq \|\boldsymbol{L}\|_1 \|\boldsymbol{L}^{-1}\|_1 \implies \|\boldsymbol{L}^{-1}\|_1 \geq \frac{1}{\|\boldsymbol{L}\|_1}$$
If the transformed space is compressed such that the absolute column sums of $\boldsymbol{L}$ are bounded below 1 ($\|\boldsymbol{L}\|_1 < 1$), it guarantees that the inverse transformation expands the space:
$$\|\boldsymbol{L}^{-1}\|_1 \geq \frac{1}{\|\boldsymbol{L}\|_1} > 1$$
Combining the effects of rotation and scaling, if the transformation satisfies $\|L\|_1<1$, then $\|L^{-1}\|_1>1$, thereby amplifying the worst-case $\ell_1$ sensitivity.

\subsection{Privacy Guarantee}
\label{appendix:privacy_guarantee}

\emph{Neighboring relation}: Under the standard $(\epsilon,\delta)$-LDP definition, the privacy inequality must hold for \emph{every} pair of possible local inputs $\boldsymbol{z}, \boldsymbol{z}' \in \mathcal{Z}$. Expressed as a neighboring relation, this is the complete relation $\mathcal{Z} \times \mathcal{Z}$: an arbitrary replacement of the entire representation.

\privacytheorem*

\begin{proof}
Fix arbitrary $\boldsymbol{z},\boldsymbol{z}'\in\mathcal{Z}$ and a measurable set $S\subseteq\mathbb{R}^m$. The public configuration $\boldsymbol{\phi}$ is fixed under the neighboring comparison; by (A1) it does not depend on the input, so the same deterministic map $f_{\boldsymbol{\phi}}$ is applied to both $\boldsymbol{z}$ and $\boldsymbol{z}'$. By (A2),
$$
f_{\boldsymbol{\phi}}(\boldsymbol{z}), f_{\boldsymbol{\phi}}(\boldsymbol{z}') \in \bar{\mathcal{Z}}_{p,\rho}.
$$
Therefore, by (A3), the uniform $(\epsilon,\delta)$-LDP guarantee of $\mathcal{M}$ applies to this pair.

Since $g_{\boldsymbol{\phi}}$ is deterministic and measurable, its pre-image
$$
T = g_{\boldsymbol{\phi}}^{-1}(S) = \left\{ \boldsymbol{v}\in\mathbb{R}^m: g_{\boldsymbol{\phi}}(\boldsymbol{v})\in S \right\}
$$
is measurable. Hence,
$$
\begin{aligned}
    \Pr[ \Phi_{\boldsymbol{\phi}}(\boldsymbol{z})\in S ]
    & = \Pr[ \mathcal{M}(f_{\boldsymbol{\phi}}(\boldsymbol{z}))\in T ] 
    \leq e^\epsilon \Pr[ \mathcal{M}(f_{\boldsymbol{\phi}}(\boldsymbol{z}'))\in T ] + \delta 
    = e^\epsilon \Pr[ \Phi_{\boldsymbol{\phi}}(\boldsymbol{z}')\in S ] + \delta.
\end{aligned}
$$
Since $\boldsymbol{z}$, $\boldsymbol{z}'$, and $S$ were arbitrary, $\Phi_{\boldsymbol{\phi}}$ satisfies $(\epsilon,\delta)$-LDP on $\mathcal{Z}$. The pure $\epsilon$-LDP result follows as the special case $\delta=0$.
\end{proof}

We provide the instantiations of randomizers, including Laplace, AGM, and PrivUnit, as Corollaries below.

\begin{corollary}[Laplace mechanism~\citep{dwork2014algorithmic}]
Let $p=1$, so that the intermediate domain is
$$ \bar{\mathcal{Z}}_{1,\rho} = \left\{ \bar{\boldsymbol{z}}\in\mathbb{R}^m: \|\bar{\boldsymbol{z}}\|_1\le\rho \right\}, $$
and let the output space be $\mathcal{Y}=\mathbb{R}^m$. Because the LDP neighboring relation is the complete relation, the relevant sensitivity is the diameter of the entire intermediate domain. For every $\bar{\boldsymbol{z}},\bar{\boldsymbol{z}}^{\prime}\in \bar{\mathcal{Z}}_{1,\rho}$,
$$ \|\bar{\boldsymbol{z}}-\bar{\boldsymbol{z}}^{\prime}\|_1 \le \|\bar{\boldsymbol{z}}\|_1 + \|\bar{\boldsymbol{z}}^{\prime}\|_1 \le 2\rho. $$
This bound is attained by $\bar{\boldsymbol{z}}=\rho\boldsymbol{e}_1$ and $\bar{\boldsymbol{z}}^{\prime}=-\rho\boldsymbol{e}_1$. Hence,
$$ \Delta_1 = \sup_{\bar{\boldsymbol{z}},\bar{\boldsymbol{z}}^{\prime}\in \bar{\mathcal{Z}}_{1,\rho}} \|\bar{\boldsymbol{z}}-\bar{\boldsymbol{z}}^{\prime}\|_1 = 2\rho. $$
Define the coordinate-wise Laplace randomizer
$$ \mathcal{M}_{\mathrm{Lap}}(\bar{\boldsymbol{z}}) = \bar{\boldsymbol{z}} + \boldsymbol{\xi}, \qquad \xi_j \overset{\mathrm{i.i.d.}}{\sim} \operatorname{Lap}(0,b), \qquad b=\frac{2\rho}{\epsilon}. $$
For every output point $\boldsymbol{y}\in\mathbb{R}^m$ and every $\bar{\boldsymbol{z}},\bar{\boldsymbol{z}}^{\prime}\in \bar{\mathcal{Z}}_{1,\rho}$, the ratio of the corresponding densities satisfies
$$ \begin{aligned} \frac{ q(\boldsymbol{y}\mid\bar{\boldsymbol{z}}) }{ q(\boldsymbol{y}\mid\bar{\boldsymbol{z}}^{\prime}) } &= \exp\!\left( \frac{ \|\boldsymbol{y}-\bar{\boldsymbol{z}}^{\prime}\|_1 - \|\boldsymbol{y}-\bar{\boldsymbol{z}}\|_1 }{b} \right) \le \exp\!\left( \frac{ \|\bar{\boldsymbol{z}}-\bar{\boldsymbol{z}}^{\prime}\|_1 }{b} \right) \le \exp\!\left( \frac{2\rho}{b} \right) = e^\epsilon. \end{aligned} $$
Integrating this pointwise inequality over any measurable set $T\subseteq\mathbb{R}^m$ gives
$$ \Pr[ \mathcal{M}_{\mathrm{Lap}}(\bar{\boldsymbol{z}})\in T ] \le e^\epsilon \Pr[ \mathcal{M}_{\mathrm{Lap}}(\bar{\boldsymbol{z}}^{\prime})\in T ]. $$
Therefore, $\mathcal{M}_{\mathrm{Lap}}$ satisfies Assumption (A3) with $\delta=0$ uniformly over $\bar{\mathcal{Z}}_{1,\rho}$. The input-domain map is
$$ f_{\boldsymbol{\phi}}(\boldsymbol{z}) = \theta_1 \left( \boldsymbol{L}^{-1} (\boldsymbol{z}-\boldsymbol{\mu}); \rho \right), $$
and the public reconstruction map is
$$ g_{\boldsymbol{\phi}}(\boldsymbol{v}) = \boldsymbol{L}\boldsymbol{v} + \boldsymbol{\mu}. $$
No separate zero-vector convention is required, because the Laplace randomizer is defined on every point of $\bar{\mathcal{Z}}_{1,\rho}$, including the origin. Consequently,
$$ \Phi_{\mathrm{Lap},\boldsymbol{\phi}} = g_{\boldsymbol{\phi}} \circ \mathcal{M}_{\mathrm{Lap}} \circ f_{\boldsymbol{\phi}} $$
satisfies pure $\epsilon$-LDP on the original representation domain $\mathcal{Z}$. 
The calibration $b=2\rho/\epsilon$ depends only on the chosen norm, radius, and privacy budget; neither $\boldsymbol{U}$ nor $\boldsymbol{\Lambda}$ enters the calibration. 
Thus, the public anisotropic rotation and rescaling do not incur an additional privacy cost.
\end{corollary}


\begin{corollary}[Analytically calibrated Gaussian mechanism~\citep{balle2018improving}]
Let $p=2$, so that the intermediate domain is
$$ \bar{\mathcal{Z}}_{2,\rho} = \left\{ \bar{\boldsymbol{z}}\in\mathbb{R}^m: \|\bar{\boldsymbol{z}}\|_2\le\rho \right\}, $$
and let the output space be $\mathcal{Y}=\mathbb{R}^m$. Because the LDP neighboring relation is the complete relation, the relevant $\ell_2$-sensitivity is the diameter of the entire intermediate domain. For every $\bar{\boldsymbol{z}},\bar{\boldsymbol{z}}^{\prime}\in\bar{\mathcal{Z}}_{2,\rho}$,
$$ \|\bar{\boldsymbol{z}}-\bar{\boldsymbol{z}}^{\prime}\|_2 \le \|\bar{\boldsymbol{z}}\|_2+\|\bar{\boldsymbol{z}}^{\prime}\|_2 \le 2\rho. $$
This bound is attained by $\bar{\boldsymbol{z}}=\rho\boldsymbol{e}_1$ and $\bar{\boldsymbol{z}}^{\prime}=-\rho\boldsymbol{e}_1$. Hence,
$$ \Delta_2 = \sup_{\bar{\boldsymbol{z}},\bar{\boldsymbol{z}}^{\prime}\in \bar{\mathcal{Z}}_{2,\rho}} \|\bar{\boldsymbol{z}}-\bar{\boldsymbol{z}}^{\prime}\|_2 = 2\rho. $$
Define the isotropic Gaussian randomizer
$$ \mathcal{M}_{\mathrm{G}}(\bar{\boldsymbol{z}}) = \bar{\boldsymbol{z}}+\boldsymbol{\xi}, \qquad \boldsymbol{\xi} \sim \mathcal{N} \left( \boldsymbol{0}, \sigma^2\boldsymbol{I}_m \right). $$
For prescribed $\epsilon>0$ and $\delta\in(0,1)$, choose $\sigma>0$ according to the analytical Gaussian calibration for sensitivity $\Delta_2=2\rho$; equivalently, choose $\sigma$ such that
$$ \Psi\!\left( \frac{\Delta_2}{2\sigma} - \frac{\epsilon\sigma}{\Delta_2} \right) - e^\epsilon \Psi\!\left( -\frac{\Delta_2}{2\sigma} - \frac{\epsilon\sigma}{\Delta_2} \right) \le \delta, $$
where $\Psi$ denotes the cumulative distribution function of the standard normal distribution. Under this calibration, for every $\bar{\boldsymbol{z}},\bar{\boldsymbol{z}}^{\prime}\in\bar{\mathcal{Z}}_{2,\rho}$ and every measurable set $T\subseteq\mathbb{R}^m$,
$$ \Pr[ \mathcal{M}_{\mathrm{AGM}}(\bar{\boldsymbol{z}})\in T ] \le e^\epsilon \Pr[ \mathcal{M}_{\mathrm{AGM}}(\bar{\boldsymbol{z}}^{\prime})\in T ] + \delta. $$
Therefore, $\mathcal{M}_{\mathrm{AGM}}$ satisfies Assumption (A3) uniformly over $\bar{\mathcal{Z}}_{2,\rho}$. The input-domain map is
$$ f_{\boldsymbol{\phi}}(\boldsymbol{z}) = \theta_2 \left( \boldsymbol{L}^{-1} (\boldsymbol{z}-\boldsymbol{\mu}); \rho \right), $$
and the public reconstruction map is
$$ g_{\boldsymbol{\phi}}(\boldsymbol{v}) = \boldsymbol{L}\boldsymbol{v} + \boldsymbol{\mu}. $$
No separate zero-vector convention is required, because the Gaussian randomizer is defined on every point of $\bar{\mathcal{Z}}_{2,\rho}$, including the origin. 
Consequently,
$$ \Phi_{\mathrm{AGM},\boldsymbol{\phi}} = g_{\boldsymbol{\phi}} \circ \mathcal{M}_{\mathrm{AGM}} \circ f_{\boldsymbol{\phi}} $$
satisfies $(\epsilon,\delta)$-LDP on the original representation domain $\mathcal{Z}$. 
The analytical calibration depends on $(\epsilon,\delta,\rho)$ through the sensitivity $\Delta_2=2\rho$; neither the public basis $\boldsymbol{U}$ nor the finite scaling spectrum $\boldsymbol{\Lambda}$ enters the calibration. 
Thus, the public anisotropic rotation and rescaling do not incur an additional privacy cost. This instantiation is included to show that the proposed wrapper also composes with approximate-LDP mechanisms. 
\end{corollary}


\begin{corollary}[PrivUnit mechanism~\citep{bhowmick2018privunit,asi2022privunitg}]
\label{corollary:privunit}
Let the output space be $\mathcal{Y}=\mathbb{S}^{m-1}$. 
PrivUnit2~\citep{bhowmick2018privunit} satisfies pure $\epsilon$-LDP uniformly over $\mathbb{S}^{m-1}$. 
Define
$$ N(\bar{\boldsymbol{z}}) = \begin{cases} \bar{\boldsymbol{z}}/\|\bar{\boldsymbol{z}}\|_2, & \bar{\boldsymbol{z}}\neq\boldsymbol{0}, \\ \boldsymbol{\eta}_0, & \bar{\boldsymbol{z}}=\boldsymbol{0}, \end{cases} $$
where $\boldsymbol{\eta}_0\in\mathbb{S}^{m-1}$ is a fixed public unit vector, and let
$$ \mathcal{M}_{\mathrm{PU2}} = \operatorname{PU2}_{\epsilon,m}\circ N. $$
For any $\bar{\boldsymbol{z}},\bar{\boldsymbol{z}}^{\prime}\in \bar{\mathcal{Z}}_{p,\rho}$, their normalized images lie in $\mathbb{S}^{m-1}$; hence the sphere-level guarantee applies directly, and $\mathcal{M}_{\mathrm{PU2}}$ satisfies Assumption (A3) with $\delta=0$. Moreover, for every non-zero $\bar{\boldsymbol{z}}$,
$$ N\!\left(\operatorname{Clip}_p(\bar{\boldsymbol{z}};\rho)\right) = N(\bar{\boldsymbol{z}}), $$
because radial clipping multiplies its input by a positive scalar. 
Thus, $p$ and $\rho$ do not affect the PrivUnit2 input direction. Let $\tau_{\epsilon,m}\neq0$ be the publicly fixed cap threshold selected by the optimized implementation as a fixed point of its mean directional response; in the implementation considered here, the cap threshold and directional debiasing constant therefore coincide. 
The reconstruction
$$ g_{\mathrm{PU2},\boldsymbol{\phi}}(\boldsymbol{y}) = \boldsymbol{L} \left( \frac{\rho}{\tau_{\epsilon,m}}\boldsymbol{y} \right) + \boldsymbol{\mu} $$
uses only public constants. 
Therefore, no input-specific transformed norm is released and no additional privacy budget is consumed. Consequently,
$$ \Phi_{\mathrm{PU2},\boldsymbol{\phi}} = g_{\mathrm{PU2},\boldsymbol{\phi}} \circ \mathcal{M}_{\mathrm{PU2}} \circ f_{\boldsymbol{\phi}} $$
satisfies pure $\epsilon$-LDP on $\mathcal{Z}$. 
The same preprocessing and composition argument applies to PrivUnitG~\citep{asi2022privunitg} under its cited pure $\epsilon$-LDP guarantee for unit-sphere inputs, with its Gaussian-based $\mathbb{R}^m$-valued randomizer and mechanism-specific public reconstruction map.
\end{corollary}


\begin{corollary}[Invariance of the certified privacy parameters]
Fix the base randomizer and its mechanism-specific privacy calibration, together with any public quantities on which that calibration depends. For the additive mechanisms, these quantities include the clipping norm $p$ and radius $\rho$. Let $\boldsymbol{\phi}_1$ and $\boldsymbol{\phi}_2$ be any two publicly selected configurations satisfying (A1)-(A4). Then the corresponding complete mechanisms $\Phi_{\boldsymbol{\phi}_1}$ and $\Phi_{\boldsymbol{\phi}_2}$ both satisfy the same certified $(\epsilon,\delta)$-LDP parameters. Consequently, any public choice of the Jacobian basis, active rank, offset, or bounded scaling spectrum may affect utility, but it does not alter the stated privacy parameters, provided that the resulting configuration remains fixed and admissible under the corresponding mechanism-specific assumptions.
\end{corollary}

\begin{proof}
The theorem applies to each $\boldsymbol{\phi}_i$ separately, and its conclusion depends on $\boldsymbol{\phi}_i$ only through the admissibility conditions (A1)-(A4). 
For PrivUnit-based mechanisms, Corollary~\ref{corollary:privunit} gives a stronger specialization: the privacy calibration depends only on $(\epsilon,m)$. Thus, the public clipping norm $p$ and radius $\rho$ may also differ across admissible configurations without changing the certified privacy parameters.
\end{proof}


\begin{corollary}[Image-level LDP under client-side feature extraction]
\label{corollary:item-level}
Let $\mathcal{X}$ be the domain of private images and let
$$ \operatorname{Enc}_{\mathrm{pub}}: \mathcal{X} \rightarrow \mathcal{Z} $$
be a fixed, deterministic, and measurable feature extractor satisfying
$$ \operatorname{Enc}_{\mathrm{pub}}(\mathcal{X}) \subseteq \mathcal{Z}. $$
Assume that $\operatorname{Enc}_{\mathrm{pub}}$ is fixed before any private input is processed, executed locally on the client, and that neither the raw image nor the unprivatized representation is released before local randomization. This is the deployment protocol illustrated in Fig. 6. For any admissible public configuration $\boldsymbol{\phi}$, let
$$ \Phi_{\boldsymbol{\phi}}: \mathcal{Z} \rightarrow \mathbb{R}^{m} $$
be the representation-level mechanism from the preceding theorem, satisfying $(\epsilon,\delta)$-LDP on $\mathcal{Z}$ under the complete neighboring relation $\mathcal{Z}\times\mathcal{Z}$. Define
$$ \widetilde{\Phi}_{\boldsymbol{\phi}} = \Phi_{\boldsymbol{\phi}} \circ \operatorname{Enc}_{\mathrm{pub}}: \mathcal{X} \rightarrow \mathbb{R}^{m}. $$
Then $\widetilde{\Phi}_{\boldsymbol{\phi}}$ satisfies $(\epsilon,\delta)$-LDP on $\mathcal{X}$ under the complete neighboring relation $\mathcal{X}\times\mathcal{X}$. That is, for every $\boldsymbol{x},\boldsymbol{x}'\in\mathcal{X}$ and every measurable set $S\subseteq\mathbb{R}^{m}$,
$$ \Pr[ \widetilde{\Phi}_{\boldsymbol{\phi}}(\boldsymbol{x})\in S ] \le e^\epsilon \Pr[ \widetilde{\Phi}_{\boldsymbol{\phi}}(\boldsymbol{x}')\in S ] + \delta. $$
\end{corollary}

\begin{proof}
Fix arbitrary $\boldsymbol{x},\boldsymbol{x}'\in\mathcal{X}$ and a measurable set $S\subseteq\mathbb{R}^{m}$. Since
$$ \operatorname{Enc}_{\mathrm{pub}}(\boldsymbol{x}), \operatorname{Enc}_{\mathrm{pub}}(\boldsymbol{x}') \in \mathcal{Z}, $$
the $(\epsilon,\delta)$-LDP guarantee of $\Phi_{\boldsymbol{\phi}}$ applies directly to this pair. Therefore,
$$ \begin{aligned} \Pr[ \widetilde{\Phi}_{\boldsymbol{\phi}}(\boldsymbol{x})\in S ] &= \Pr[ \Phi_{\boldsymbol{\phi}} (\operatorname{Enc}_{\mathrm{pub}}(\boldsymbol{x})) \in S ] \\ &\le e^\epsilon \Pr[ \Phi_{\boldsymbol{\phi}} (\operatorname{Enc}_{\mathrm{pub}}(\boldsymbol{x}')) \in S ] + \delta \\ &= e^\epsilon \Pr[ \widetilde{\Phi}_{\boldsymbol{\phi}}(\boldsymbol{x}')\in S ] + \delta. \end{aligned} $$
Because $\boldsymbol{x}$, $\boldsymbol{x}'$, and $S$ were arbitrary, the claim follows.
\end{proof}

The client-side condition is essential to the image-level interpretation. If the raw image is transmitted to a server before feature extraction or local randomization, that server observes the unprotected image, and the corollary does not provide image-level LDP against it. In that deployment, the stated guarantee applies only to the representation released after randomization.

\textbf{Limitations.} The primary privacy unit of the mechanism is the representation. Image-level LDP follows only under the client-side deployment described in Fig. 6, where neither the raw image nor the unprivatized representation is released. Our method also relies on a fixed public encoder and downstream task model, which limits applicability when such public components are unavailable or poorly matched to the private distribution.

\subsection{Compatibility} \label{sec:apdx_comp}
Our approach can be seamlessly integrated with mechanisms that achieve randomization by sampling noise from Laplace, Gaussian, or other distributions and adding it to the representation. 
Furthermore, it is compatible with mechanisms based on probabilistic sampling designed to preserve directional information~\citep{bhowmick2018privunit,asi2022privunitg}. 
This is achieved simply by incorporating an additional step to transform $\bar{\boldsymbol{z}}$ into the required input format of the target mechanism, followed by an inverse transformation prior to applying the post-processing function. 

\subsection{Aggregation}\label{sec:apdx_agg}
Because of the heavy randomization in LDP, practical applications frequently adopt an aggregation step to average out the injected noise. 
In our proposed approach, this aggregation takes place between the randomization and post-processing steps. 
Since we inject zero-mean, isotropic noise during randomization, the aggregation step remains unbiased and reduces the variance of the injected noise as the number of samples increases.
Our post-processing function is applied to this aggregated output, ensuring that the resulting representations perform effectively in the downstream model.

\subsection{Privacy-Preserving Federated Learning} \label{sec:apdx_fl}
Our approach is applicable to federated learning (FL) employing LDP for privacy preservation. 

Consider an FL system where $K$ clients ($K \in \mathbb{Z}^+$) compute and send their local gradients $\boldsymbol{g}^{(k)} \in \mathbb{R}^d$ for $k \in \{1,\dots,K\}$ to a central server, which aggregates them to update the global model and distributes the updated model to the clients.
Assuming an untrusted server and network, clients apply LDP to randomize their gradients prior to transmission. 
Let $\boldsymbol{g}_{i}^{(k)} \in \mathbb{R}^d$ denote the gradient computed from the $i$-th sample ($i \in \{1,\dots,N\}, N \in \mathbb{Z}^+$) at the $k$-th client. 
The gradient $\boldsymbol{g}_{i}^{(k)}$ consists of $d$ components, expressed as $\boldsymbol{g}_{i}^{(k)} = \begin{bmatrix} g_{1i}^{(k)} & \dots & g_{di}^{(k)} \end{bmatrix}^\top$.
For simplicity, we assume that each component lies within a bounded range.
Without considering privacy preservation, the local gradient $\boldsymbol{g}^{(k)}$ is computed as: $\boldsymbol{g}^{(k)} = \frac{1}{N} \sum_{i=1}^N \boldsymbol{g}_{i}^{(k)}$.

When applying LDP for privacy preservation, two conventional methods are used for randomizing the gradient. 
The first directly randomizes each $\boldsymbol{g}_{i}^{(k)}$ prior to averaging. 
The second adds noise directly to the aggregated mean, utilizing the sensitivity of the mean function determined by the bounded $\boldsymbol{g}_{i}^{(k)}$.
In contrast to these methods, one can preserve privacy by first permuting the $N$-dimensional vectors, denoted as $\begin{bmatrix} g_{j1}^{(k)} & \dots & g_{jN}^{(k)}\end{bmatrix}^\top \in \mathbb{R}^N$ for $j \in \{1,\dots,d\}$, and then randomizing them using an LDP mechanism $\mathcal{M}$. 
In this context, the downstream operation (or model) is represented as an $N$-dimensional row vector $\boldsymbol{W}=\frac{1}{N}\mathbf{1}_N^\top \in \mathbb{R}^{1 \times N}$, which corresponds to the element-wise summation and scaling of the permuted vectors. 
The mechanism $\mathcal{M}$ is applied between our pre-processing and post-processing steps.
Notably, this method can also be extended to central differential privacy (CDP).

Furthermore, the linear formulation of federated aggregation suggests an additional application of our approach to secure aggregation with DP-SGD.

Consider $K$ clients, each holding a $d$-dimensional local gradient $\boldsymbol{g}^{(k)} \in \mathbb{R}^{d}$.
Let $\boldsymbol{W} = \begin{bmatrix} w^{(1)} & \cdots & w^{(K)} \end{bmatrix} \in \mathbb{R}^{1\times K}$ denote a known aggregation-weight vector, where $w^{(k)}$ is the weight assigned to client $k$.
Let $\boldsymbol{G} = \begin{bmatrix} \boldsymbol{g}^{(1)} & \cdots & \boldsymbol{g}^{(K)} \end{bmatrix}^\top \in \mathbb{R}^{K\times d}$ denote the row-wise stacked local gradient matrix.
The server computes the weighted aggregate as
\begin{equation}\nonumber
\boldsymbol{W}\boldsymbol{G} = \left( \sum_{k=1}^{K} w^{(k)}\boldsymbol{g}^{(k)} \right)^\top \in \mathbb{R}^{1\times d}.
\end{equation}
For example, uniform aggregation corresponds to $w^{(k)}=1/K$, whereas FedAvg-type aggregation may assign weights according to the number of local samples held by each participating client.

Suppose that the participating clients can establish shared random seeds.
Our approach allows the perturbations applied to the local gradients to be structured according to the row space and null space of the aggregation operator $\boldsymbol{W}$.
In particular, correlated perturbations can be constructed within $\operatorname{Null}(\boldsymbol{W})$ such that they cancel under the prescribed aggregation.
In contrast, perturbations lying in $\operatorname{Row}(\boldsymbol{W})$ remain visible after aggregation.
Accordingly, the null-space component can serve as an aggregation-canceling correlated mask, whereas the row-space component determines the perturbation that remains in the aggregated update.
When Gaussian perturbations are employed and appropriately calibrated, the surviving row-space component can serve as the privacy-preserving noise in the aggregated update, analogous to the Gaussian perturbation used in DP-SGD.

\subsection{Experimental Details} \label{appendix:exp}

We evaluate our method across three downstream tasks, including one time-series regression and two image classification tasks. 
The evaluation covers a range of privacy budgets ($\epsilon \in [0.5, 10]$), representation dimensionalities ($m\in\{16,32,64\}$), and downstream models, such as linear regression (LR), linear classifiers (LC) and multi-layer perceptron (MLP) classifiers. 

\subsubsection{Global Setup}

\textbf{Experimental Setup.}
All experiments were conducted on two workstations: one equipped with an Intel Xeon Platinum 8358 CPU, 64 GB of RAM, and a single NVIDIA L40S GPU, and the other with an Intel Core i7-8700 CPU, 32 GB of RAM, and a single NVIDIA GTX 1080 Ti GPU.

\textbf{Hyperparameters.}
All clipping thresholds $\rho$ are set to the 90th percentile of the corresponding norms computed over the public training set. 
Furthermore, we set $\lambda_{\mathrm{max}} = 1000$ based on the experimental observations provided in Appendix~\ref{appdx:finite_cap}.
                                                
\subsubsection{Setup 1: Regression on London Smart Meters Dataset}    
\textbf{Dataset.} We evaluate our method using a linear regression (LR) model trained on the publicly available London Smart Meters dataset \citep{jeanmidev2019smartmeters,ukpowernetworks_lcl}. 
This dataset records the half-hourly electricity consumption of 5,547 London households over approximately 27 months. 
The dataset provides several aggregated features for each household, including
\textit{household\_id}, \textit{day}, \textit{energy\_count}, \textit{energy\_mean}, \textit{energy\_median}, \textit{energy\_max}, \textit{energy\_min}, and \textit{energy\_std}. 

\textbf{Downstream Task.} The linear regression (LR) predicts the next day's \textit{energy\_mean}. 
We use a $m$-dimensional vector $\boldsymbol{z}$ as the input, which is made up of \textit{energy\_mean} values from the previous $m$ consecutive days.
The LR model is trained via Ordinary Least Squares (OLS).

\textbf{Privacy Model.} We employ \emph{user-level} LDP as our privacy constraint, considering that smart meter readings can expose private household activities. The goal of our experiments is to improve data utility under this localized privacy guarantee.

\begin{figure}[h]
  \centering
  \includegraphics[width=0.5\textwidth]{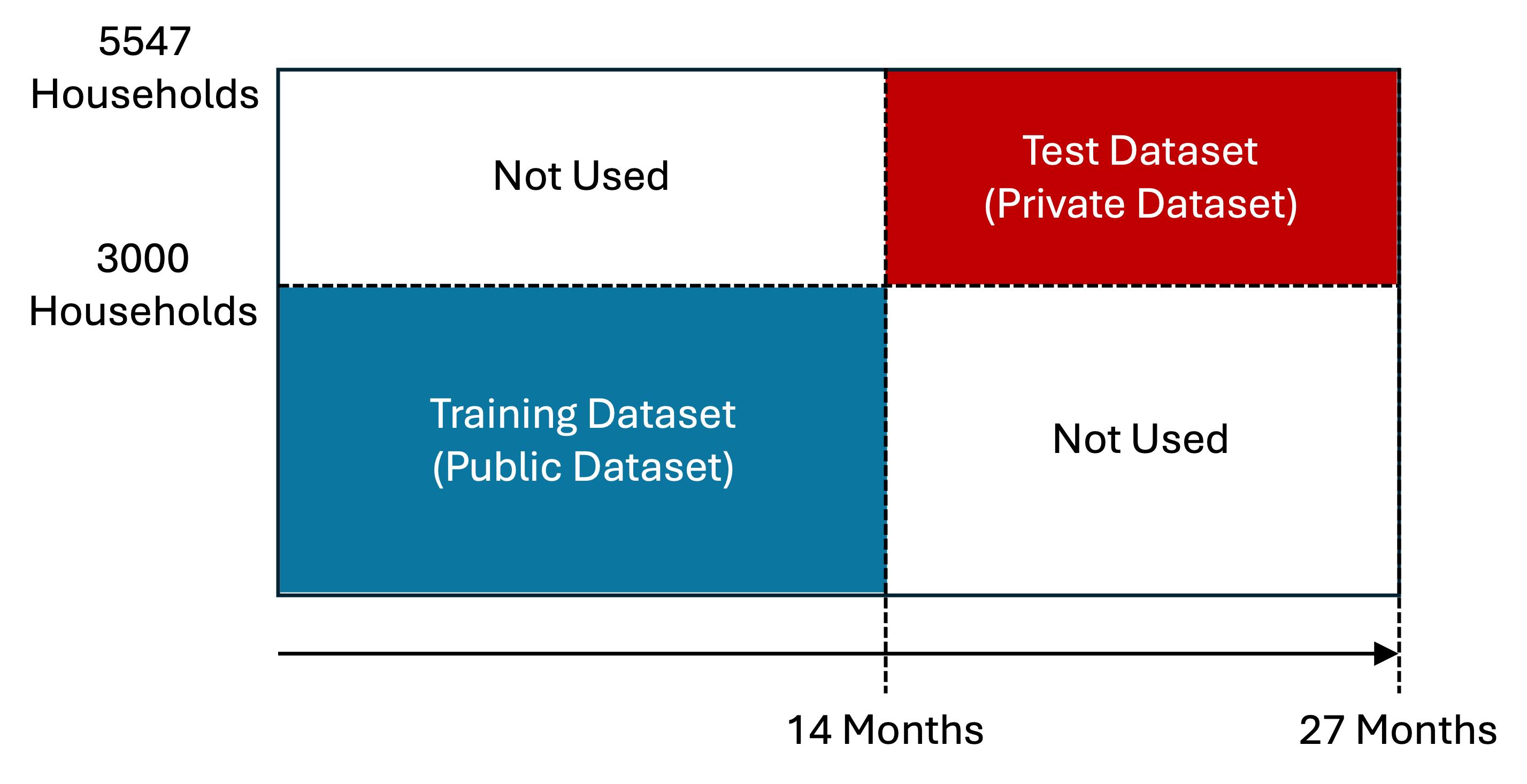} 
  \caption{London Smart Meters Dataset Splitting Scheme for Evaluation with Regression}
  \label{fig:london_smart_meters_dataset}
\end{figure}

\textbf{Public and Private Datasets.}
The dataset is partitioned both by household and chronologically into training and test sets, which serve as the public and private datasets, respectively. 
Specifically, we utilize the initial 14 months of data from 3,000 randomly selected households as the public dataset to train the LR model. 
The subsequent 13 months of data from the remaining households constitute the private dataset used to evaluate the mechanisms.

\textbf{Evaluation Protocol.} We adopt an evaluation scenario in which an untrusted server collects and aggregates randomized vectors from households in the private dataset pool before performing the regression task. 
To avoid privacy budget composition across evaluation rounds, the private household pool of 2,547 households is partitioned into disjoint, equal-sized subsets across the $N$ rounds prior to evaluation. 
Each round is assigned an exclusive subset of $\lfloor 2547 / N \rfloor$ households, ensuring that no household appears in more than one round. 
On each date, $K$ ($1 \leq K \leq \lfloor 2547 / N \rfloor$) households are then uniformly sampled at random without replacement from the round's assigned subset. 
For instance, with $N=20$, each round is allocated 127 households, yielding a maximum of $K = 127$ per round. 
Furthermore, evaluation dates are spaced at least $m$ days apart, ensuring that no household's measurement from a single day appears in more than one evaluation window. 
Together, these conditions guarantee that each household contributes to at most one aggregation across the entire evaluation, ensuring that the privacy cost incurred is $\epsilon$ per household regardless of $N$.

The evaluation is conducted over $N$ rounds. 
For instance, setting $m=16$ corresponds to performing $N=20$ regression rounds over a span of 320 days. 

We evaluate utility by measuring the Mean Squared Error (MSE) against ground-truth values. 

\textbf{Composition.} In our evaluation, the sequential composition of privacy budgets is circumvented. 
As detailed above, the partitioning of the household pool into disjoint subsets and the  $m$-day spacing between evaluation dates ensure that the data from any given household is queried at most once across all $N$ rounds. 
According to the parallel composition theorem of differential privacy, evaluating queries on disjoint subsets of data does not accumulate the privacy cost. 
Therefore, the total privacy guarantee for any individual household remains bounded by the single-query budget $\epsilon$, without requiring advanced composition mechanisms or budget degradation.

\subsubsection{Setup 2: Image Classification on MNIST and CIFAR-10 Datasets} 
\textbf{Dataset.} We evaluate our method using linear and nonlinear classifiers pre-trained on two public datasets, MNIST and CIFAR-10:
MNIST~\citep{lecun1998mnist} consists of 70,000 grayscale images of handwritten digits (0--9) with a size of $28 \times 28$ pixels. 
It is partitioned into a training set of 60,000 images and a test set of 10,000 images.
CIFAR-10~\citep{krizhevsky2009cifar} is utilized to assess performance on more complex, high-dimensional color images. 
It contains 60,000 $32 \times 32$ images across 10 classes, with 6,000 images per class. 
The dataset is divided into 50,000 training images and 10,000 test images.
Additionally, we incorporate CIFAR-10-C~\citep{hendrycks2019benchmarking} to evaluate the robustness of our mechanisms against common real-world corruptions. 
This dataset consists of the original CIFAR-10 test set subjected to 19 types of corruptions, including numerous forms of noise, blur, weather, and digital effects, across five levels of severity. 
It serves as a benchmark for assessing the model's performance under distribution shifts and corrupted inputs.

In our experiments, the training sets from CIFAR-10 and MNIST are treated as public. 
These are partitioned into training and validation sets with an 80:20 split to pre-train the public feature extractors (VAE~\citep{kingma2013auto} and ResNet-20~\citep{he2016deep}) and classifiers. 
To assess the utility of the mechanisms, the test sets of CIFAR-10-C~\citep{hendrycks2019benchmarking} (e.g., brightness, fog, defocus blur, and Gaussian noise), CIFAR-10, and MNIST are treated as private. 
Note that the CIFAR-10-C test set exhibits a distribution shift compared to the CIFAR-10 training set.

\begin{figure}[h]
  \centering
  \includegraphics[width=0.5\textwidth]{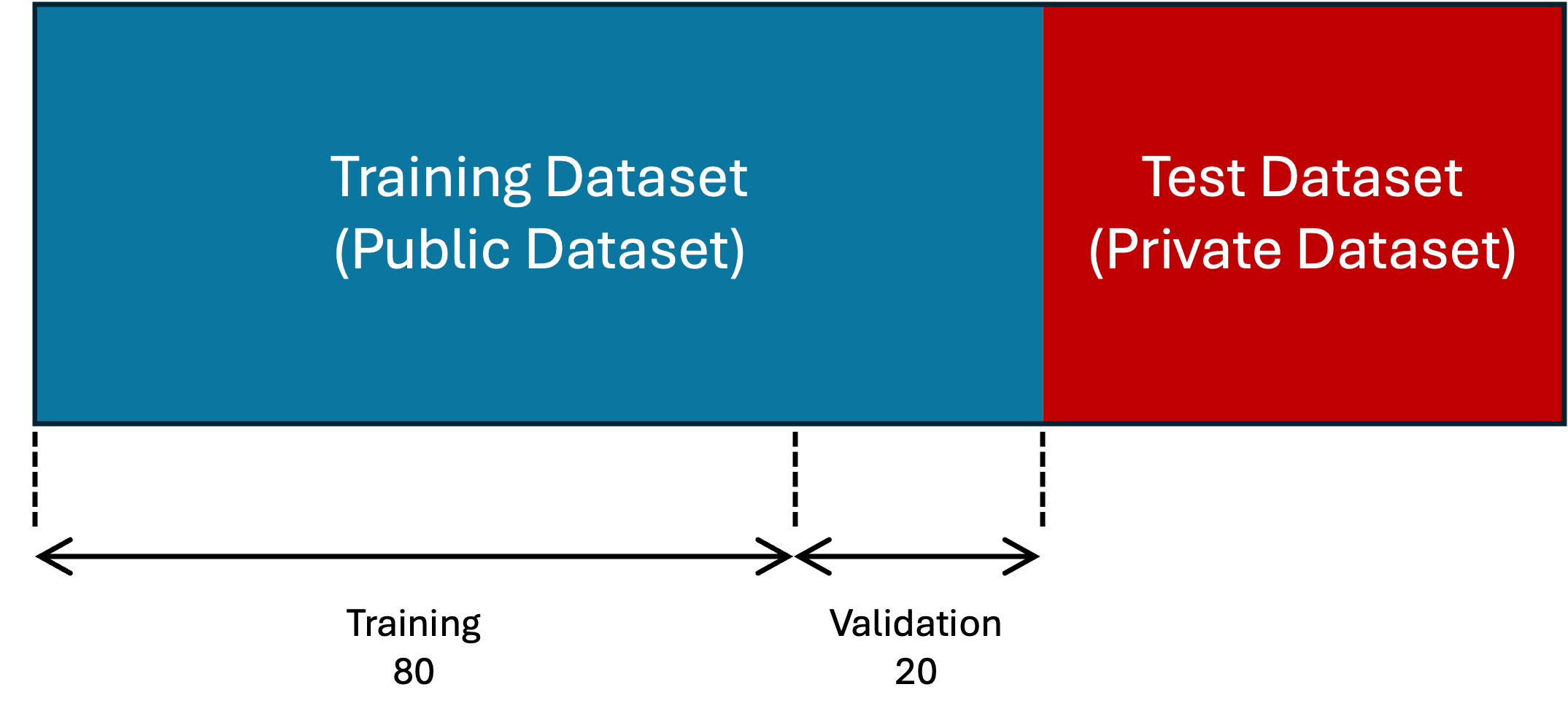} 
  \caption{Image Dataset Splitting Scheme for Evaluation with Classification using MNIST and CIFAR-10}
  \label{fig:image_dataset}
\end{figure}

\textbf{Downstream Task.} For the classification task, the objective is to predict the class label associated with each input image, such as digits for MNIST or object categories for CIFAR-10 and CIFAR-10-C.
We use an $m$-dimensional representation vector $\boldsymbol{z}$ as the input, which is extracted from the raw image data through a pre-trained feature extractor.
The classification models are trained by minimizing the standard cross-entropy loss.

\textbf{Privacy Model.} For the image classification tasks using MNIST, CIFAR-10 and CIFAR-10-C, unlike the regression task, we apply \emph{item-level} LDP, where the privacy guarantee is provided for each individual image sample.

\textbf{Public and Private Datasets.} 
Consistent with the regression task, the training datasets are treated as public. 
These are further partitioned into training and validation sets with an 80:20 split to pre-train the public feature extractor and classifiers. 
Conversely, the test datasets are treated as private, serving as the evaluation data to assess the utility of our method under LDP constraints.

\textbf{Evaluation Protocol.}
We consider a distributed setting where a central server aggregates representations $\boldsymbol{z}$ from $K$ virtual clients. 
These representations are extracted by processing private test images through a public feature extractor. 
For an $N$-round evaluation, the 10,000 samples in the private test dataset are partitioned into $N$ disjoint subsets and randomly assigned to clients, regardless of their class labels. 
For instance, with $N=10000$ and $K=1$, each client is allocated 1 unique test sample. 
To avoid privacy budget accumulation, we ensure that each representation vector is transmitted only once across the $N$ rounds.

\textbf{Utility Metrics and Composition.}
Utility is quantified by classification accuracy. 
From a privacy perspective, because each representation is transmitted only once, the budget does not accumulate over multiple rounds. 

\begin{figure}[t]
  \centering
  \includegraphics[width=1.0\textwidth]{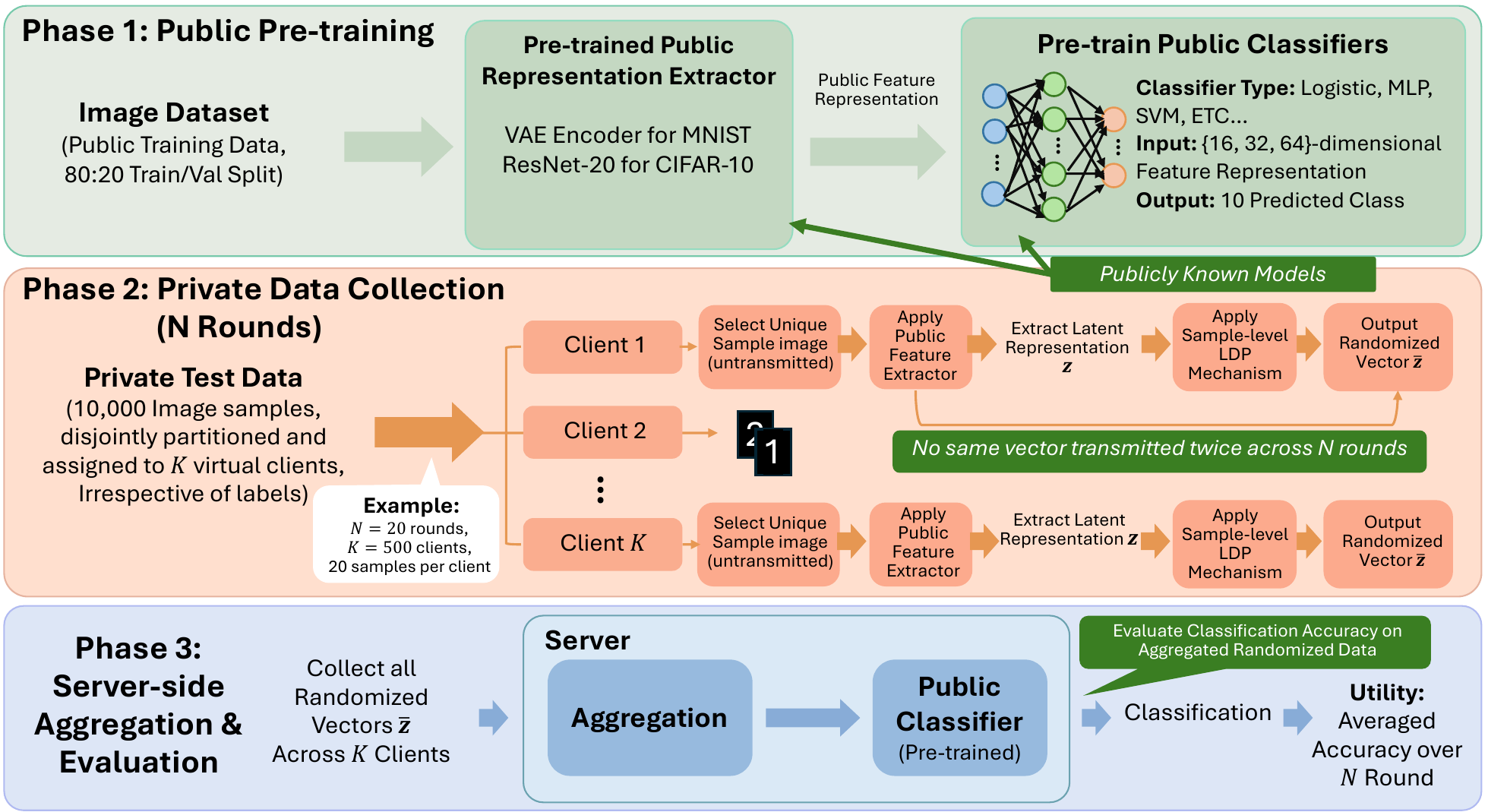} 
  \caption{Evaluation Protocol for Classification.}
\end{figure}
    
\subsubsection{Downstream Models}
\textbf{Linear Regression.} 
For the regression task, we evaluate the utility of our representations using the London smart meter dataset. 
The input feature is an $m$-dimensional representation vector $\boldsymbol{z}$, constructed from a sliding window of $m$ consecutive days of energy data. 
The models are trained via OLS.
The objective is to predict an energy metric for the following day, such as the \textit{energy\_mean}. 
The OLS algorithm fits the input data to learn a coefficient vector $\boldsymbol{W} \in \mathbb{R}^{m}$ and an intercept scalar, minimizing the mean squared error (MSE) between the predicted and actual target.

\textbf{Linear Classifier.} 
In our experiments, we use a multinomial linear regression model for CIFAR-10 classification. 
Specifically, we utilize a linear classifier head corresponding to the final fully-connected layer of the ResNet-20 architecture. 
This head directly maps the fixed 64-dimensional feature vectors, obtained from the pre-trained feature extractor, to 10 class logits. 
Under this standard pre-training, the model achieves an accuracy over 90\% on the CIFAR-10 validation set. 
Since the classifier relies on a single linear layer mapping to 10 output classes, the Jacobian matrix of its outputs with respect to the input representations has a rank of at most 10. 
This restricted rank bounds the task-relevant subspace.

\textbf{Nonlinear MLP Classifier.} 
To demonstrate that our approach generalizes to complex nonlinear settings, we additionally employ Multi-Layer Perceptron (MLP) classifiers. 
For the CIFAR-10 dataset, the MLP processes 64-dimensional ResNet features through two hidden layers with widths of 10 and 32, interleaved with nonlinear activations (ReLU, GELU, or Tanh). 
To evaluate our method across different input dimensions, we also apply a ReLU-based MLP with an identical architecture to MNIST classification, taking an $m$-dimensional latent vector ($m =32$) from a jointly fine-tuned VAE encoder as input.

\subsubsection{Feature Extraction}
For the image classification tasks, we employ two distinct models for feature extraction: a standard Variational Autoencoder (VAE) and a ResNet-20 architecture.

The VAE is applied to the MNIST dataset to facilitate experiments across various latent dimensions, enabling us to evaluate the impact of dimensionality on LDP utility. 
Specifically, we investigate latent representation vectors of size $m=32$.

Conversely, we utilize ResNet-20 for the CIFAR-10 and CIFAR-10-C datasets. 
Unlike the simple grayscale digits in MNIST, CIFAR-10 and CIFAR-10-C consist of complex color images, necessitating a feature extractor capable of processing multi-channel inputs. 
ResNet-20 was chosen for its optimal balance of efficiency and performance; it is a relatively lightweight model that minimizes computational overhead while achieving approximately 90\% accuracy on CIFAR-10. 
To construct the base representations for our evaluations, we utilize the pre-trained ResNet-20 to extract a fixed 64-dimensional feature vector, which serves as the input to the downstream linear classifier.

\textbf{VAE.} For the MNIST dataset, we pre-train a standard Variational Autoencoder (VAE) composed of multi-layer perceptrons (MLPs) using the public dataset. 
The encoder network takes the flattened $28 \times 28$ grayscale images as input and processes them through two shared hidden layers, each containing 512 units with ReLU activations. 
This is followed by two parallel linear layers that output the mean and log-variance of the $m$-dimensional latent space. The decoder mirrors this architecture, mapping the latent vector back to the original image space through two 512-unit hidden layers, culminating in a sigmoid output layer. 
The VAE is trained to minimize the standard objective function, combining binary cross-entropy for reconstruction loss and Kullback-Leibler (KL) divergence. 
We utilize the Adam optimizer with a learning rate of $10^{-3}$ and a batch size of 512. To prevent overfitting, training incorporates early stopping with a patience of 10 epochs, monitoring the validation loss on the 20\% split of the public training data.

Following this unsupervised pre-training, the encoder is jointly fine-tuned end-to-end with the downstream classifiers using cross-entropy loss. 
This joint optimization ensures that the encoder generates task-specific representations optimized for classification accuracy rather than solely for image reconstruction. 
Finally, for the downstream LDP evaluations, we deterministically extract the mean vector directly from this fine-tuned encoder to serve as the base representation vector, bypassing the reparameterization trick to ensure stable evaluation.

\textbf{ResNet-20.} For the CIFAR-10 and CIFAR-10-C datasets, we implement a standard ResNet-20 architecture as the feature extractor. 
The network processes the 3-channel color images through an initial convolutional layer (16 channels), followed by three sequential stages of residual blocks. 
Each stage consists of three basic blocks, progressively increasing the feature channels from 16 to 32 and finally to 64, with spatial downsampling applied via strided convolutions. 
Network weights are initialized using Kaiming normal initialization. 
To extract the representation vector $\boldsymbol{z}$, the final feature map is processed through an adaptive average pooling layer and flattened, yielding a 64-dimensional deterministic feature vector.

Similar to the pre-trained classifiers, the feature extractors are pre-trained exclusively on public datasets, specifically the training sets of MNIST and CIFAR-10. 
During evaluation, these extractors process inputs from private datasets (the test sets of MNIST, CIFAR-10, and CIFAR-10-C) to produce an $m$-dimensional vector $\boldsymbol{z}$ as output.

\clearpage

\subsubsection{Main Results: Comparison with the Task-Aware Approach}
\label{apdx:task-aware}
We evaluate the performance of our PA-integrated mechanisms against the Task-Aware Approach proposed by Cheng et al.~\citep{cheng2022task} for MNIST classification.
Detailed methodological descriptions are available in the original work, and its implementation can be found in the official GitHub repository: \url{https://github.com/chengjiangnan/task_aware_privacy}.
We use their best performance settings.

Mechanisms with fully integrated PA consistently outperform the Task-Aware approach across all privacy budgets. 
Notably, the Task-Aware approach performs at a level similar to or below Laplace+PA while a low-dimensional latent space of $m=3$ was selected to achieve its best performance.
Given that the scale of the injected noise grows at a rate of $\mathcal{O}(m)$ with increasing dimensionality, this setting provided an advantageous condition for the Task-Aware approach.
The performance gap is substantial when comparing the Task-Aware approach to the PrivUnit2+PA and PrivUnitG+PA mechanisms. 
For instance, at $\epsilon=7.5$, Task-Aware achieves an accuracy of only 0.2857, whereas PrivUnit2+PA and PrivUnitG+PA reach 0.7252 and 0.6578, respectively, representing a substantial difference of over 0.3.
CW+PA w/o bounding underperforms relative to the Task-Aware approach at all $\epsilon$ values. 
This is primarily because CW is optimized for the DMSE objective.

\begin{figure}[h]
  \centering
  \includegraphics[width=0.8\textwidth]{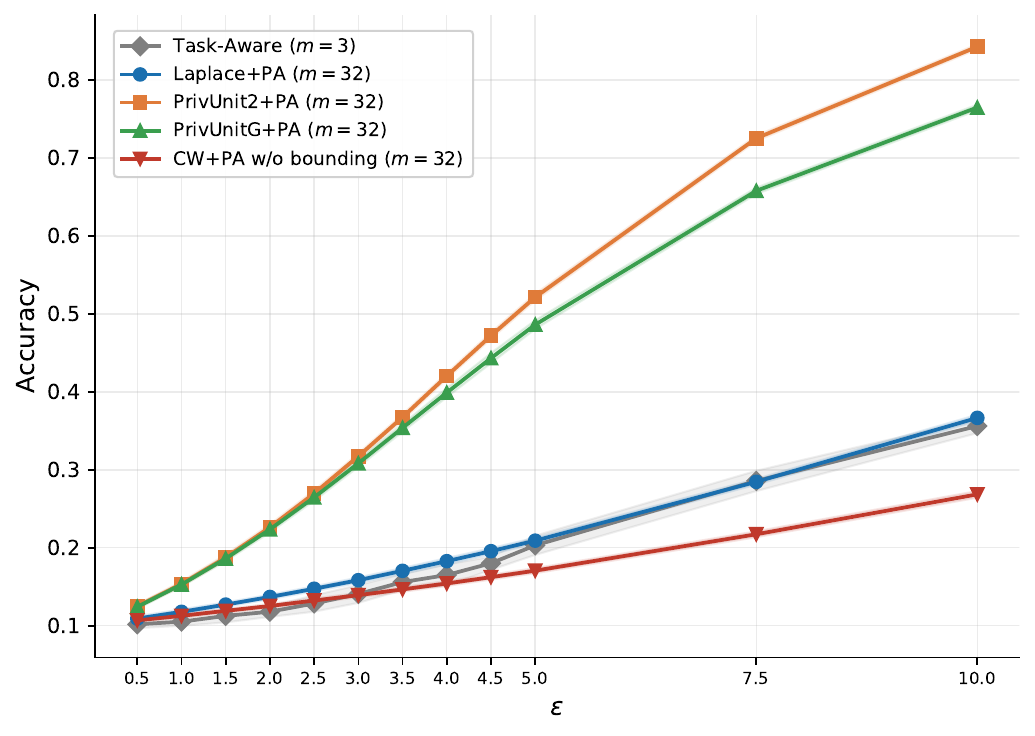} 
  \caption{Test accuracy of five mechanisms on MNIST under $\epsilon$-LDP, with shaded regions indicating $\pm 1$ standard deviation. PrivUnit2+PA and PrivUnitG+PA substantially outperform the baselines across all privacy budgets, while Laplace+PA performs comparably to Task-Aware despite using a larger projection dimension ($m=32$ vs. $m=3$).}
\end{figure}

\begin{table}[h]
\caption{Main results for \textbf{Nonlinear Classification} on \textbf{MNIST} under $\epsilon$-LDP mechanisms. We evaluate utility by measuring the classification accuracy using ground-truth labels. Higher accuracy indicates higher utility.}
\centering
\resizebox{\textwidth}{!}{
\begin{tabular}{lc cccccccccccc}
\toprule
& \multicolumn{12}{c}{Privacy Budget ($\epsilon$)} \\
 \cmidrule(l){3-14}
 \textbf{Mechanism} & $m$ & 0.5 & 1.0 & 1.5 & 2.0 & 2.5 & 3.0 & 3.5 & 4.0 & 4.5 & 5.0 & 7.5 & 10.0 \\
 \midrule
 \makecell[l]{Task-Aware}  
 & 3 & \makecell{0.1018\\(0.0049)} & \makecell{0.1052\\(0.0053)} & \makecell{0.1126\\(0.0083)} & \makecell{0.1180\\(0.0067)} & \makecell{0.1285\\(0.0105)} & \makecell{0.1407\\(0.0114)} & \makecell{0.1560\\(0.0101)} & \makecell{0.1650\\(0.0135)} & \makecell{0.1801\\(0.0084)} & \makecell{0.2031\\(0.0123)} & \makecell{0.2857\\(0.0130)} & \makecell{0.3561\\(0.0091)} \\
 \midrule
 \makecell[l]{Laplace+PA} 
 & 32 & \makecell{0.1092 \\ (0.0030)} & \makecell{0.1178 \\ (0.0032)} & \makecell{0.1271 \\ (0.0029)} & \makecell{0.1369 \\ (0.0029)} & \makecell{0.1474 \\ (0.0032)} & \makecell{0.1584 \\ (0.0035)} & \makecell{0.1704 \\ (0.0037)} & \makecell{0.1828 \\ (0.0040)} & \makecell{0.1956 \\ (0.0039)} & \makecell{0.2091 \\ (0.0037)} & \makecell{0.2846 \\ (0.0037)} & \makecell{0.3666 \\ (0.0041)} \\
 \midrule
 \makecell[l]{PrivUnit2+PA} 
 & 32 & \makecell{0.1250 \\ (0.0025)} & \makecell{0.1536 \\ (0.0032)} & \makecell{0.1875 \\ (0.0038)} & \makecell{0.2263 \\ (0.0038)} & \makecell{0.2696 \\ (0.0040)} & \makecell{0.3174 \\ (0.0036)} & \makecell{0.3674 \\ (0.0039)} & \makecell{0.4203 \\ (0.0038)} & \makecell{0.4720 \\ (0.0041)} & \makecell{0.5214 \\ (0.0043)} & \makecell{0.7252 \\ (0.0036)} & \makecell{0.8427 \\ (0.0032)} \\
 \midrule
 \makecell[l]{PrivUnitG+PA} 
 & 32 & \makecell{0.1242 \\ (0.0027)} & \makecell{0.1528 \\ (0.0025)} & \makecell{0.1862 \\ (0.0035)} & \makecell{0.2236 \\ (0.0040)} & \makecell{0.2646 \\ (0.0044)} & \makecell{0.3083 \\ (0.0047)} & \makecell{0.3538 \\ (0.0051)} & \makecell{0.3989 \\ (0.0059)} & \makecell{0.4432 \\ (0.0066)} & \makecell{0.4861 \\ (0.0057)} & \makecell{0.6578 \\ (0.0042)} & \makecell{0.7647 \\ (0.0037)} \\
 \midrule
 \makecell[l]{CW+PA w/o bounding} 
 & 32 & \makecell{0.1069 \\ (0.0027)} & \makecell{0.1127 \\ (0.0027)} & \makecell{0.1191 \\ (0.0025)} & \makecell{0.1253 \\ (0.0027)} & \makecell{0.1323 \\ (0.0028)} & \makecell{0.1392 \\ (0.0029)} & \makecell{0.1466 \\ (0.0031)} & \makecell{0.1542 \\ (0.0032)} & \makecell{0.1623 \\ (0.0032)} & \makecell{0.1706 \\ (0.0033)} & \makecell{0.2173 \\ (0.0038)} & \makecell{0.2685 \\ (0.0043)} \\
 \bottomrule
\end{tabular}
}
\end{table}

\clearpage

\clearpage

\subsubsection{Main Results: Accuracy of Nonlinear MLP Classifiers with ReLU Activations on CIFAR-10-C (Brightness, Fog, Defocus Blur, and Gaussian) Test Datasets}~\label{apdx_corruption}
A central assumption of our approach is that the Jacobian computed on public data captures the task-critical subspaces of the private data. 
To stress-test this assumption, we evaluate PA across four corruption types (Brightness, Fog, Defocus Blur, and Gaussian Noise) and five severity levels on CIFAR-10-C, yielding 20 distinct distribution shift scenarios. 
These corruptions vary not only in severity but also in nature: Brightness and Fog primarily alter global statistics, Defocus Blur degrades spatial frequency, and Gaussian Noise directly corrupts pixel-level information, offering a diverse probe of how different types of public-private mismatch affect the utility of Jacobian-based subspace identification.

A direct comparison of absolute accuracy across severities is confounded by the fact that the no-randomization baseline, which serves as a natural upper bound on achievable utility, itself degrades as corruption severity increases. For instance, under Gaussian Noise at severity 5, the no-randomization accuracy drops to 0.2621, leaving far less room for any mechanism to demonstrate improvement. 
To account for this, we measure the performance of PA in terms of the fraction of recoverable accuracy it restores.
Formally, for each corruption type and severity, we define the recoverable gap as the difference between the no-randomization accuracy and the base mechanism accuracy, and the performance gain as the improvement brought by integrating PA:
$$ \text{Recovery Ratio} = \frac{\text{Mean Acc(Baseline+PA)} - \text{Mean Acc(Baseline)}}{\text{Mean Acc(No Randomization)} - \text{Mean Acc(Baseline)}} $$
A higher recovery ratio indicates that PA recovers a larger fraction of the utility lost to LDP noise, independent of the absolute difficulty imposed by the corruption. 
This metric cleanly isolates PA's contribution from the confounding effect of corruption-induced accuracy degradation.

Tables~\ref{tb:brightness} through \ref{tb:Gaussian} present the accuracy of PrivUnit2+PA and PrivUnitG+PA across all 20 distribution shift scenarios, while Figure~\ref{fig:distribution_shift} summarizes the corresponding recovery ratios at representative privacy budgets $\epsilon \in \{1.0, 3.0, 5.0, 7.5\}$. 

In summary, the results demonstrate that \emph{PA is robust to both the type and severity of practical distribution shifts}.

These results show several consistent patterns.
First, PA consistently enhances utility across all 20 scenarios. 
The recovery ratio remains positive for every combination of corruption type and severity. 
This confirms that the Jacobian computed on clean public data remains a viable proxy for identifying task-critical subspaces even when private data undergoes substantial distributional shift. 
Such findings provide direct empirical support for our central hypothesis.

Second, the recovery ratio scales positively with the privacy budget. 
Across all corruption types and severities, the ratio grows monotonically with $\epsilon$. 
This reflects the fact that PA becomes increasingly effective as the noise scale decreases, a trend observed without exception across all 20 scenarios.

However, interpretations of absolute accuracy necessitate careful consideration. 
The no-randomization accuracy, which serves as the upper bound on recoverable utility, itself degrades as corruption severity increases. 
This is particularly evident under Gaussian Noise, where the no-randomization accuracy drops to 0.2621 at severity 5 and drastically narrows the recoverable gap. 
Consequently, a lower recovery ratio for a specific mechanism, such as Laplace+PA, may not stem from a failure of PA to improve utility. 
Instead, it may result from a disproportionately large gap created by the baseline mechanism's own instability. 
The recovery ratio metric effectively decouples these distinct sources of variation.

Furthermore, the type of corruption modulates the stability of subspace identification. 
Brightness corruption yields a nearly constant recovery ratio across severities, suggesting that luminance shifts leave the downstream model's gradient directions largely intact. 
In contrast, Fog and Defocus Blur show a modest decline in recovery ratio as severity increases. 
This might be due to the progressive degradation of spatial frequency content, which creates a mismatch between public and private feature distributions. 
Gaussian Noise, conversely, exhibits an apparent increase in the recovery ratio at higher severities. 
This is not indicative of enhanced PA effectiveness. 
Rather, it reflects the rapid contraction of the recoverable gap as no-randomization accuracy collapses, which mechanically inflates the ratio even when absolute gains remain stable.

\begin{figure}[h]
  \centering
  \includegraphics[width=1.0\textwidth]{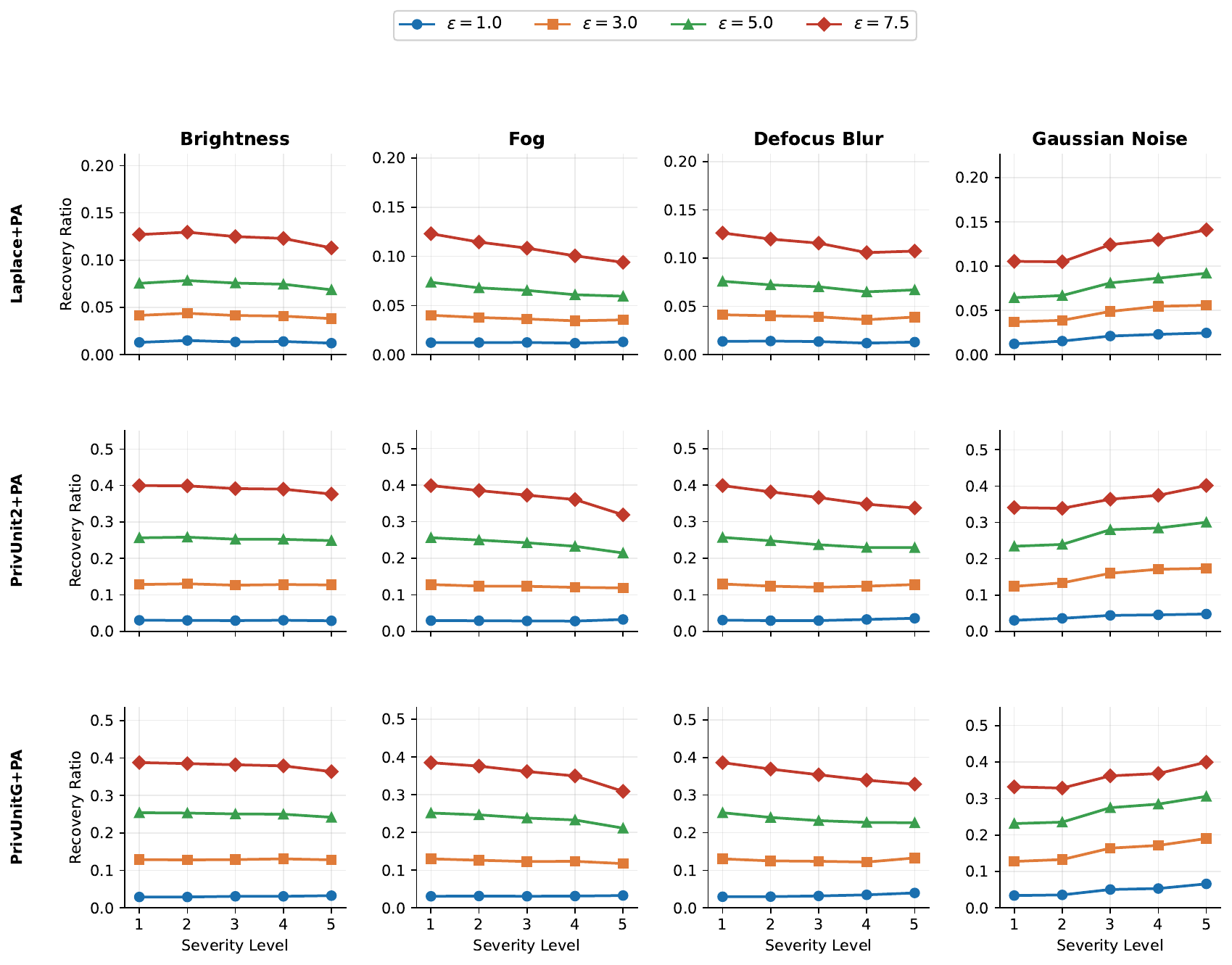} 
  \caption{\textbf{Recovery Ratio of PA across corruption types, severity levels, and privacy budgets.} Recovery Ratio is defined as $(\text{Acc(Baseline+PA)} - \text{Acc(Baseline)}) / (\text{Acc(No Randomization)} - \text{Acc(Baseline)})$. Results are shown for Laplace, PrivUnit2, and PrivUnitG baselines (rows) on four CIFAR-10-C corruption types (columns) at severity levels 1–5, with $\epsilon \in {1.0, 3.0, 5.0, 7.5}$. A higher Recovery Ratio indicates that PA restores a greater fraction of the accuracy degraded by LDP noise, independent of the absolute difficulty imposed by the corruption.}
  \label{fig:distribution_shift}
\end{figure}


\begin{table}[h]
\caption{Main results for \textbf{Nonlinear MLP Classifier with ReLU} $(m=64)$ on \textbf{CIFAR-10} and \textbf{CIFAR-10-C (Brightness)} under $\epsilon$-LDP mechanisms. We evaluate utility by measuring the Accuracy. Higher accuracy indicates higher utility.}
\centering
\resizebox{\textwidth}{!}{
\begin{tabular}{ll cccccccccccc}
\toprule
& \multicolumn{12}{c}{Privacy Budget ($\epsilon$)} \\
 \cmidrule(l){3-14}
 \textbf{Mechanism} & $S$ & 0.5 & 1.0 & 1.5 & 2.0 & 2.5 & 3.0 & 3.5 & 4.0 & 4.5 & 5.0 & 7.5 & 10.0 \\
 \midrule
 \multirow{6}{*} \makecell{No randomization} 
 & 0 & 0.8937 & 0.8937 & 0.8937 & 0.8937 & 0.8937 & 0.8937 & 0.8937 & 0.8937 & 0.8937 & 0.8937 & 0.8937 & 0.8937 \\
 & 1 & 0.8909 & 0.8909 & 0.8909 & 0.8909 & 0.8909 & 0.8909 & 0.8909 & 0.8909 & 0.8909 & 0.8909 & 0.8909 & 0.8909 \\
 & 2 & 0.8856 & 0.8856 & 0.8856 & 0.8856 & 0.8856 & 0.8856 & 0.8856 & 0.8856 & 0.8856 & 0.8856 & 0.8856 & 0.8856 \\
 & 3 & 0.8783 & 0.8783 & 0.8783 & 0.8783 & 0.8783 & 0.8783 & 0.8783 & 0.8783 & 0.8783 & 0.8783 & 0.8783 & 0.8783 \\
 & 4 & 0.8661 & 0.8661 & 0.8661 & 0.8661 & 0.8661 & 0.8661 & 0.8661 & 0.8661 & 0.8661 & 0.8661 & 0.8661 & 0.8661 \\
 & 5 & 0.8365 & 0.8365 & 0.8365 & 0.8365 & 0.8365 & 0.8365 & 0.8365 & 0.8365 & 0.8365 & 0.8365 & 0.8365 & 0.8365 \\
 \midrule
 \multirow{6}{*} \makecell{Laplace ($\ell_1$)}
 & 0 & \makecell{0.1002 \\ (0.0029)} & \makecell{0.1017 \\ (0.0029)} & \makecell{0.1031 \\ (0.0028)} & \makecell{0.1046 \\ (0.0028)} & \makecell{0.1059 \\ (0.0027)} & \makecell{0.1074 \\ (0.0026)} & \makecell{0.1088 \\ (0.0025)} & \makecell{0.1103 \\ (0.0027)} & \makecell{0.1119 \\ (0.0026)} & \makecell{0.1133 \\ (0.0025)} & \makecell{0.1212 \\ (0.0026)} & \makecell{0.1301 \\ (0.0024)} \\
 & 1 & \makecell{0.1002 \\ (0.0029)} & \makecell{0.1017 \\ (0.0029)} & \makecell{0.1031 \\ (0.0028)} & \makecell{0.1046 \\ (0.0028)} & \makecell{0.1060 \\ (0.0027)} & \makecell{0.1073 \\ (0.0026)} & \makecell{0.1088 \\ (0.0026)} & \makecell{0.1103 \\ (0.0027)} & \makecell{0.1119 \\ (0.0026)} & \makecell{0.1134 \\ (0.0025)} & \makecell{0.1212 \\ (0.0026)} & \makecell{0.1301 \\ (0.0025)} \\
 & 2 & \makecell{0.1002 \\ (0.0029)} & \makecell{0.1016 \\ (0.0029)} & \makecell{0.1030 \\ (0.0029)} & \makecell{0.1045 \\ (0.0028)} & \makecell{0.1059 \\ (0.0027)} & \makecell{0.1073 \\ (0.0026)} & \makecell{0.1086 \\ (0.0027)} & \makecell{0.1102 \\ (0.0026)} & \makecell{0.1116 \\ (0.0026)} & \makecell{0.1131 \\ (0.0026)} & \makecell{0.1209 \\ (0.0026)} & \makecell{0.1295 \\ (0.0026)} \\
 & 3 & \makecell{0.1002 \\ (0.0029)} & \makecell{0.1016 \\ (0.0029)} & \makecell{0.1030 \\ (0.0029)} & \makecell{0.1044 \\ (0.0028)} & \makecell{0.1058 \\ (0.0028)} & \makecell{0.1071 \\ (0.0026)} & \makecell{0.1084 \\ (0.0026)} & \makecell{0.1099 \\ (0.0026)} & \makecell{0.1113 \\ (0.0026)} & \makecell{0.1128 \\ (0.0026)} & \makecell{0.1203 \\ (0.0027)} & \makecell{0.1287 \\ (0.0026)} \\
 & 4 & \makecell{0.1001 \\ (0.0029)} & \makecell{0.1015 \\ (0.0029)} & \makecell{0.1028 \\ (0.0028)} & \makecell{0.1042 \\ (0.0028)} & \makecell{0.1056 \\ (0.0028)} & \makecell{0.1068 \\ (0.0027)} & \makecell{0.1081 \\ (0.0026)} & \makecell{0.1096 \\ (0.0026)} & \makecell{0.1109 \\ (0.0026)} & \makecell{0.1124 \\ (0.0025)} & \makecell{0.1196 \\ (0.0026)} & \makecell{0.1276 \\ (0.0024)} \\
 & 5 & \makecell{0.1001 \\ (0.0029)} & \makecell{0.1013 \\ (0.0028)} & \makecell{0.1025 \\ (0.0028)} & \makecell{0.1038 \\ (0.0028)} & \makecell{0.1050 \\ (0.0028)} & \makecell{0.1063 \\ (0.0027)} & \makecell{0.1074 \\ (0.0026)} & \makecell{0.1087 \\ (0.0026)} & \makecell{0.1099 \\ (0.0026)} & \makecell{0.1111 \\ (0.0025)} & \makecell{0.1177 \\ (0.0027)} & \makecell{0.1251 \\ (0.0025)} \\
 
 \midrule
 \multirow{6}{*} \makecell{Laplace+PA} 
 & 0 & \makecell{0.1062 \\ (0.0024)} & \makecell{0.1122 \\ (0.0022)} & \makecell{0.1188 \\ (0.0023)} & \makecell{0.1258 \\ (0.0025)} & \makecell{0.1329 \\ (0.0023)} & \makecell{0.1404 \\ (0.0023)} & \makecell{0.1480 \\ (0.0024)} & \makecell{0.1563 \\ (0.0023)} & \makecell{0.1650 \\ (0.0024)} & \makecell{0.1737 \\ (0.0024)} & \makecell{0.2204 \\ (0.0038)} & \makecell{0.2723 \\ (0.0046)} \\
 & 1 & \makecell{0.1059 \\ (0.0024)} & \makecell{0.1120 \\ (0.0027)} & \makecell{0.1187 \\ (0.0027)} & \makecell{0.1256 \\ (0.0025)} & \makecell{0.1328 \\ (0.0031)} & \makecell{0.1399 \\ (0.0032)} & \makecell{0.1475 \\ (0.0030)} & \makecell{0.1554 \\ (0.0033)} & \makecell{0.1637 \\ (0.0034)} & \makecell{0.1721 \\ (0.0034)} & \makecell{0.2190 \\ (0.0036)} & \makecell{0.2703 \\ (0.0045)} \\
 & 2 & \makecell{0.1071 \\ (0.0023)} & \makecell{0.1134 \\ (0.0023)} & \makecell{0.1201 \\ (0.0022)} & \makecell{0.1270 \\ (0.0024)} & \makecell{0.1340 \\ (0.0025)} & \makecell{0.1414 \\ (0.0030)} & \makecell{0.1489 \\ (0.0033)} & \makecell{0.1568 \\ (0.0034)} & \makecell{0.1651 \\ (0.0035)} & \makecell{0.1737 \\ (0.0035)} & \makecell{0.2200 \\ (0.0044)} & \makecell{0.2704 \\ (0.0054)} \\
 & 3 & \makecell{0.1063 \\ (0.0022)} & \makecell{0.1121 \\ (0.0021)} & \makecell{0.1184 \\ (0.0022)} & \makecell{0.1248 \\ (0.0023)} & \makecell{0.1321 \\ (0.0026)} & \makecell{0.1391 \\ (0.0026)} & \makecell{0.1466 \\ (0.0025)} & \makecell{0.1545 \\ (0.0026)} & \makecell{0.1628 \\ (0.0026)} & \makecell{0.1708 \\ (0.0026)} & \makecell{0.2150 \\ (0.0034)} & \makecell{0.2650 \\ (0.0046)} \\
 & 4 & \makecell{0.1061 \\ (0.0028)} & \makecell{0.1122 \\ (0.0026)} & \makecell{0.1181 \\ (0.0027)} & \makecell{0.1241 \\ (0.0031)} & \makecell{0.1307 \\ (0.0034)} & \makecell{0.1378 \\ (0.0034)} & \makecell{0.1452 \\ (0.0032)} & \makecell{0.1528 \\ (0.0034)} & \makecell{0.1605 \\ (0.0036)} & \makecell{0.1686 \\ (0.0036)} & \makecell{0.2114 \\ (0.0050)} & \makecell{0.2583 \\ (0.0044)} \\
 & 5 & \makecell{0.1048 \\ (0.0033)} & \makecell{0.1103 \\ (0.0033)} & \makecell{0.1161 \\ (0.0033)} & \makecell{0.1219 \\ (0.0036)} & \makecell{0.1280 \\ (0.0034)} & \makecell{0.1342 \\ (0.0034)} & \makecell{0.1405 \\ (0.0035)} & \makecell{0.1467 \\ (0.0037)} & \makecell{0.1538 \\ (0.0039)} & \makecell{0.1609 \\ (0.0037)} & \makecell{0.1989 \\ (0.0045)} & \makecell{0.2416 \\ (0.0050)} \\
 
 \midrule
 \multirow{6}{*} \makecell{PrivUnit2} 
 & 0 & \makecell{0.1107 \\ (0.0032)} & \makecell{0.1221 \\ (0.0032)} & \makecell{0.1348 \\ (0.0032)} & \makecell{0.1481 \\ (0.0031)} & \makecell{0.1620 \\ (0.0028)} & \makecell{0.1762 \\ (0.0029)} & \makecell{0.1913 \\ (0.0038)} & \makecell{0.2063 \\ (0.0044)} & \makecell{0.2229 \\ (0.0044)} & \makecell{0.2383 \\ (0.0040)} & \makecell{0.3185 \\ (0.0042)} & \makecell{0.3978 \\ (0.0035)} \\
 & 1 & \makecell{0.1105 \\ (0.0030)} & \makecell{0.1219 \\ (0.0030)} & \makecell{0.1346 \\ (0.0031)} & \makecell{0.1480 \\ (0.0034)} & \makecell{0.1620 \\ (0.0029)} & \makecell{0.1761 \\ (0.0029)} & \makecell{0.1911 \\ (0.0038)} & \makecell{0.2059 \\ (0.0045)} & \makecell{0.2224 \\ (0.0046)} & \makecell{0.2378 \\ (0.0041)} & \makecell{0.3179 \\ (0.0044)} & \makecell{0.3962 \\ (0.0035)} \\
 & 2 & \makecell{0.1104 \\ (0.0032)} & \makecell{0.1216 \\ (0.0031)} & \makecell{0.1342 \\ (0.0029)} & \makecell{0.1475 \\ (0.0032)} & \makecell{0.1613 \\ (0.0028)} & \makecell{0.1754 \\ (0.0026)} & \makecell{0.1901 \\ (0.0033)} & \makecell{0.2048 \\ (0.0041)} & \makecell{0.2211 \\ (0.0043)} & \makecell{0.2362 \\ (0.0040)} & \makecell{0.3154 \\ (0.0047)} & \makecell{0.3930 \\ (0.0039)} \\
 & 3 & \makecell{0.1103 \\ (0.0034)} & \makecell{0.1214 \\ (0.0033)} & \makecell{0.1336 \\ (0.0031)} & \makecell{0.1468 \\ (0.0034)} & \makecell{0.1603 \\ (0.0029)} & \makecell{0.1740 \\ (0.0028)} & \makecell{0.1884 \\ (0.0036)} & \makecell{0.2031 \\ (0.0039)} & \makecell{0.2188 \\ (0.0040)} & \makecell{0.2338 \\ (0.0041)} & \makecell{0.3117 \\ (0.0047)} & \makecell{0.3873 \\ (0.0038)} \\
 & 4 & \makecell{0.1103 \\ (0.0033)} & \makecell{0.1212 \\ (0.0032)} & \makecell{0.1332 \\ (0.0030)} & \makecell{0.1458 \\ (0.0032)} & \makecell{0.1591 \\ (0.0029)} & \makecell{0.1727 \\ (0.0029)} & \makecell{0.1867 \\ (0.0038)} & \makecell{0.2011 \\ (0.0043)} & \makecell{0.2162 \\ (0.0043)} & \makecell{0.2308 \\ (0.0043)} & \makecell{0.3068 \\ (0.0037)} & \makecell{0.3802 \\ (0.0030)} \\
 & 5 & \makecell{0.1099 \\ (0.0035)} & \makecell{0.1203 \\ (0.0033)} & \makecell{0.1316 \\ (0.0034)} & \makecell{0.1434 \\ (0.0037)} & \makecell{0.1560 \\ (0.0033)} & \makecell{0.1686 \\ (0.0034)} & \makecell{0.1820 \\ (0.0040)} & \makecell{0.1950 \\ (0.0047)} & \makecell{0.2092 \\ (0.0043)} & \makecell{0.2227 \\ (0.0044)} & \makecell{0.2933 \\ (0.0028)} & \makecell{0.3615 \\ (0.0024)} \\
 
 \midrule
 \multirow{6}{*} \makecell{PrivUnit2+PA} 
 & 0 & \makecell{0.1203 \\ (0.0022)} & \makecell{0.1440 \\ (0.0030)} & \makecell{0.1710 \\ (0.0027)} & \makecell{0.2012 \\ (0.0028)} & \makecell{0.2341 \\ (0.0031)} & \makecell{0.2673 \\ (0.0040)} & \makecell{0.3027 \\ (0.0039)} & \makecell{0.3377 \\ (0.0041)} & \makecell{0.3733 \\ (0.0048)} & \makecell{0.4076 \\ (0.0052)} & \makecell{0.5477 \\ (0.0055)} & \makecell{0.6451 \\ (0.0057)} \\
 & 1 & \makecell{0.1216 \\ (0.0032)} & \makecell{0.1452 \\ (0.0030)} & \makecell{0.1722 \\ (0.0037)} & \makecell{0.2017 \\ (0.0034)} & \makecell{0.2339 \\ (0.0036)} & \makecell{0.2679 \\ (0.0039)} & \makecell{0.3029 \\ (0.0034)} & \makecell{0.3368 \\ (0.0039)} & \makecell{0.3724 \\ (0.0039)} & \makecell{0.4053 \\ (0.0043)} & \makecell{0.5470 \\ (0.0033)} & \makecell{0.6432 \\ (0.0052)} \\
 & 2 & \makecell{0.1206 \\ (0.0027)} & \makecell{0.1445 \\ (0.0023)} & \makecell{0.1720 \\ (0.0028)} & \makecell{0.2018 \\ (0.0030)} & \makecell{0.2345 \\ (0.0038)} & \makecell{0.2680 \\ (0.0044)} & \makecell{0.3029 \\ (0.0045)} & \makecell{0.3365 \\ (0.0050)} & \makecell{0.3708 \\ (0.0054)} & \makecell{0.4039 \\ (0.0051)} & \makecell{0.5429 \\ (0.0050)} & \makecell{0.6386 \\ (0.0045)} \\
 & 3 & \makecell{0.1204 \\ (0.0026)} & \makecell{0.1437 \\ (0.0028)} & \makecell{0.1702 \\ (0.0031)} & \makecell{0.1993 \\ (0.0031)} & \makecell{0.2304 \\ (0.0033)} & \makecell{0.2631 \\ (0.0030)} & \makecell{0.2964 \\ (0.0035)} & \makecell{0.3302 \\ (0.0039)} & \makecell{0.3644 \\ (0.0032)} & \makecell{0.3966 \\ (0.0035)} & \makecell{0.5335 \\ (0.0045)} & \makecell{0.6277 \\ (0.0038)} \\
 & 4 & \makecell{0.1205 \\ (0.0027)} & \makecell{0.1438 \\ (0.0029)} & \makecell{0.1699 \\ (0.0029)} & \makecell{0.1987 \\ (0.0026)} & \makecell{0.2299 \\ (0.0021)} & \makecell{0.2617 \\ (0.0026)} & \makecell{0.2949 \\ (0.0024)} & \makecell{0.3280 \\ (0.0028)} & \makecell{0.3603 \\ (0.0028)} & \makecell{0.3911 \\ (0.0038)} & \makecell{0.5249 \\ (0.0050)} & \makecell{0.6165 \\ (0.0055)} \\
 & 5 & \makecell{0.1194 \\ (0.0020)} & \makecell{0.1410 \\ (0.0026)} & \makecell{0.1660 \\ (0.0034)} & \makecell{0.1936 \\ (0.0031)} & \makecell{0.2223 \\ (0.0034)} & \makecell{0.2536 \\ (0.0043)} & \makecell{0.2840 \\ (0.0049)} & \makecell{0.3152 \\ (0.0054)} & \makecell{0.3453 \\ (0.0062)} & \makecell{0.3754 \\ (0.0059)} & \makecell{0.4978 \\ (0.0049)} & \makecell{0.5840 \\ (0.0045)} \\
 
 \midrule
 \multirow{6}{*} \makecell{PrivUnitG} 
 & 0 & \makecell{0.1100 \\ (0.0027)} & \makecell{0.1212 \\ (0.0030)} & \makecell{0.1335 \\ (0.0028)} & \makecell{0.1464 \\ (0.0035)} & \makecell{0.1600 \\ (0.0037)} & \makecell{0.1741 \\ (0.0040)} & \makecell{0.1887 \\ (0.0044)} & \makecell{0.2036 \\ (0.0039)} & \makecell{0.2187 \\ (0.0040)} & \makecell{0.2342 \\ (0.0043)} & \makecell{0.3103 \\ (0.0034)} & \makecell{0.3685 \\ (0.0035)} \\
 & 1 & \makecell{0.1097 \\ (0.0030)} & \makecell{0.1210 \\ (0.0028)} & \makecell{0.1335 \\ (0.0028)} & \makecell{0.1465 \\ (0.0034)} & \makecell{0.1599 \\ (0.0038)} & \makecell{0.1738 \\ (0.0040)} & \makecell{0.1885 \\ (0.0046)} & \makecell{0.2034 \\ (0.0042)} & \makecell{0.2182 \\ (0.0042)} & \makecell{0.2336 \\ (0.0045)} & \makecell{0.3100 \\ (0.0037)} & \makecell{0.3679 \\ (0.0033)} \\
 & 2 & \makecell{0.1096 \\ (0.0029)} & \makecell{0.1207 \\ (0.0026)} & \makecell{0.1332 \\ (0.0028)} & \makecell{0.1460 \\ (0.0036)} & \makecell{0.1594 \\ (0.0037)} & \makecell{0.1731 \\ (0.0038)} & \makecell{0.1877 \\ (0.0044)} & \makecell{0.2024 \\ (0.0039)} & \makecell{0.2172 \\ (0.0043)} & \makecell{0.2324 \\ (0.0042)} & \makecell{0.3076 \\ (0.0041)} & \makecell{0.3653 \\ (0.0036)} \\
 & 3 & \makecell{0.1096 \\ (0.0029)} & \makecell{0.1206 \\ (0.0027)} & \makecell{0.1326 \\ (0.0028)} & \makecell{0.1453 \\ (0.0033)} & \makecell{0.1587 \\ (0.0041)} & \makecell{0.1723 \\ (0.0038)} & \makecell{0.1864 \\ (0.0043)} & \makecell{0.2007 \\ (0.0041)} & \makecell{0.2152 \\ (0.0042)} & \makecell{0.2301 \\ (0.0044)} & \makecell{0.3037 \\ (0.0041)} & \makecell{0.3605 \\ (0.0032)} \\
 & 4 & \makecell{0.1092 \\ (0.0029)} & \makecell{0.1202 \\ (0.0031)} & \makecell{0.1321 \\ (0.0029)} & \makecell{0.1443 \\ (0.0032)} & \makecell{0.1577 \\ (0.0040)} & \makecell{0.1709 \\ (0.0040)} & \makecell{0.1848 \\ (0.0042)} & \makecell{0.1986 \\ (0.0041)} & \makecell{0.2127 \\ (0.0040)} & \makecell{0.2276 \\ (0.0041)} & \makecell{0.2986 \\ (0.0039)} & \makecell{0.3539 \\ (0.0035)} \\
 & 5 & \makecell{0.1083 \\ (0.0030)} & \makecell{0.1190 \\ (0.0031)} & \makecell{0.1299 \\ (0.0031)} & \makecell{0.1415 \\ (0.0032)} & \makecell{0.1541 \\ (0.0041)} & \makecell{0.1665 \\ (0.0040)} & \makecell{0.1798 \\ (0.0042)} & \makecell{0.1931 \\ (0.0038)} & \makecell{0.2061 \\ (0.0038)} & \makecell{0.2197 \\ (0.0040)} & \makecell{0.2858 \\ (0.0039)} & \makecell{0.3375 \\ (0.0034)} \\
 
 \midrule
 \multirow{6}{*} \makecell{PrivUnitG+PA} 
 & 0 & \makecell{0.1213 \\ (0.0041)} & \makecell{0.1456 \\ (0.0049)} & \makecell{0.1726 \\ (0.0050)} & \makecell{0.2020 \\ (0.0047)} & \makecell{0.2349 \\ (0.0044)} & \makecell{0.2684 \\ (0.0043)} & \makecell{0.3029 \\ (0.0041)} & \makecell{0.3367 \\ (0.0043)} & \makecell{0.3700 \\ (0.0044)} & \makecell{0.4016 \\ (0.0046)} & \makecell{0.5364 \\ (0.0044)} & \makecell{0.6254 \\ (0.0045)} \\
 & 1 & \makecell{0.1199 \\ (0.0030)} & \makecell{0.1435 \\ (0.0030)} & \makecell{0.1706 \\ (0.0033)} & \makecell{0.2000 \\ (0.0029)} & \makecell{0.2326 \\ (0.0030)} & \makecell{0.2662 \\ (0.0031)} & \makecell{0.3002 \\ (0.0034)} & \makecell{0.3350 \\ (0.0034)} & \makecell{0.3685 \\ (0.0036)} & \makecell{0.4005 \\ (0.0036)} & \makecell{0.5352 \\ (0.0035)} & \makecell{0.6234 \\ (0.0039)} \\
 & 2 & \makecell{0.1199 \\ (0.0022)} & \makecell{0.1430 \\ (0.0026)} & \makecell{0.1697 \\ (0.0037)} & \makecell{0.1992 \\ (0.0047)} & \makecell{0.2308 \\ (0.0047)} & \makecell{0.2644 \\ (0.0050)} & \makecell{0.2985 \\ (0.0056)} & \makecell{0.3328 \\ (0.0060)} & \makecell{0.3660 \\ (0.0062)} & \makecell{0.3977 \\ (0.0059)} & \makecell{0.5301 \\ (0.0062)} & \makecell{0.6182 \\ (0.0051)} \\
 & 3 & \makecell{0.1207 \\ (0.0029)} & \makecell{0.1441 \\ (0.0030)} & \makecell{0.1710 \\ (0.0028)} & \makecell{0.1996 \\ (0.0029)} & \makecell{0.2310 \\ (0.0033)} & \makecell{0.2633 \\ (0.0038)} & \makecell{0.2963 \\ (0.0038)} & \makecell{0.3287 \\ (0.0039)} & \makecell{0.3618 \\ (0.0041)} & \makecell{0.3926 \\ (0.0048)} & \makecell{0.5231 \\ (0.0048)} & \makecell{0.6083 \\ (0.0059)} \\
 & 4 & \makecell{0.1204 \\ (0.0029)} & \makecell{0.1433 \\ (0.0035)} & \makecell{0.1697 \\ (0.0037)} & \makecell{0.1983 \\ (0.0034)} & \makecell{0.2298 \\ (0.0033)} & \makecell{0.2619 \\ (0.0032)} & \makecell{0.2942 \\ (0.0035)} & \makecell{0.3257 \\ (0.0033)} & \makecell{0.3576 \\ (0.0040)} & \makecell{0.3872 \\ (0.0041)} & \makecell{0.5136 \\ (0.0040)} & \makecell{0.5976 \\ (0.0045)} \\
 & 5 & \makecell{0.1201 \\ (0.0038)} & \makecell{0.1425 \\ (0.0038)} & \makecell{0.1672 \\ (0.0035)} & \makecell{0.1936 \\ (0.0043)} & \makecell{0.2227 \\ (0.0043)} & \makecell{0.2524 \\ (0.0041)} & \makecell{0.2822 \\ (0.0052)} & \makecell{0.3121 \\ (0.0056)} & \makecell{0.3414 \\ (0.0061)} & \makecell{0.3688 \\ (0.0057)} & \makecell{0.4860 \\ (0.0045)} & \makecell{0.5648 \\ (0.0048)} \\
 
 \midrule
 \multirow{6}{*} \makecell{CW+PA w/o bounding} 
 & 0 & \makecell{0.1039 \\ (0.0024)} & \makecell{0.1077 \\ (0.0025)} & \makecell{0.1120 \\ (0.0026)} & \makecell{0.1161 \\ (0.0026)} & \makecell{0.1205 \\ (0.0026)} & \makecell{0.1251 \\ (0.0027)} & \makecell{0.1298 \\ (0.0026)} & \makecell{0.1346 \\ (0.0025)} & \makecell{0.1397 \\ (0.0026)} & \makecell{0.1446 \\ (0.0028)} & \makecell{0.1716 \\ (0.0027)} & \makecell{0.2008 \\ (0.0026)} \\
 & 1 & \makecell{0.1038 \\ (0.0030)} & \makecell{0.1075 \\ (0.0032)} & \makecell{0.1115 \\ (0.0030)} & \makecell{0.1156 \\ (0.0028)} & \makecell{0.1200 \\ (0.0029)} & \makecell{0.1245 \\ (0.0027)} & \makecell{0.1291 \\ (0.0029)} & \makecell{0.1338 \\ (0.0029)} & \makecell{0.1386 \\ (0.0030)} & \makecell{0.1434 \\ (0.0029)} & \makecell{0.1698 \\ (0.0038)} & \makecell{0.1999 \\ (0.0041)} \\
 & 2 & \makecell{0.1050 \\ (0.0022)} & \makecell{0.1088 \\ (0.0022)} & \makecell{0.1128 \\ (0.0022)} & \makecell{0.1174 \\ (0.0023)} & \makecell{0.1215 \\ (0.0024)} & \makecell{0.1261 \\ (0.0025)} & \makecell{0.1305 \\ (0.0031)} & \makecell{0.1353 \\ (0.0031)} & \makecell{0.1399 \\ (0.0032)} & \makecell{0.1449 \\ (0.0034)} & \makecell{0.1715 \\ (0.0035)} & \makecell{0.2003 \\ (0.0041)} \\
 & 3 & \makecell{0.1041 \\ (0.0026)} & \makecell{0.1076 \\ (0.0025)} & \makecell{0.1114 \\ (0.0026)} & \makecell{0.1155 \\ (0.0027)} & \makecell{0.1196 \\ (0.0028)} & \makecell{0.1239 \\ (0.0030)} & \makecell{0.1285 \\ (0.0029)} & \makecell{0.1329 \\ (0.0031)} & \makecell{0.1380 \\ (0.0029)} & \makecell{0.1429 \\ (0.0032)} & \makecell{0.1686 \\ (0.0033)} & \makecell{0.1961 \\ (0.0035)} \\
 & 4 & \makecell{0.1046 \\ (0.0032)} & \makecell{0.1082 \\ (0.0032)} & \makecell{0.1119 \\ (0.0031)} & \makecell{0.1159 \\ (0.0030)} & \makecell{0.1197 \\ (0.0031)} & \makecell{0.1238 \\ (0.0034)} & \makecell{0.1281 \\ (0.0033)} & \makecell{0.1325 \\ (0.0032)} & \makecell{0.1370 \\ (0.0032)} & \makecell{0.1416 \\ (0.0033)} & \makecell{0.1662 \\ (0.0040)} & \makecell{0.1934 \\ (0.0050)} \\
 & 5 & \makecell{0.1025 \\ (0.0032)} & \makecell{0.1058 \\ (0.0030)} & \makecell{0.1092 \\ (0.0031)} & \makecell{0.1128 \\ (0.0030)} & \makecell{0.1164 \\ (0.0033)} & \makecell{0.1201 \\ (0.0036)} & \makecell{0.1241 \\ (0.0039)} & \makecell{0.1282 \\ (0.0037)} & \makecell{0.1323 \\ (0.0038)} & \makecell{0.1366 \\ (0.0037)} & \makecell{0.1581 \\ (0.0042)} & \makecell{0.1828 \\ (0.0040)} \\
 \bottomrule
\end{tabular}
}\label{tb:brightness}
\end{table}


\begin{table}[h]
\caption{Main results for \textbf{Nonlinear MLP Classifier with ReLU} $(m=64)$ on \textbf{CIFAR-10} and \textbf{CIFAR-10-C (Fog)} under $\epsilon$-LDP mechanisms. We evaluate utility by measuring the Accuracy. Higher accuracy indicates higher utility.}
\centering
\resizebox{\textwidth}{!}{
\begin{tabular}{ll cccccccccccc}
\toprule
& \multicolumn{12}{c}{Privacy Budget ($\epsilon$)} \\
 \cmidrule(l){3-14}
 \textbf{Mechanism} & $S$ & 0.5 & 1.0 & 1.5 & 2.0 & 2.5 & 3.0 & 3.5 & 4.0 & 4.5 & 5.0 & 7.5 & 10.0 \\
 \midrule
 \multirow{6}{*} \makecell {No randomization} 
 & 0 & 0.8937 & 0.8937 & 0.8937 & 0.8937 & 0.8937 & 0.8937 & 0.8937 & 0.8937 & 0.8937 & 0.8937 & 0.8937 & 0.8937 \\
 & 1 & 0.8924 & 0.8924 & 0.8924 & 0.8924 & 0.8924 & 0.8924 & 0.8924 & 0.8924 & 0.8924 & 0.8924 & 0.8924 & 0.8924 \\
 & 2 & 0.8701 & 0.8701 & 0.8701 & 0.8701 & 0.8701 & 0.8701 & 0.8701 & 0.8701 & 0.8701 & 0.8701 & 0.8701 & 0.8701 \\
 & 3 & 0.8460 & 0.8460 & 0.8460 & 0.8460 & 0.8460 & 0.8460 & 0.8460 & 0.8460 & 0.8460 & 0.8460 & 0.8460 & 0.8460 \\
 & 4 & 0.7978 & 0.7978 & 0.7978 & 0.7978 & 0.7978 & 0.7978 & 0.7978 & 0.7978 & 0.7978 & 0.7978 & 0.7978 & 0.7978 \\
 & 5 & 0.6138 & 0.6138 & 0.6138 & 0.6138 & 0.6138 & 0.6138 & 0.6138 & 0.6138 & 0.6138 & 0.6138 & 0.6138 & 0.6138 \\
 
 \midrule
 \multirow{6}{*} \makecell{Laplace ($\ell_1$)}
 & 0 & \makecell{0.1002 \\ (0.0029)} & \makecell{0.1017 \\ (0.0029)} & \makecell{0.1031 \\ (0.0028)} & \makecell{0.1046 \\ (0.0028)} & \makecell{0.1059 \\ (0.0027)} & \makecell{0.1074 \\ (0.0026)} & \makecell{0.1088 \\ (0.0025)} & \makecell{0.1103 \\ (0.0027)} & \makecell{0.1119 \\ (0.0026)} & \makecell{0.1133 \\ (0.0025)} & \makecell{0.1212 \\ (0.0026)} & \makecell{0.1301 \\ (0.0024)} \\
 & 1 & \makecell{0.1002 \\ (0.0028)} & \makecell{0.1016 \\ (0.0029)} & \makecell{0.1030 \\ (0.0028)} & \makecell{0.1044 \\ (0.0028)} & \makecell{0.1058 \\ (0.0027)} & \makecell{0.1072 \\ (0.0026)} & \makecell{0.1085 \\ (0.0026)} & \makecell{0.1100 \\ (0.0025)} & \makecell{0.1115 \\ (0.0027)} & \makecell{0.1131 \\ (0.0026)} & \makecell{0.1207 \\ (0.0027)} & \makecell{0.1291 \\ (0.0026)} \\
 & 2 & \makecell{0.1001 \\ (0.0029)} & \makecell{0.1014 \\ (0.0029)} & \makecell{0.1026 \\ (0.0029)} & \makecell{0.1038 \\ (0.0027)} & \makecell{0.1052 \\ (0.0028)} & \makecell{0.1064 \\ (0.0027)} & \makecell{0.1076 \\ (0.0027)} & \makecell{0.1091 \\ (0.0026)} & \makecell{0.1105 \\ (0.0027)} & \makecell{0.1118 \\ (0.0026)} & \makecell{0.1187 \\ (0.0026)} & \makecell{0.1263 \\ (0.0026)} \\
 & 3 & \makecell{0.1000 \\ (0.0029)} & \makecell{0.1011 \\ (0.0029)} & \makecell{0.1023 \\ (0.0029)} & \makecell{0.1034 \\ (0.0028)} & \makecell{0.1046 \\ (0.0028)} & \makecell{0.1058 \\ (0.0028)} & \makecell{0.1069 \\ (0.0027)} & \makecell{0.1082 \\ (0.0025)} & \makecell{0.1095 \\ (0.0026)} & \makecell{0.1108 \\ (0.0025)} & \makecell{0.1171 \\ (0.0027)} & \makecell{0.1240 \\ (0.0027)} \\
 & 4 & \makecell{0.0999 \\ (0.0029)} & \makecell{0.1009 \\ (0.0029)} & \makecell{0.1020 \\ (0.0029)} & \makecell{0.1029 \\ (0.0028)} & \makecell{0.1039 \\ (0.0026)} & \makecell{0.1049 \\ (0.0027)} & \makecell{0.1060 \\ (0.0027)} & \makecell{0.1072 \\ (0.0027)} & \makecell{0.1082 \\ (0.0026)} & \makecell{0.1094 \\ (0.0027)} & \makecell{0.1149 \\ (0.0027)} & \makecell{0.1210 \\ (0.0027)} \\
 & 5 & \makecell{0.0996 \\ (0.0029)} & \makecell{0.1003 \\ (0.0029)} & \makecell{0.1010 \\ (0.0029)} & \makecell{0.1017 \\ (0.0028)} & \makecell{0.1024 \\ (0.0027)} & \makecell{0.1033 \\ (0.0027)} & \makecell{0.1041 \\ (0.0026)} & \makecell{0.1047 \\ (0.0026)} & \makecell{0.1054 \\ (0.0026)} & \makecell{0.1063 \\ (0.0026)} & \makecell{0.1105 \\ (0.0028)} & \makecell{0.1146 \\ (0.0028)} \\
 
 \midrule
 \multirow{6}{*} \makecell{Laplace+PA} 
 & 0 & \makecell{0.1062 \\ (0.0024)} & \makecell{0.1122 \\ (0.0022)} & \makecell{0.1188 \\ (0.0023)} & \makecell{0.1258 \\ (0.0025)} & \makecell{0.1329 \\ (0.0023)} & \makecell{0.1404 \\ (0.0023)} & \makecell{0.1480 \\ (0.0024)} & \makecell{0.1563 \\ (0.0023)} & \makecell{0.1650 \\ (0.0024)} & \makecell{0.1737 \\ (0.0024)} & \makecell{0.2204 \\ (0.0038)} & \makecell{0.2723 \\ (0.0046)} \\
 & 1 & \makecell{0.1052 \\ (0.0021)} & \makecell{0.1113 \\ (0.0024)} & \makecell{0.1176 \\ (0.0024)} & \makecell{0.1246 \\ (0.0025)} & \makecell{0.1313 \\ (0.0028)} & \makecell{0.1387 \\ (0.0028)} & \makecell{0.1462 \\ (0.0025)} & \makecell{0.1541 \\ (0.0023)} & \makecell{0.1622 \\ (0.0023)} & \makecell{0.1705 \\ (0.0025)} & \makecell{0.2157 \\ (0.0033)} & \makecell{0.2654 \\ (0.0035)} \\
 & 2 & \makecell{0.1053 \\ (0.0023)} & \makecell{0.1109 \\ (0.0019)} & \makecell{0.1168 \\ (0.0021)} & \makecell{0.1228 \\ (0.0020)} & \makecell{0.1291 \\ (0.0025)} & \makecell{0.1353 \\ (0.0026)} & \makecell{0.1424 \\ (0.0029)} & \makecell{0.1494 \\ (0.0033)} & \makecell{0.1561 \\ (0.0035)} & \makecell{0.1634 \\ (0.0035)} & \makecell{0.2048 \\ (0.0040)} & \makecell{0.2497 \\ (0.0053)} \\
 & 3 & \makecell{0.1055 \\ (0.0033)} & \makecell{0.1104 \\ (0.0032)} & \makecell{0.1157 \\ (0.0036)} & \makecell{0.1213 \\ (0.0037)} & \makecell{0.1268 \\ (0.0039)} & \makecell{0.1327 \\ (0.0041)} & \makecell{0.1387 \\ (0.0040)} & \makecell{0.1454 \\ (0.0038)} & \makecell{0.1519 \\ (0.0040)} & \makecell{0.1589 \\ (0.0042)} & \makecell{0.1961 \\ (0.0044)} & \makecell{0.2377 \\ (0.0047)} \\
 & 4 & \makecell{0.1046 \\ (0.0037)} & \makecell{0.1091 \\ (0.0038)} & \makecell{0.1139 \\ (0.0039)} & \makecell{0.1189 \\ (0.0038)} & \makecell{0.1240 \\ (0.0042)} & \makecell{0.1288 \\ (0.0040)} & \makecell{0.1342 \\ (0.0040)} & \makecell{0.1397 \\ (0.0040)} & \makecell{0.1454 \\ (0.0041)} & \makecell{0.1514 \\ (0.0040)} & \makecell{0.1836 \\ (0.0048)} & \makecell{0.2195 \\ (0.0042)} \\
 & 5 & \makecell{0.1040 \\ (0.0028)} & \makecell{0.1070 \\ (0.0031)} & \makecell{0.1106 \\ (0.0031)} & \makecell{0.1142 \\ (0.0031)} & \makecell{0.1178 \\ (0.0031)} & \makecell{0.1214 \\ (0.0033)} & \makecell{0.1250 \\ (0.0035)} & \makecell{0.1286 \\ (0.0035)} & \makecell{0.1323 \\ (0.0033)} & \makecell{0.1365 \\ (0.0033)} & \makecell{0.1578 \\ (0.0033)} & \makecell{0.1814 \\ (0.0033)} \\
 
 \midrule
 \multirow{6}{*} \makecell{PrivUnit2} 
 & 0 & \makecell{0.1107 \\ (0.0032)} & \makecell{0.1221 \\ (0.0032)} & \makecell{0.1348 \\ (0.0032)} & \makecell{0.1481 \\ (0.0031)} & \makecell{0.1620 \\ (0.0028)} & \makecell{0.1762 \\ (0.0029)} & \makecell{0.1913 \\ (0.0038)} & \makecell{0.2063 \\ (0.0044)} & \makecell{0.2229 \\ (0.0044)} & \makecell{0.2383 \\ (0.0040)} & \makecell{0.3185 \\ (0.0042)} & \makecell{0.3978 \\ (0.0035)} \\
 & 1 & \makecell{0.1102 \\ (0.0031)} & \makecell{0.1216 \\ (0.0031)} & \makecell{0.1341 \\ (0.0033)} & \makecell{0.1474 \\ (0.0035)} & \makecell{0.1610 \\ (0.0031)} & \makecell{0.1751 \\ (0.0029)} & \makecell{0.1898 \\ (0.0035)} & \makecell{0.2047 \\ (0.0042)} & \makecell{0.2212 \\ (0.0042)} & \makecell{0.2364 \\ (0.0039)} & \makecell{0.3155 \\ (0.0042)} & \makecell{0.3937 \\ (0.0036)} \\
 & 2 & \makecell{0.1099 \\ (0.0034)} & \makecell{0.1209 \\ (0.0031)} & \makecell{0.1328 \\ (0.0031)} & \makecell{0.1455 \\ (0.0033)} & \makecell{0.1586 \\ (0.0029)} & \makecell{0.1719 \\ (0.0033)} & \makecell{0.1859 \\ (0.0036)} & \makecell{0.1999 \\ (0.0039)} & \makecell{0.2154 \\ (0.0045)} & \makecell{0.2293 \\ (0.0040)} & \makecell{0.3046 \\ (0.0038)} & \makecell{0.3786 \\ (0.0030)} \\
 & 3 & \makecell{0.1094 \\ (0.0033)} & \makecell{0.1198 \\ (0.0033)} & \makecell{0.1313 \\ (0.0034)} & \makecell{0.1432 \\ (0.0031)} & \makecell{0.1553 \\ (0.0031)} & \makecell{0.1675 \\ (0.0036)} & \makecell{0.1806 \\ (0.0042)} & \makecell{0.1936 \\ (0.0044)} & \makecell{0.2082 \\ (0.0046)} & \makecell{0.2209 \\ (0.0040)} & \makecell{0.2921 \\ (0.0036)} & \makecell{0.3610 \\ (0.0025)} \\
 & 4 & \makecell{0.1091 \\ (0.0036)} & \makecell{0.1187 \\ (0.0035)} & \makecell{0.1291 \\ (0.0036)} & \makecell{0.1399 \\ (0.0034)} & \makecell{0.1508 \\ (0.0034)} & \makecell{0.1618 \\ (0.0037)} & \makecell{0.1736 \\ (0.0042)} & \makecell{0.1853 \\ (0.0045)} & \makecell{0.1983 \\ (0.0050)} & \makecell{0.2106 \\ (0.0046)} & \makecell{0.2741 \\ (0.0044)} & \makecell{0.3366 \\ (0.0034)} \\
 & 5 & \makecell{0.1067 \\ (0.0035)} & \makecell{0.1139 \\ (0.0036)} & \makecell{0.1217 \\ (0.0037)} & \makecell{0.1296 \\ (0.0037)} & \makecell{0.1375 \\ (0.0035)} & \makecell{0.1455 \\ (0.0037)} & \makecell{0.1540 \\ (0.0044)} & \makecell{0.1625 \\ (0.0047)} & \makecell{0.1713 \\ (0.0044)} & \makecell{0.1799 \\ (0.0039)} & \makecell{0.2241 \\ (0.0047)} & \makecell{0.2671 \\ (0.0035)} \\
 
 \midrule
 \multirow{6}{*} \makecell{PrivUnit2+PA} 
 & 0 & \makecell{0.1203 \\ (0.0022)} & \makecell{0.1440 \\ (0.0030)} & \makecell{0.1710 \\ (0.0027)} & \makecell{0.2012 \\ (0.0028)} & \makecell{0.2341 \\ (0.0031)} & \makecell{0.2673 \\ (0.0040)} & \makecell{0.3027 \\ (0.0039)} & \makecell{0.3377 \\ (0.0041)} & \makecell{0.3733 \\ (0.0048)} & \makecell{0.4076 \\ (0.0052)} & \makecell{0.5477 \\ (0.0055)} & \makecell{0.6451 \\ (0.0057)} \\
 & 1 & \makecell{0.1202 \\ (0.0028)} & \makecell{0.1442 \\ (0.0024)} & \makecell{0.1707 \\ (0.0027)} & \makecell{0.2007 \\ (0.0029)} & \makecell{0.2332 \\ (0.0032)} & \makecell{0.2669 \\ (0.0032)} & \makecell{0.3014 \\ (0.0038)} & \makecell{0.3356 \\ (0.0052)} & \makecell{0.3701 \\ (0.0052)} & \makecell{0.4044 \\ (0.0047)} & \makecell{0.5455 \\ (0.0044)} & \makecell{0.6413 \\ (0.0044)} \\
 & 2 & \makecell{0.1206 \\ (0.0032)} & \makecell{0.1426 \\ (0.0036)} & \makecell{0.1684 \\ (0.0033)} & \makecell{0.1961 \\ (0.0036)} & \makecell{0.2263 \\ (0.0033)} & \makecell{0.2583 \\ (0.0039)} & \makecell{0.2918 \\ (0.0041)} & \makecell{0.3242 \\ (0.0050)} & \makecell{0.3572 \\ (0.0043)} & \makecell{0.3893 \\ (0.0044)} & \makecell{0.5224 \\ (0.0036)} & \makecell{0.6164 \\ (0.0047)} \\
 & 3 & \makecell{0.1185 \\ (0.0024)} & \makecell{0.1404 \\ (0.0029)} & \makecell{0.1651 \\ (0.0030)} & \makecell{0.1921 \\ (0.0039)} & \makecell{0.2204 \\ (0.0041)} & \makecell{0.2511 \\ (0.0046)} & \makecell{0.2820 \\ (0.0041)} & \makecell{0.3119 \\ (0.0047)} & \makecell{0.3427 \\ (0.0053)} & \makecell{0.3723 \\ (0.0054)} & \makecell{0.4984 \\ (0.0052)} & \makecell{0.5885 \\ (0.0044)} \\
 & 4 & \makecell{0.1175 \\ (0.0036)} & \makecell{0.1377 \\ (0.0036)} & \makecell{0.1608 \\ (0.0036)} & \makecell{0.1855 \\ (0.0039)} & \makecell{0.2115 \\ (0.0036)} & \makecell{0.2383 \\ (0.0036)} & \makecell{0.2663 \\ (0.0033)} & \makecell{0.2941 \\ (0.0042)} & \makecell{0.3212 \\ (0.0058)} & \makecell{0.3472 \\ (0.0057)} & \makecell{0.4629 \\ (0.0052)} & \makecell{0.5453 \\ (0.0043)} \\
 & 5 & \makecell{0.1149 \\ (0.0026)} & \makecell{0.1301 \\ (0.0030)} & \makecell{0.1469 \\ (0.0030)} & \makecell{0.1643 \\ (0.0029)} & \makecell{0.1822 \\ (0.0039)} & \makecell{0.2011 \\ (0.0035)} & \makecell{0.2194 \\ (0.0043)} & \makecell{0.2374 \\ (0.0044)} & \makecell{0.2552 \\ (0.0047)} & \makecell{0.2729 \\ (0.0048)} & \makecell{0.3483 \\ (0.0047)} & \makecell{0.4033 \\ (0.0057)} \\
 
 \midrule
 \multirow{6}{*} \makecell{PrivUnitG} 
 & 0 & \makecell{0.1100 \\ (0.0027)} & \makecell{0.1212 \\ (0.0030)} & \makecell{0.1335 \\ (0.0028)} & \makecell{0.1464 \\ (0.0035)} & \makecell{0.1600 \\ (0.0037)} & \makecell{0.1741 \\ (0.0040)} & \makecell{0.1887 \\ (0.0044)} & \makecell{0.2036 \\ (0.0039)} & \makecell{0.2187 \\ (0.0040)} & \makecell{0.2342 \\ (0.0043)} & \makecell{0.3103 \\ (0.0034)} & \makecell{0.3685 \\ (0.0035)} \\
 & 1 & \makecell{0.1095 \\ (0.0028)} & \makecell{0.1207 \\ (0.0028)} & \makecell{0.1331 \\ (0.0030)} & \makecell{0.1460 \\ (0.0033)} & \makecell{0.1591 \\ (0.0038)} & \makecell{0.1728 \\ (0.0039)} & \makecell{0.1875 \\ (0.0044)} & \makecell{0.2019 \\ (0.0042)} & \makecell{0.2170 \\ (0.0043)} & \makecell{0.2326 \\ (0.0041)} & \makecell{0.3077 \\ (0.0035)} & \makecell{0.3660 \\ (0.0037)} \\
 & 2 & \makecell{0.1088 \\ (0.0030)} & \makecell{0.1191 \\ (0.0030)} & \makecell{0.1313 \\ (0.0031)} & \makecell{0.1438 \\ (0.0038)} & \makecell{0.1565 \\ (0.0039)} & \makecell{0.1698 \\ (0.0036)} & \makecell{0.1834 \\ (0.0033)} & \makecell{0.1972 \\ (0.0033)} & \makecell{0.2116 \\ (0.0037)} & \makecell{0.2263 \\ (0.0036)} & \makecell{0.2972 \\ (0.0032)} & \makecell{0.3530 \\ (0.0034)} \\
 & 3 & \makecell{0.1081 \\ (0.0026)} & \makecell{0.1180 \\ (0.0030)} & \makecell{0.1293 \\ (0.0030)} & \makecell{0.1410 \\ (0.0034)} & \makecell{0.1537 \\ (0.0042)} & \makecell{0.1659 \\ (0.0038)} & \makecell{0.1785 \\ (0.0039)} & \makecell{0.1918 \\ (0.0038)} & \makecell{0.2051 \\ (0.0041)} & \makecell{0.2184 \\ (0.0045)} & \makecell{0.2851 \\ (0.0039)} & \makecell{0.3363 \\ (0.0035)} \\
 & 4 & \makecell{0.1078 \\ (0.0029)} & \makecell{0.1170 \\ (0.0031)} & \makecell{0.1274 \\ (0.0030)} & \makecell{0.1379 \\ (0.0035)} & \makecell{0.1488 \\ (0.0036)} & \makecell{0.1600 \\ (0.0038)} & \makecell{0.1718 \\ (0.0038)} & \makecell{0.1835 \\ (0.0038)} & \makecell{0.1955 \\ (0.0039)} & \makecell{0.2075 \\ (0.0041)} & \makecell{0.2672 \\ (0.0039)} & \makecell{0.3137 \\ (0.0040)} \\
 & 5 & \makecell{0.1054 \\ (0.0027)} & \makecell{0.1119 \\ (0.0029)} & \makecell{0.1196 \\ (0.0028)} & \makecell{0.1275 \\ (0.0033)} & \makecell{0.1356 \\ (0.0036)} & \makecell{0.1439 \\ (0.0034)} & \makecell{0.1522 \\ (0.0035)} & \makecell{0.1609 \\ (0.0033)} & \makecell{0.1696 \\ (0.0034)} & \makecell{0.1779 \\ (0.0036)} & \makecell{0.2200 \\ (0.0039)} & \makecell{0.2518 \\ (0.0038)} \\
 
 \midrule
 \multirow{6}{*} \makecell{PrivUnitG+PA} 
 & 0 & \makecell{0.1213 \\ (0.0041)} & \makecell{0.1456 \\ (0.0049)} & \makecell{0.1726 \\ (0.0050)} & \makecell{0.2020 \\ (0.0047)} & \makecell{0.2349 \\ (0.0044)} & \makecell{0.2684 \\ (0.0043)} & \makecell{0.3029 \\ (0.0041)} & \makecell{0.3367 \\ (0.0043)} & \makecell{0.3700 \\ (0.0044)} & \makecell{0.4016 \\ (0.0046)} & \makecell{0.5364 \\ (0.0044)} & \makecell{0.6254 \\ (0.0045)} \\
 & 1 & \makecell{0.1210 \\ (0.0026)} & \makecell{0.1445 \\ (0.0030)} & \makecell{0.1718 \\ (0.0030)} & \makecell{0.2011 \\ (0.0036)} & \makecell{0.2330 \\ (0.0039)} & \makecell{0.2664 \\ (0.0041)} & \makecell{0.3004 \\ (0.0048)} & \makecell{0.3340 \\ (0.0053)} & \makecell{0.3671 \\ (0.0050)} & \makecell{0.3986 \\ (0.0051)} & \makecell{0.5328 \\ (0.0053)} & \makecell{0.6207 \\ (0.0053)} \\
 & 2 & \makecell{0.1205 \\ (0.0021)} & \makecell{0.1428 \\ (0.0023)} & \makecell{0.1690 \\ (0.0028)} & \makecell{0.1971 \\ (0.0035)} & \makecell{0.2269 \\ (0.0036)} & \makecell{0.2583 \\ (0.0033)} & \makecell{0.2909 \\ (0.0036)} & \makecell{0.3234 \\ (0.0039)} & \makecell{0.3549 \\ (0.0041)} & \makecell{0.3851 \\ (0.0040)} & \makecell{0.5125 \\ (0.0043)} & \makecell{0.5968 \\ (0.0053)} \\
 & 3 & \makecell{0.1186 \\ (0.0026)} & \makecell{0.1405 \\ (0.0025)} & \makecell{0.1650 \\ (0.0026)} & \makecell{0.1917 \\ (0.0029)} & \makecell{0.2201 \\ (0.0033)} & \makecell{0.2495 \\ (0.0032)} & \makecell{0.2796 \\ (0.0044)} & \makecell{0.3093 \\ (0.0050)} & \makecell{0.3392 \\ (0.0055)} & \makecell{0.3679 \\ (0.0056)} & \makecell{0.4879 \\ (0.0055)} & \makecell{0.5678 \\ (0.0050)} \\
 & 4 & \makecell{0.1180 \\ (0.0036)} & \makecell{0.1384 \\ (0.0036)} & \makecell{0.1613 \\ (0.0042)} & \makecell{0.1859 \\ (0.0043)} & \makecell{0.2121 \\ (0.0041)} & \makecell{0.2388 \\ (0.0045)} & \makecell{0.2663 \\ (0.0051)} & \makecell{0.2932 \\ (0.0056)} & \makecell{0.3198 \\ (0.0057)} & \makecell{0.3451 \\ (0.0054)} & \makecell{0.4528 \\ (0.0043)} & \makecell{0.5262 \\ (0.0037)} \\
 & 5 & \makecell{0.1136 \\ (0.0035)} & \makecell{0.1284 \\ (0.0034)} & \makecell{0.1448 \\ (0.0033)} & \makecell{0.1621 \\ (0.0038)} & \makecell{0.1807 \\ (0.0044)} & \makecell{0.1991 \\ (0.0038)} & \makecell{0.2178 \\ (0.0038)} & \makecell{0.2353 \\ (0.0041)} & \makecell{0.2535 \\ (0.0041)} & \makecell{0.2701 \\ (0.0041)} & \makecell{0.3417 \\ (0.0049)} & \makecell{0.3914 \\ (0.0045)} \\
 
 \midrule
 \multirow{6}{*} \makecell{CW+PA w/o bounding} 
 & 0 & \makecell{0.1039 \\ (0.0024)} & \makecell{0.1077 \\ (0.0025)} & \makecell{0.1120 \\ (0.0026)} & \makecell{0.1161 \\ (0.0026)} & \makecell{0.1205 \\ (0.0026)} & \makecell{0.1251 \\ (0.0027)} & \makecell{0.1298 \\ (0.0026)} & \makecell{0.1346 \\ (0.0025)} & \makecell{0.1397 \\ (0.0026)} & \makecell{0.1446 \\ (0.0028)} & \makecell{0.1716 \\ (0.0027)} & \makecell{0.2008 \\ (0.0026)} \\
 & 1 & \makecell{0.1025 \\ (0.0029)} & \makecell{0.1064 \\ (0.0029)} & \makecell{0.1104 \\ (0.0029)} & \makecell{0.1145 \\ (0.0027)} & \makecell{0.1187 \\ (0.0028)} & \makecell{0.1232 \\ (0.0025)} & \makecell{0.1278 \\ (0.0026)} & \makecell{0.1325 \\ (0.0025)} & \makecell{0.1372 \\ (0.0028)} & \makecell{0.1420 \\ (0.0029)} & \makecell{0.1684 \\ (0.0034)} & \makecell{0.1973 \\ (0.0044)} \\
 & 2 & \makecell{0.1035 \\ (0.0023)} & \makecell{0.1067 \\ (0.0024)} & \makecell{0.1103 \\ (0.0023)} & \makecell{0.1140 \\ (0.0024)} & \makecell{0.1178 \\ (0.0023)} & \makecell{0.1220 \\ (0.0025)} & \makecell{0.1262 \\ (0.0026)} & \makecell{0.1303 \\ (0.0027)} & \makecell{0.1346 \\ (0.0028)} & \makecell{0.1389 \\ (0.0028)} & \makecell{0.1620 \\ (0.0031)} & \makecell{0.1881 \\ (0.0032)} \\
 & 3 & \makecell{0.1035 \\ (0.0030)} & \makecell{0.1066 \\ (0.0030)} & \makecell{0.1101 \\ (0.0029)} & \makecell{0.1135 \\ (0.0032)} & \makecell{0.1168 \\ (0.0032)} & \makecell{0.1204 \\ (0.0032)} & \makecell{0.1241 \\ (0.0035)} & \makecell{0.1279 \\ (0.0036)} & \makecell{0.1315 \\ (0.0036)} & \makecell{0.1356 \\ (0.0037)} & \makecell{0.1571 \\ (0.0039)} & \makecell{0.1805 \\ (0.0039)} \\
 & 4 & \makecell{0.1032 \\ (0.0033)} & \makecell{0.1059 \\ (0.0035)} & \makecell{0.1087 \\ (0.0036)} & \makecell{0.1117 \\ (0.0035)} & \makecell{0.1147 \\ (0.0033)} & \makecell{0.1178 \\ (0.0032)} & \makecell{0.1210 \\ (0.0032)} & \makecell{0.1244 \\ (0.0034)} & \makecell{0.1277 \\ (0.0036)} & \makecell{0.1309 \\ (0.0037)} & \makecell{0.1493 \\ (0.0035)} & \makecell{0.1699 \\ (0.0037)} \\
 & 5 & \makecell{0.1020 \\ (0.0029)} & \makecell{0.1041 \\ (0.0027)} & \makecell{0.1063 \\ (0.0030)} & \makecell{0.1084 \\ (0.0032)} & \makecell{0.1108 \\ (0.0031)} & \makecell{0.1129 \\ (0.0032)} & \makecell{0.1154 \\ (0.0032)} & \makecell{0.1177 \\ (0.0031)} & \makecell{0.1202 \\ (0.0035)} & \makecell{0.1225 \\ (0.0036)} & \makecell{0.1350 \\ (0.0040)} & \makecell{0.1484 \\ (0.0042)} \\
 \bottomrule
\end{tabular}
}\label{tb:fog}
\end{table}


\begin{table}[h]
\caption{Main results for \textbf{Nonlinear MLP Classifier with ReLU} $(m=64)$ on \textbf{CIFAR-10} and \textbf{CIFAR-10-C (Defocus Blur)} under $\epsilon$-LDP mechanisms. We evaluate utility by measuring the Accuracy. Higher accuracy indicates higher utility.}
\centering
\resizebox{\textwidth}{!}{
\begin{tabular}{ll cccccccccccc}
\toprule
& \multicolumn{12}{c}{Privacy Budget ($\epsilon$)} \\
 \cmidrule(l){3-14}
 \textbf{Mechanism} & $S$ & 0.5 & 1.0 & 1.5 & 2.0 & 2.5 & 3.0 & 3.5 & 4.0 & 4.5 & 5.0 & 7.5 & 10.0 \\
 \midrule
 \multirow{6}{*} \makecell {No randomization} 
 & 0 & 0.8937 & 0.8937 & 0.8937 & 0.8937 & 0.8937 & 0.8937 & 0.8937 & 0.8937 & 0.8937 & 0.8937 & 0.8937 & 0.8937 \\
 & 1 & 0.8905 & 0.8905 & 0.8905 & 0.8905 & 0.8905 & 0.8905 & 0.8905 & 0.8905 & 0.8905 & 0.8905 & 0.8905 & 0.8905 \\
 & 2 & 0.8729 & 0.8729 & 0.8729 & 0.8729 & 0.8729 & 0.8729 & 0.8729 & 0.8729 & 0.8729 & 0.8729 & 0.8729 & 0.8729 \\
 & 3 & 0.8057 & 0.8057 & 0.8057 & 0.8057 & 0.8057 & 0.8057 & 0.8057 & 0.8057 & 0.8057 & 0.8057 & 0.8057 & 0.8057 \\
 & 4 & 0.6922 & 0.6922 & 0.6922 & 0.6922 & 0.6922 & 0.6922 & 0.6922 & 0.6922 & 0.6922 & 0.6922 & 0.6922 & 0.6922 \\
 & 5 & 0.4898 & 0.4898 & 0.4898 & 0.4898 & 0.4898 & 0.4898 & 0.4898 & 0.4898 & 0.4898 & 0.4898 & 0.4898 & 0.4898 \\
 
 \midrule
 \multirow{6}{*} \makecell{Laplace ($\ell_1$)}
 & 0 & \makecell{0.1002 \\ (0.0029)} & \makecell{0.1017 \\ (0.0029)} & \makecell{0.1031 \\ (0.0028)} & \makecell{0.1046 \\ (0.0028)} & \makecell{0.1059 \\ (0.0027)} & \makecell{0.1074 \\ (0.0026)} & \makecell{0.1088 \\ (0.0025)} & \makecell{0.1103 \\ (0.0027)} & \makecell{0.1119 \\ (0.0026)} & \makecell{0.1133 \\ (0.0025)} & \makecell{0.1212 \\ (0.0026)} & \makecell{0.1301 \\ (0.0024)} \\
 & 1 & \makecell{0.1002 \\ (0.0029)} & \makecell{0.1017 \\ (0.0029)} & \makecell{0.1031 \\ (0.0029)} & \makecell{0.1045 \\ (0.0028)} & \makecell{0.1059 \\ (0.0027)} & \makecell{0.1073 \\ (0.0026)} & \makecell{0.1088 \\ (0.0026)} & \makecell{0.1103 \\ (0.0025)} & \makecell{0.1118 \\ (0.0026)} & \makecell{0.1133 \\ (0.0026)} & \makecell{0.1210 \\ (0.0027)} & \makecell{0.1297 \\ (0.0026)} \\
 & 2 & \makecell{0.1002 \\ (0.0029)} & \makecell{0.1015 \\ (0.0029)} & \makecell{0.1028 \\ (0.0030)} & \makecell{0.1042 \\ (0.0029)} & \makecell{0.1055 \\ (0.0027)} & \makecell{0.1068 \\ (0.0026)} & \makecell{0.1081 \\ (0.0026)} & \makecell{0.1095 \\ (0.0026)} & \makecell{0.1110 \\ (0.0026)} & \makecell{0.1124 \\ (0.0027)} & \makecell{0.1196 \\ (0.0027)} & \makecell{0.1276 \\ (0.0025)} \\
 & 3 & \makecell{0.1000 \\ (0.0029)} & \makecell{0.1012 \\ (0.0029)} & \makecell{0.1022 \\ (0.0029)} & \makecell{0.1034 \\ (0.0028)} & \makecell{0.1046 \\ (0.0028)} & \makecell{0.1058 \\ (0.0026)} & \makecell{0.1070 \\ (0.0026)} & \makecell{0.1082 \\ (0.0026)} & \makecell{0.1095 \\ (0.0027)} & \makecell{0.1106 \\ (0.0027)} & \makecell{0.1167 \\ (0.0028)} & \makecell{0.1237 \\ (0.0027)} \\
 & 4 & \makecell{0.0998 \\ (0.0028)} & \makecell{0.1007 \\ (0.0029)} & \makecell{0.1017 \\ (0.0030)} & \makecell{0.1025 \\ (0.0028)} & \makecell{0.1036 \\ (0.0027)} & \makecell{0.1045 \\ (0.0027)} & \makecell{0.1054 \\ (0.0027)} & \makecell{0.1065 \\ (0.0026)} & \makecell{0.1075 \\ (0.0026)} & \makecell{0.1085 \\ (0.0027)} & \makecell{0.1133 \\ (0.0028)} & \makecell{0.1190 \\ (0.0029)} \\
 & 5 & \makecell{0.0995 \\ (0.0029)} & \makecell{0.1001 \\ (0.0029)} & \makecell{0.1007 \\ (0.0030)} & \makecell{0.1013 \\ (0.0029)} & \makecell{0.1018 \\ (0.0029)} & \makecell{0.1025 \\ (0.0029)} & \makecell{0.1032 \\ (0.0028)} & \makecell{0.1039 \\ (0.0028)} & \makecell{0.1046 \\ (0.0026)} & \makecell{0.1054 \\ (0.0027)} & \makecell{0.1087 \\ (0.0028)} & \makecell{0.1123 \\ (0.0029)} \\
 
 \midrule
 \multirow{6}{*} \makecell{Laplace+PA} 
 & 0 & \makecell{0.1062 \\ (0.0024)} & \makecell{0.1122 \\ (0.0022)} & \makecell{0.1188 \\ (0.0023)} & \makecell{0.1258 \\ (0.0025)} & \makecell{0.1329 \\ (0.0023)} & \makecell{0.1404 \\ (0.0023)} & \makecell{0.1480 \\ (0.0024)} & \makecell{0.1563 \\ (0.0023)} & \makecell{0.1650 \\ (0.0024)} & \makecell{0.1737 \\ (0.0024)} & \makecell{0.2204 \\ (0.0038)} & \makecell{0.2723 \\ (0.0046)} \\
 & 1 & \makecell{0.1064 \\ (0.0029)} & \makecell{0.1126 \\ (0.0030)} & \makecell{0.1190 \\ (0.0029)} & \makecell{0.1256 \\ (0.0029)} & \makecell{0.1324 \\ (0.0032)} & \makecell{0.1397 \\ (0.0031)} & \makecell{0.1476 \\ (0.0032)} & \makecell{0.1558 \\ (0.0034)} & \makecell{0.1639 \\ (0.0036)} & \makecell{0.1724 \\ (0.0035)} & \makecell{0.2180 \\ (0.0038)} & \makecell{0.2692 \\ (0.0033)} \\
 & 2 & \makecell{0.1065 \\ (0.0031)} & \makecell{0.1124 \\ (0.0036)} & \makecell{0.1185 \\ (0.0037)} & \makecell{0.1246 \\ (0.0038)} & \makecell{0.1312 \\ (0.0041)} & \makecell{0.1377 \\ (0.0044)} & \makecell{0.1445 \\ (0.0046)} & \makecell{0.1517 \\ (0.0046)} & \makecell{0.1594 \\ (0.0047)} & \makecell{0.1674 \\ (0.0047)} & \makecell{0.2098 \\ (0.0052)} & \makecell{0.2583 \\ (0.0052)} \\
 & 3 & \makecell{0.1059 \\ (0.0030)} & \makecell{0.1108 \\ (0.0031)} & \makecell{0.1160 \\ (0.0030)} & \makecell{0.1215 \\ (0.0031)} & \makecell{0.1273 \\ (0.0030)} & \makecell{0.1333 \\ (0.0029)} & \makecell{0.1397 \\ (0.0029)} & \makecell{0.1460 \\ (0.0033)} & \makecell{0.1527 \\ (0.0034)} & \makecell{0.1595 \\ (0.0032)} & \makecell{0.1963 \\ (0.0046)} & \makecell{0.2358 \\ (0.0054)} \\
 & 4 & \makecell{0.1036 \\ (0.0040)} & \makecell{0.1078 \\ (0.0041)} & \makecell{0.1121 \\ (0.0045)} & \makecell{0.1166 \\ (0.0047)} & \makecell{0.1211 \\ (0.0051)} & \makecell{0.1258 \\ (0.0050)} & \makecell{0.1306 \\ (0.0048)} & \makecell{0.1358 \\ (0.0050)} & \makecell{0.1412 \\ (0.0048)} & \makecell{0.1465 \\ (0.0049)} & \makecell{0.1745 \\ (0.0048)} & \makecell{0.2052 \\ (0.0043)} \\
 & 5 & \makecell{0.1025 \\ (0.0033)} & \makecell{0.1052 \\ (0.0031)} & \makecell{0.1083 \\ (0.0033)} & \makecell{0.1113 \\ (0.0033)} & \makecell{0.1144 \\ (0.0032)} & \makecell{0.1176 \\ (0.0033)} & \makecell{0.1210 \\ (0.0033)} & \makecell{0.1244 \\ (0.0034)} & \makecell{0.1277 \\ (0.0033)} & \makecell{0.1312 \\ (0.0036)} & \makecell{0.1496 \\ (0.0037)} & \makecell{0.1690 \\ (0.0037)} \\
 
 \midrule
 \multirow{6}{*} \makecell{PrivUnit2} 
 & 0 & \makecell{0.1107 \\ (0.0032)} & \makecell{0.1221 \\ (0.0032)} & \makecell{0.1348 \\ (0.0032)} & \makecell{0.1481 \\ (0.0031)} & \makecell{0.1620 \\ (0.0028)} & \makecell{0.1762 \\ (0.0029)} & \makecell{0.1913 \\ (0.0038)} & \makecell{0.2063 \\ (0.0044)} & \makecell{0.2229 \\ (0.0044)} & \makecell{0.2383 \\ (0.0040)} & \makecell{0.3185 \\ (0.0042)} & \makecell{0.3978 \\ (0.0035)} \\
 & 1 & \makecell{0.1105 \\ (0.0031)} & \makecell{0.1220 \\ (0.0030)} & \makecell{0.1347 \\ (0.0031)} & \makecell{0.1479 \\ (0.0034)} & \makecell{0.1617 \\ (0.0028)} & \makecell{0.1759 \\ (0.0029)} & \makecell{0.1910 \\ (0.0035)} & \makecell{0.2059 \\ (0.0040)} & \makecell{0.2224 \\ (0.0041)} & \makecell{0.2375 \\ (0.0036)} & \makecell{0.3176 \\ (0.0037)} & \makecell{0.3966 \\ (0.0039)} \\
 & 2 & \makecell{0.1098 \\ (0.0033)} & \makecell{0.1209 \\ (0.0031)} & \makecell{0.1331 \\ (0.0031)} & \makecell{0.1463 \\ (0.0036)} & \makecell{0.1596 \\ (0.0033)} & \makecell{0.1733 \\ (0.0035)} & \makecell{0.1877 \\ (0.0042)} & \makecell{0.2020 \\ (0.0044)} & \makecell{0.2177 \\ (0.0043)} & \makecell{0.2322 \\ (0.0037)} & \makecell{0.3094 \\ (0.0045)} & \makecell{0.3841 \\ (0.0036)} \\
 & 3 & \makecell{0.1094 \\ (0.0038)} & \makecell{0.1194 \\ (0.0035)} & \makecell{0.1304 \\ (0.0036)} & \makecell{0.1422 \\ (0.0035)} & \makecell{0.1541 \\ (0.0035)} & \makecell{0.1662 \\ (0.0039)} & \makecell{0.1789 \\ (0.0044)} & \makecell{0.1917 \\ (0.0045)} & \makecell{0.2052 \\ (0.0044)} & \makecell{0.2177 \\ (0.0041)} & \makecell{0.2849 \\ (0.0043)} & \makecell{0.3510 \\ (0.0038)} \\
 & 4 & \makecell{0.1082 \\ (0.0034)} & \makecell{0.1166 \\ (0.0031)} & \makecell{0.1259 \\ (0.0033)} & \makecell{0.1357 \\ (0.0031)} & \makecell{0.1456 \\ (0.0031)} & \makecell{0.1557 \\ (0.0032)} & \makecell{0.1665 \\ (0.0038)} & \makecell{0.1769 \\ (0.0039)} & \makecell{0.1881 \\ (0.0044)} & \makecell{0.1984 \\ (0.0038)} & \makecell{0.2528 \\ (0.0044)} & \makecell{0.3062 \\ (0.0042)} \\
 & 5 & \makecell{0.1057 \\ (0.0038)} & \makecell{0.1118 \\ (0.0039)} & \makecell{0.1183 \\ (0.0038)} & \makecell{0.1252 \\ (0.0036)} & \makecell{0.1325 \\ (0.0037)} & \makecell{0.1392 \\ (0.0038)} & \makecell{0.1463 \\ (0.0035)} & \makecell{0.1532 \\ (0.0038)} & \makecell{0.1606 \\ (0.0037)} & \makecell{0.1678 \\ (0.0034)} & \makecell{0.2023 \\ (0.0040)} & \makecell{0.2362 \\ (0.0034)} \\
 
 \midrule
 \multirow{6}{*} \makecell{PrivUnit2+PA} 
 & 0 & \makecell{0.1203 \\ (0.0022)} & \makecell{0.1440 \\ (0.0030)} & \makecell{0.1710 \\ (0.0027)} & \makecell{0.2012 \\ (0.0028)} & \makecell{0.2341 \\ (0.0031)} & \makecell{0.2673 \\ (0.0040)} & \makecell{0.3027 \\ (0.0039)} & \makecell{0.3377 \\ (0.0041)} & \makecell{0.3733 \\ (0.0048)} & \makecell{0.4076 \\ (0.0052)} & \makecell{0.5477 \\ (0.0055)} & \makecell{0.6451 \\ (0.0057)} \\
 & 1 & \makecell{0.1213 \\ (0.0023)} & \makecell{0.1456 \\ (0.0023)} & \makecell{0.1727 \\ (0.0024)} & \makecell{0.2024 \\ (0.0030)} & \makecell{0.2343 \\ (0.0033)} & \makecell{0.2685 \\ (0.0030)} & \makecell{0.3026 \\ (0.0030)} & \makecell{0.3373 \\ (0.0038)} & \makecell{0.3719 \\ (0.0051)} & \makecell{0.4054 \\ (0.0054)} & \makecell{0.5462 \\ (0.0049)} & \makecell{0.6428 \\ (0.0055)} \\
 & 2 & \makecell{0.1200 \\ (0.0022)} & \makecell{0.1429 \\ (0.0025)} & \makecell{0.1692 \\ (0.0030)} & \makecell{0.1981 \\ (0.0033)} & \makecell{0.2280 \\ (0.0031)} & \makecell{0.2598 \\ (0.0044)} & \makecell{0.2927 \\ (0.0044)} & \makecell{0.3254 \\ (0.0047)} & \makecell{0.3586 \\ (0.0052)} & \makecell{0.3910 \\ (0.0043)} & \makecell{0.5245 \\ (0.0054)} & \makecell{0.6183 \\ (0.0048)} \\
 & 3 & \makecell{0.1183 \\ (0.0030)} & \makecell{0.1396 \\ (0.0038)} & \makecell{0.1636 \\ (0.0038)} & \makecell{0.1883 \\ (0.0033)} & \makecell{0.2150 \\ (0.0029)} & \makecell{0.2433 \\ (0.0027)} & \makecell{0.2718 \\ (0.0031)} & \makecell{0.3007 \\ (0.0038)} & \makecell{0.3288 \\ (0.0040)} & \makecell{0.3571 \\ (0.0047)} & \makecell{0.4757 \\ (0.0045)} & \makecell{0.5595 \\ (0.0062)} \\
 & 4 & \makecell{0.1172 \\ (0.0029)} & \makecell{0.1353 \\ (0.0027)} & \makecell{0.1554 \\ (0.0030)} & \makecell{0.1770 \\ (0.0029)} & \makecell{0.1992 \\ (0.0029)} & \makecell{0.2220 \\ (0.0036)} & \makecell{0.2450 \\ (0.0037)} & \makecell{0.2670 \\ (0.0041)} & \makecell{0.2893 \\ (0.0041)} & \makecell{0.3117 \\ (0.0050)} & \makecell{0.4057 \\ (0.0049)} & \makecell{0.4735 \\ (0.0044)} \\
 & 5 & \makecell{0.1121 \\ (0.0032)} & \makecell{0.1253 \\ (0.0039)} & \makecell{0.1398 \\ (0.0044)} & \makecell{0.1543 \\ (0.0045)} & \makecell{0.1690 \\ (0.0051)} & \makecell{0.1841 \\ (0.0054)} & \makecell{0.1993 \\ (0.0060)} & \makecell{0.2141 \\ (0.0053)} & \makecell{0.2282 \\ (0.0054)} & \makecell{0.2416 \\ (0.0050)} & \makecell{0.2994 \\ (0.0041)} & \makecell{0.3423 \\ (0.0038)} \\
 
 \midrule
 \multirow{6}{*} \makecell{PrivUnitG} 
 & 0 & \makecell{0.1100 \\ (0.0027)} & \makecell{0.1212 \\ (0.0030)} & \makecell{0.1335 \\ (0.0028)} & \makecell{0.1464 \\ (0.0035)} & \makecell{0.1600 \\ (0.0037)} & \makecell{0.1741 \\ (0.0040)} & \makecell{0.1887 \\ (0.0044)} & \makecell{0.2036 \\ (0.0039)} & \makecell{0.2187 \\ (0.0040)} & \makecell{0.2342 \\ (0.0043)} & \makecell{0.3103 \\ (0.0034)} & \makecell{0.3685 \\ (0.0035)} \\
 & 1 & \makecell{0.1100 \\ (0.0027)} & \makecell{0.1209 \\ (0.0028)} & \makecell{0.1334 \\ (0.0030)} & \makecell{0.1467 \\ (0.0034)} & \makecell{0.1597 \\ (0.0037)} & \makecell{0.1736 \\ (0.0037)} & \makecell{0.1886 \\ (0.0044)} & \makecell{0.2033 \\ (0.0043)} & \makecell{0.2183 \\ (0.0044)} & \makecell{0.2336 \\ (0.0042)} & \makecell{0.3101 \\ (0.0041)} & \makecell{0.3687 \\ (0.0032)} \\
 & 2 & \makecell{0.1092 \\ (0.0029)} & \makecell{0.1201 \\ (0.0028)} & \makecell{0.1323 \\ (0.0027)} & \makecell{0.1449 \\ (0.0031)} & \makecell{0.1577 \\ (0.0040)} & \makecell{0.1714 \\ (0.0037)} & \makecell{0.1855 \\ (0.0039)} & \makecell{0.1995 \\ (0.0041)} & \makecell{0.2144 \\ (0.0042)} & \makecell{0.2286 \\ (0.0043)} & \makecell{0.3018 \\ (0.0042)} & \makecell{0.3578 \\ (0.0034)} \\
 & 3 & \makecell{0.1083 \\ (0.0025)} & \makecell{0.1181 \\ (0.0029)} & \makecell{0.1287 \\ (0.0029)} & \makecell{0.1400 \\ (0.0033)} & \makecell{0.1517 \\ (0.0038)} & \makecell{0.1641 \\ (0.0034)} & \makecell{0.1764 \\ (0.0037)} & \makecell{0.1891 \\ (0.0041)} & \makecell{0.2024 \\ (0.0042)} & \makecell{0.2156 \\ (0.0044)} & \makecell{0.2790 \\ (0.0041)} & \makecell{0.3284 \\ (0.0040)} \\
 & 4 & \makecell{0.1067 \\ (0.0025)} & \makecell{0.1152 \\ (0.0030)} & \makecell{0.1247 \\ (0.0029)} & \makecell{0.1343 \\ (0.0032)} & \makecell{0.1440 \\ (0.0034)} & \makecell{0.1541 \\ (0.0033)} & \makecell{0.1648 \\ (0.0037)} & \makecell{0.1754 \\ (0.0040)} & \makecell{0.1861 \\ (0.0042)} & \makecell{0.1967 \\ (0.0038)} & \makecell{0.2489 \\ (0.0040)} & \makecell{0.2887 \\ (0.0044)} \\
 & 5 & \makecell{0.1048 \\ (0.0033)} & \makecell{0.1105 \\ (0.0034)} & \makecell{0.1168 \\ (0.0033)} & \makecell{0.1236 \\ (0.0036)} & \makecell{0.1304 \\ (0.0037)} & \makecell{0.1374 \\ (0.0037)} & \makecell{0.1445 \\ (0.0040)} & \makecell{0.1517 \\ (0.0038)} & \makecell{0.1586 \\ (0.0038)} & \makecell{0.1660 \\ (0.0038)} & \makecell{0.2003 \\ (0.0036)} & \makecell{0.2253 \\ (0.0042)} \\
 
 \midrule
 \multirow{6}{*} \makecell{PrivUnitG+PA} 
 & 0 & \makecell{0.1213 \\ (0.0041)} & \makecell{0.1456 \\ (0.0049)} & \makecell{0.1726 \\ (0.0050)} & \makecell{0.2020 \\ (0.0047)} & \makecell{0.2349 \\ (0.0044)} & \makecell{0.2684 \\ (0.0043)} & \makecell{0.3029 \\ (0.0041)} & \makecell{0.3367 \\ (0.0043)} & \makecell{0.3700 \\ (0.0044)} & \makecell{0.4016 \\ (0.0046)} & \makecell{0.5364 \\ (0.0044)} & \makecell{0.6254 \\ (0.0045)} \\
 & 1 & \makecell{0.1201 \\ (0.0033)} & \makecell{0.1438 \\ (0.0034)} & \makecell{0.1713 \\ (0.0033)} & \makecell{0.2007 \\ (0.0035)} & \makecell{0.2332 \\ (0.0038)} & \makecell{0.2671 \\ (0.0041)} & \makecell{0.3011 \\ (0.0042)} & \makecell{0.3349 \\ (0.0047)} & \makecell{0.3680 \\ (0.0042)} & \makecell{0.3998 \\ (0.0053)} & \makecell{0.5342 \\ (0.0055)} & \makecell{0.6230 \\ (0.0042)} \\
 & 2 & \makecell{0.1195 \\ (0.0030)} & \makecell{0.1427 \\ (0.0033)} & \makecell{0.1685 \\ (0.0038)} & \makecell{0.1971 \\ (0.0039)} & \makecell{0.2275 \\ (0.0045)} & \makecell{0.2591 \\ (0.0049)} & \makecell{0.2908 \\ (0.0052)} & \makecell{0.3231 \\ (0.0052)} & \makecell{0.3540 \\ (0.0065)} & \makecell{0.3835 \\ (0.0066)} & \makecell{0.5125 \\ (0.0060)} & \makecell{0.5982 \\ (0.0063)} \\
 & 3 & \makecell{0.1194 \\ (0.0035)} & \makecell{0.1400 \\ (0.0038)} & \makecell{0.1639 \\ (0.0035)} & \makecell{0.1891 \\ (0.0034)} & \makecell{0.2158 \\ (0.0034)} & \makecell{0.2436 \\ (0.0035)} & \makecell{0.2722 \\ (0.0040)} & \makecell{0.2999 \\ (0.0048)} & \makecell{0.3266 \\ (0.0050)} & \makecell{0.3525 \\ (0.0051)} & \makecell{0.4654 \\ (0.0054)} & \makecell{0.5407 \\ (0.0055)} \\
 & 4 & \makecell{0.1168 \\ (0.0020)} & \makecell{0.1353 \\ (0.0023)} & \makecell{0.1547 \\ (0.0028)} & \makecell{0.1757 \\ (0.0030)} & \makecell{0.1974 \\ (0.0035)} & \makecell{0.2197 \\ (0.0038)} & \makecell{0.2428 \\ (0.0035)} & \makecell{0.2662 \\ (0.0032)} & \makecell{0.2882 \\ (0.0036)} & \makecell{0.3093 \\ (0.0033)} & \makecell{0.3994 \\ (0.0031)} & \makecell{0.4603 \\ (0.0038)} \\
 & 5 & \makecell{0.1125 \\ (0.0032)} & \makecell{0.1256 \\ (0.0037)} & \makecell{0.1399 \\ (0.0034)} & \makecell{0.1546 \\ (0.0036)} & \makecell{0.1695 \\ (0.0040)} & \makecell{0.1842 \\ (0.0036)} & \makecell{0.1984 \\ (0.0039)} & \makecell{0.2121 \\ (0.0046)} & \makecell{0.2261 \\ (0.0051)} & \makecell{0.2393 \\ (0.0045)} & \makecell{0.2955 \\ (0.0046)} & \makecell{0.3348 \\ (0.0047)} \\
 
 \midrule
 \multirow{6}{*} \makecell{CW+PA w/o bounding} 
 & 0 & \makecell{0.1039 \\ (0.0024)} & \makecell{0.1077 \\ (0.0025)} & \makecell{0.1120 \\ (0.0026)} & \makecell{0.1161 \\ (0.0026)} & \makecell{0.1205 \\ (0.0026)} & \makecell{0.1251 \\ (0.0027)} & \makecell{0.1298 \\ (0.0026)} & \makecell{0.1346 \\ (0.0025)} & \makecell{0.1397 \\ (0.0026)} & \makecell{0.1446 \\ (0.0028)} & \makecell{0.1716 \\ (0.0027)} & \makecell{0.2008 \\ (0.0026)} \\
 & 1 & \makecell{0.1035 \\ (0.0030)} & \makecell{0.1076 \\ (0.0027)} & \makecell{0.1116 \\ (0.0028)} & \makecell{0.1161 \\ (0.0026)} & \makecell{0.1204 \\ (0.0027)} & \makecell{0.1250 \\ (0.0031)} & \makecell{0.1297 \\ (0.0033)} & \makecell{0.1346 \\ (0.0034)} & \makecell{0.1392 \\ (0.0032)} & \makecell{0.1442 \\ (0.0032)} & \makecell{0.1711 \\ (0.0041)} & \makecell{0.2011 \\ (0.0040)} \\
 & 2 & \makecell{0.1043 \\ (0.0031)} & \makecell{0.1079 \\ (0.0029)} & \makecell{0.1114 \\ (0.0032)} & \makecell{0.1154 \\ (0.0034)} & \makecell{0.1196 \\ (0.0035)} & \makecell{0.1235 \\ (0.0034)} & \makecell{0.1274 \\ (0.0033)} & \makecell{0.1318 \\ (0.0035)} & \makecell{0.1364 \\ (0.0038)} & \makecell{0.1410 \\ (0.0038)} & \makecell{0.1657 \\ (0.0039)} & \makecell{0.1926 \\ (0.0041)} \\
 & 3 & \makecell{0.1041 \\ (0.0027)} & \makecell{0.1074 \\ (0.0029)} & \makecell{0.1104 \\ (0.0030)} & \makecell{0.1137 \\ (0.0031)} & \makecell{0.1172 \\ (0.0031)} & \makecell{0.1208 \\ (0.0031)} & \makecell{0.1244 \\ (0.0030)} & \makecell{0.1281 \\ (0.0032)} & \makecell{0.1321 \\ (0.0030)} & \makecell{0.1361 \\ (0.0028)} & \makecell{0.1574 \\ (0.0029)} & \makecell{0.1804 \\ (0.0039)} \\
 & 4 & \makecell{0.1022 \\ (0.0035)} & \makecell{0.1049 \\ (0.0035)} & \makecell{0.1079 \\ (0.0036)} & \makecell{0.1108 \\ (0.0038)} & \makecell{0.1135 \\ (0.0038)} & \makecell{0.1163 \\ (0.0040)} & \makecell{0.1192 \\ (0.0041)} & \makecell{0.1220 \\ (0.0042)} & \makecell{0.1251 \\ (0.0043)} & \makecell{0.1283 \\ (0.0042)} & \makecell{0.1453 \\ (0.0043)} & \makecell{0.1634 \\ (0.0043)} \\
 & 5 & \makecell{0.1017 \\ (0.0030)} & \makecell{0.1035 \\ (0.0032)} & \makecell{0.1053 \\ (0.0033)} & \makecell{0.1072 \\ (0.0036)} & \makecell{0.1091 \\ (0.0036)} & \makecell{0.1113 \\ (0.0036)} & \makecell{0.1131 \\ (0.0035)} & \makecell{0.1150 \\ (0.0035)} & \makecell{0.1171 \\ (0.0036)} & \makecell{0.1193 \\ (0.0035)} & \makecell{0.1300 \\ (0.0031)} & \makecell{0.1416 \\ (0.0034)} \\
 \bottomrule
\end{tabular}
}\label{tb:defocus_blur}
\end{table}


\begin{table}[h]
\caption{Main results for \textbf{Nonlinear MLP Classifier with ReLU} $(m=64)$ on \textbf{CIFAR-10} and \textbf{CIFAR-10-C (Gaussian Noise)} under $\epsilon$-LDP mechanisms. We evaluate utility by measuring the Accuracy. Higher accuracy indicates higher utility.}
\centering
\resizebox{\textwidth}{!}{
\begin{tabular}{ll cccccccccccc}
\toprule
& \multicolumn{12}{c}{Privacy Budget ($\epsilon$)} \\
 \cmidrule(l){3-14}
 \textbf{Mechanism} & $S$ & 0.5 & 1.0 & 1.5 & 2.0 & 2.5 & 3.0 & 3.5 & 4.0 & 4.5 & 5.0 & 7.5 & 10.0 \\
 \midrule
 \multirow{6}{*} \makecell {No randomization} 
 & 0 & 0.8937 & 0.8937 & 0.8937 & 0.8937 & 0.8937 & 0.8937 & 0.8937 & 0.8937 & 0.8937 & 0.8937 & 0.8937 & 0.8937 \\
 & 1 & 0.7046 & 0.7046 & 0.7046 & 0.7046 & 0.7046 & 0.7046 & 0.7046 & 0.7046 & 0.7046 & 0.7046 & 0.7046 & 0.7046 \\
 & 2 & 0.5093 & 0.5093 & 0.5093 & 0.5093 & 0.5093 & 0.5093 & 0.5093 & 0.5093 & 0.5093 & 0.5093 & 0.5093 & 0.5093 \\
 & 3 & 0.3512 & 0.3512 & 0.3512 & 0.3512 & 0.3512 & 0.3512 & 0.3512 & 0.3512 & 0.3512 & 0.3512 & 0.3512 & 0.3512 \\
 & 4 & 0.3040 & 0.3040 & 0.3040 & 0.3040 & 0.3040 & 0.3040 & 0.3040 & 0.3040 & 0.3040 & 0.3040 & 0.3040 & 0.3040 \\
 & 5 & 0.2621 & 0.2621 & 0.2621 & 0.2621 & 0.2621 & 0.2621 & 0.2621 & 0.2621 & 0.2621 & 0.2621 & 0.2621 & 0.2621 \\
 
 \midrule
 \multirow{6}{*} \makecell{Laplace ($\ell_1$)}
 & 0 & \makecell{0.1002 \\ (0.0029)} & \makecell{0.1017 \\ (0.0029)} & \makecell{0.1031 \\ (0.0028)} & \makecell{0.1046 \\ (0.0028)} & \makecell{0.1059 \\ (0.0027)} & \makecell{0.1074 \\ (0.0026)} & \makecell{0.1088 \\ (0.0025)} & \makecell{0.1103 \\ (0.0027)} & \makecell{0.1119 \\ (0.0026)} & \makecell{0.1133 \\ (0.0025)} & \makecell{0.1212 \\ (0.0026)} & \makecell{0.1301 \\ (0.0024)} \\
 & 1 & \makecell{0.0999 \\ (0.0029)} & \makecell{0.1007 \\ (0.0029)} & \makecell{0.1017 \\ (0.0029)} & \makecell{0.1026 \\ (0.0028)} & \makecell{0.1035 \\ (0.0028)} & \makecell{0.1044 \\ (0.0027)} & \makecell{0.1055 \\ (0.0027)} & \makecell{0.1065 \\ (0.0026)} & \makecell{0.1074 \\ (0.0024)} & \makecell{0.1085 \\ (0.0025)} & \makecell{0.1138 \\ (0.0025)} & \makecell{0.1194 \\ (0.0028)} \\
 & 2 & \makecell{0.0995 \\ (0.0029)} & \makecell{0.1000 \\ (0.0029)} & \makecell{0.1007 \\ (0.0029)} & \makecell{0.1015 \\ (0.0028)} & \makecell{0.1021 \\ (0.0028)} & \makecell{0.1028 \\ (0.0027)} & \makecell{0.1035 \\ (0.0026)} & \makecell{0.1041 \\ (0.0025)} & \makecell{0.1048 \\ (0.0025)} & \makecell{0.1054 \\ (0.0025)} & \makecell{0.1091 \\ (0.0025)} & \makecell{0.1132 \\ (0.0029)} \\
 & 3 & \makecell{0.0993 \\ (0.0029)} & \makecell{0.0998 \\ (0.0028)} & \makecell{0.1002 \\ (0.0028)} & \makecell{0.1008 \\ (0.0026)} & \makecell{0.1012 \\ (0.0026)} & \makecell{0.1017 \\ (0.0026)} & \makecell{0.1021 \\ (0.0025)} & \makecell{0.1025 \\ (0.0025)} & \makecell{0.1029 \\ (0.0026)} & \makecell{0.1033 \\ (0.0026)} & \makecell{0.1058 \\ (0.0027)} & \makecell{0.1087 \\ (0.0027)} \\
 & 4 & \makecell{0.0993 \\ (0.0029)} & \makecell{0.0996 \\ (0.0028)} & \makecell{0.0999 \\ (0.0028)} & \makecell{0.1004 \\ (0.0027)} & \makecell{0.1008 \\ (0.0026)} & \makecell{0.1010 \\ (0.0027)} & \makecell{0.1016 \\ (0.0026)} & \makecell{0.1019 \\ (0.0025)} & \makecell{0.1024 \\ (0.0024)} & \makecell{0.1028 \\ (0.0024)} & \makecell{0.1048 \\ (0.0024)} & \makecell{0.1070 \\ (0.0024)} \\
 & 5 & \makecell{0.0993 \\ (0.0029)} & \makecell{0.0995 \\ (0.0029)} & \makecell{0.0999 \\ (0.0028)} & \makecell{0.1002 \\ (0.0027)} & \makecell{0.1006 \\ (0.0027)} & \makecell{0.1009 \\ (0.0026)} & \makecell{0.1013 \\ (0.0026)} & \makecell{0.1016 \\ (0.0026)} & \makecell{0.1020 \\ (0.0027)} & \makecell{0.1023 \\ (0.0026)} & \makecell{0.1041 \\ (0.0026)} & \makecell{0.1059 \\ (0.0026)} \\
 
 \midrule
 \multirow{6}{*} \makecell{Laplace+PA} 
 & 0 & \makecell{0.1062 \\ (0.0024)} & \makecell{0.1122 \\ (0.0022)} & \makecell{0.1188 \\ (0.0023)} & \makecell{0.1258 \\ (0.0025)} & \makecell{0.1329 \\ (0.0023)} & \makecell{0.1404 \\ (0.0023)} & \makecell{0.1480 \\ (0.0024)} & \makecell{0.1563 \\ (0.0023)} & \makecell{0.1650 \\ (0.0024)} & \makecell{0.1737 \\ (0.0024)} & \makecell{0.2204 \\ (0.0038)} & \makecell{0.2723 \\ (0.0046)} \\
 & 1 & \makecell{0.1037 \\ (0.0023)} & \makecell{0.1081 \\ (0.0024)} & \makecell{0.1126 \\ (0.0027)} & \makecell{0.1171 \\ (0.0028)} & \makecell{0.1219 \\ (0.0029)} & \makecell{0.1267 \\ (0.0032)} & \makecell{0.1316 \\ (0.0029)} & \makecell{0.1365 \\ (0.0029)} & \makecell{0.1416 \\ (0.0030)} & \makecell{0.1469 \\ (0.0030)} & \makecell{0.1761 \\ (0.0036)} & \makecell{0.2074 \\ (0.0040)} \\
 & 2 & \makecell{0.1034 \\ (0.0031)} & \makecell{0.1063 \\ (0.0033)} & \makecell{0.1093 \\ (0.0035)} & \makecell{0.1123 \\ (0.0033)} & \makecell{0.1154 \\ (0.0037)} & \makecell{0.1186 \\ (0.0037)} & \makecell{0.1220 \\ (0.0040)} & \makecell{0.1252 \\ (0.0041)} & \makecell{0.1287 \\ (0.0039)} & \makecell{0.1324 \\ (0.0037)} & \makecell{0.1511 \\ (0.0044)} & \makecell{0.1718 \\ (0.0045)} \\
 & 3 & \makecell{0.1029 \\ (0.0027)} & \makecell{0.1051 \\ (0.0026)} & \makecell{0.1072 \\ (0.0026)} & \makecell{0.1094 \\ (0.0028)} & \makecell{0.1116 \\ (0.0032)} & \makecell{0.1139 \\ (0.0031)} & \makecell{0.1163 \\ (0.0028)} & \makecell{0.1186 \\ (0.0026)} & \makecell{0.1211 \\ (0.0027)} & \makecell{0.1234 \\ (0.0028)} & \makecell{0.1363 \\ (0.0028)} & \makecell{0.1490 \\ (0.0030)} \\
 & 4 & \makecell{0.1024 \\ (0.0020)} & \makecell{0.1043 \\ (0.0022)} & \makecell{0.1063 \\ (0.0021)} & \makecell{0.1081 \\ (0.0025)} & \makecell{0.1104 \\ (0.0025)} & \makecell{0.1121 \\ (0.0028)} & \makecell{0.1140 \\ (0.0026)} & \makecell{0.1160 \\ (0.0025)} & \makecell{0.1181 \\ (0.0026)} & \makecell{0.1202 \\ (0.0026)} & \makecell{0.1307 \\ (0.0031)} & \makecell{0.1415 \\ (0.0039)} \\
 & 5 & \makecell{0.1022 \\ (0.0033)} & \makecell{0.1035 \\ (0.0035)} & \makecell{0.1048 \\ (0.0035)} & \makecell{0.1065 \\ (0.0034)} & \makecell{0.1083 \\ (0.0034)} & \makecell{0.1099 \\ (0.0032)} & \makecell{0.1117 \\ (0.0032)} & \makecell{0.1136 \\ (0.0031)} & \makecell{0.1153 \\ (0.0034)} & \makecell{0.1170 \\ (0.0034)} & \makecell{0.1264 \\ (0.0033)} & \makecell{0.1357 \\ (0.0035)} \\

 \midrule
 \multirow{6}{*} \makecell{PrivUnit2} 
 & 0 & \makecell{0.1107 \\ (0.0032)} & \makecell{0.1221 \\ (0.0032)} & \makecell{0.1348 \\ (0.0032)} & \makecell{0.1481 \\ (0.0031)} & \makecell{0.1620 \\ (0.0028)} & \makecell{0.1762 \\ (0.0029)} & \makecell{0.1913 \\ (0.0038)} & \makecell{0.2063 \\ (0.0044)} & \makecell{0.2229 \\ (0.0044)} & \makecell{0.2383 \\ (0.0040)} & \makecell{0.3185 \\ (0.0042)} & \makecell{0.3978 \\ (0.0035)} \\
 & 1 & \makecell{0.1081 \\ (0.0041)} & \makecell{0.1166 \\ (0.0037)} & \makecell{0.1265 \\ (0.0036)} & \makecell{0.1366 \\ (0.0033)} & \makecell{0.1469 \\ (0.0031)} & \makecell{0.1569 \\ (0.0032)} & \makecell{0.1678 \\ (0.0037)} & \makecell{0.1783 \\ (0.0042)} & \makecell{0.1897 \\ (0.0043)} & \makecell{0.2008 \\ (0.0038)} & \makecell{0.2576 \\ (0.0044)} & \makecell{0.3113 \\ (0.0037)} \\
 & 2 & \makecell{0.1058 \\ (0.0036)} & \makecell{0.1126 \\ (0.0036)} & \makecell{0.1200 \\ (0.0038)} & \makecell{0.1274 \\ (0.0035)} & \makecell{0.1349 \\ (0.0034)} & \makecell{0.1421 \\ (0.0033)} & \makecell{0.1500 \\ (0.0038)} & \makecell{0.1580 \\ (0.0047)} & \makecell{0.1666 \\ (0.0047)} & \makecell{0.1749 \\ (0.0044)} & \makecell{0.2143 \\ (0.0029)} & \makecell{0.2503 \\ (0.0026)} \\
 & 3 & \makecell{0.1046 \\ (0.0034)} & \makecell{0.1096 \\ (0.0032)} & \makecell{0.1150 \\ (0.0031)} & \makecell{0.1205 \\ (0.0031)} & \makecell{0.1257 \\ (0.0033)} & \makecell{0.1310 \\ (0.0034)} & \makecell{0.1370 \\ (0.0039)} & \makecell{0.1429 \\ (0.0045)} & \makecell{0.1480 \\ (0.0045)} & \makecell{0.1539 \\ (0.0044)} & \makecell{0.1813 \\ (0.0034)} & \makecell{0.2053 \\ (0.0046)} \\
 & 4 & \makecell{0.1038 \\ (0.0037)} & \makecell{0.1080 \\ (0.0034)} & \makecell{0.1125 \\ (0.0036)} & \makecell{0.1176 \\ (0.0038)} & \makecell{0.1220 \\ (0.0037)} & \makecell{0.1267 \\ (0.0035)} & \makecell{0.1318 \\ (0.0039)} & \makecell{0.1367 \\ (0.0047)} & \makecell{0.1414 \\ (0.0047)} & \makecell{0.1465 \\ (0.0044)} & \makecell{0.1694 \\ (0.0033)} & \makecell{0.1887 \\ (0.0039)} \\
 & 5 & \makecell{0.1029 \\ (0.0038)} & \makecell{0.1067 \\ (0.0038)} & \makecell{0.1107 \\ (0.0039)} & \makecell{0.1150 \\ (0.0038)} & \makecell{0.1187 \\ (0.0038)} & \makecell{0.1225 \\ (0.0037)} & \makecell{0.1268 \\ (0.0041)} & \makecell{0.1311 \\ (0.0046)} & \makecell{0.1350 \\ (0.0048)} & \makecell{0.1392 \\ (0.0044)} & \makecell{0.1575 \\ (0.0027)} & \makecell{0.1746 \\ (0.0032)} \\

 \midrule
 \multirow{6}{*} \makecell{PrivUnit2+PA} 
 & 0 & \makecell{0.1203 \\ (0.0022)} & \makecell{0.1440 \\ (0.0030)} & \makecell{0.1710 \\ (0.0027)} & \makecell{0.2012 \\ (0.0028)} & \makecell{0.2341 \\ (0.0031)} & \makecell{0.2673 \\ (0.0040)} & \makecell{0.3027 \\ (0.0039)} & \makecell{0.3377 \\ (0.0041)} & \makecell{0.3733 \\ (0.0048)} & \makecell{0.4076 \\ (0.0052)} & \makecell{0.5477 \\ (0.0055)} & \makecell{0.6451 \\ (0.0057)} \\
 & 1 & \makecell{0.1159 \\ (0.0030)} & \makecell{0.1344 \\ (0.0035)} & \makecell{0.1551 \\ (0.0038)} & \makecell{0.1774 \\ (0.0041)} & \makecell{0.2007 \\ (0.0036)} & \makecell{0.2247 \\ (0.0033)} & \makecell{0.2491 \\ (0.0033)} & \makecell{0.2734 \\ (0.0042)} & \makecell{0.2961 \\ (0.0039)} & \makecell{0.3188 \\ (0.0039)} & \makecell{0.4099 \\ (0.0057)} & \makecell{0.4742 \\ (0.0040)} \\
 & 2 & \makecell{0.1125 \\ (0.0019)} & \makecell{0.1268 \\ (0.0024)} & \makecell{0.1419 \\ (0.0027)} & \makecell{0.1580 \\ (0.0030)} & \makecell{0.1744 \\ (0.0029)} & \makecell{0.1911 \\ (0.0027)} & \makecell{0.2077 \\ (0.0024)} & \makecell{0.2240 \\ (0.0028)} & \makecell{0.2395 \\ (0.0028)} & \makecell{0.2549 \\ (0.0026)} & \makecell{0.3142 \\ (0.0040)} & \makecell{0.3528 \\ (0.0046)} \\
 & 3 & \makecell{0.1096 \\ (0.0025)} & \makecell{0.1202 \\ (0.0030)} & \makecell{0.1315 \\ (0.0035)} & \makecell{0.1429 \\ (0.0035)} & \makecell{0.1545 \\ (0.0040)} & \makecell{0.1662 \\ (0.0041)} & \makecell{0.1777 \\ (0.0038)} & \makecell{0.1884 \\ (0.0037)} & \makecell{0.1983 \\ (0.0029)} & \makecell{0.2091 \\ (0.0033)} & \makecell{0.2431 \\ (0.0038)} & \makecell{0.2646 \\ (0.0042)} \\
 & 4 & \makecell{0.1080 \\ (0.0034)} & \makecell{0.1169 \\ (0.0030)} & \makecell{0.1263 \\ (0.0030)} & \makecell{0.1363 \\ (0.0031)} & \makecell{0.1467 \\ (0.0036)} & \makecell{0.1570 \\ (0.0039)} & \makecell{0.1661 \\ (0.0036)} & \makecell{0.1749 \\ (0.0037)} & \makecell{0.1830 \\ (0.0031)} & \makecell{0.1913 \\ (0.0020)} & \makecell{0.2198 \\ (0.0023)} & \makecell{0.2356 \\ (0.0036)} \\
 & 5 & \makecell{0.1067 \\ (0.0039)} & \makecell{0.1141 \\ (0.0036)} & \makecell{0.1218 \\ (0.0041)} & \makecell{0.1299 \\ (0.0039)} & \makecell{0.1385 \\ (0.0036)} & \makecell{0.1467 \\ (0.0040)} & \makecell{0.1545 \\ (0.0037)} & \makecell{0.1618 \\ (0.0038)} & \makecell{0.1693 \\ (0.0033)} & \makecell{0.1761 \\ (0.0027)} & \makecell{0.1995 \\ (0.0024)} & \makecell{0.2126 \\ (0.0035)} \\

 \midrule
 \multirow{6}{*} \makecell{PrivUnitG} 
 & 0 & \makecell{0.1100 \\ (0.0027)} & \makecell{0.1212 \\ (0.0030)} & \makecell{0.1335 \\ (0.0028)} & \makecell{0.1464 \\ (0.0035)} & \makecell{0.1600 \\ (0.0037)} & \makecell{0.1741 \\ (0.0040)} & \makecell{0.1887 \\ (0.0044)} & \makecell{0.2036 \\ (0.0039)} & \makecell{0.2187 \\ (0.0040)} & \makecell{0.2342 \\ (0.0043)} & \makecell{0.3103 \\ (0.0034)} & \makecell{0.3685 \\ (0.0035)} \\
 & 1 & \makecell{0.1071 \\ (0.0028)} & \makecell{0.1158 \\ (0.0029)} & \makecell{0.1250 \\ (0.0034)} & \makecell{0.1348 \\ (0.0039)} & \makecell{0.1448 \\ (0.0040)} & \makecell{0.1553 \\ (0.0039)} & \makecell{0.1658 \\ (0.0040)} & \makecell{0.1764 \\ (0.0043)} & \makecell{0.1873 \\ (0.0047)} & \makecell{0.1983 \\ (0.0043)} & \makecell{0.2513 \\ (0.0035)} & \makecell{0.2920 \\ (0.0037)} \\
 & 2 & \makecell{0.1051 \\ (0.0028)} & \makecell{0.1118 \\ (0.0031)} & \makecell{0.1190 \\ (0.0032)} & \makecell{0.1264 \\ (0.0036)} & \makecell{0.1336 \\ (0.0040)} & \makecell{0.1415 \\ (0.0041)} & \makecell{0.1493 \\ (0.0043)} & \makecell{0.1569 \\ (0.0043)} & \makecell{0.1646 \\ (0.0043)} & \makecell{0.1726 \\ (0.0046)} & \makecell{0.2101 \\ (0.0046)} & \makecell{0.2378 \\ (0.0045)} \\
 & 3 & \makecell{0.1031 \\ (0.0026)} & \makecell{0.1078 \\ (0.0023)} & \makecell{0.1132 \\ (0.0025)} & \makecell{0.1188 \\ (0.0028)} & \makecell{0.1242 \\ (0.0029)} & \makecell{0.1297 \\ (0.0034)} & \makecell{0.1352 \\ (0.0039)} & \makecell{0.1405 \\ (0.0046)} & \makecell{0.1458 \\ (0.0045)} & \makecell{0.1511 \\ (0.0046)} & \makecell{0.1772 \\ (0.0051)} & \makecell{0.1965 \\ (0.0044)} \\
 & 4 & \makecell{0.1024 \\ (0.0028)} & \makecell{0.1067 \\ (0.0029)} & \makecell{0.1113 \\ (0.0030)} & \makecell{0.1162 \\ (0.0032)} & \makecell{0.1209 \\ (0.0034)} & \makecell{0.1254 \\ (0.0036)} & \makecell{0.1298 \\ (0.0041)} & \makecell{0.1345 \\ (0.0041)} & \makecell{0.1390 \\ (0.0041)} & \makecell{0.1438 \\ (0.0044)} & \makecell{0.1656 \\ (0.0035)} & \makecell{0.1807 \\ (0.0035)} \\
 & 5 & \makecell{0.1021 \\ (0.0026)} & \makecell{0.1059 \\ (0.0027)} & \makecell{0.1100 \\ (0.0028)} & \makecell{0.1139 \\ (0.0032)} & \makecell{0.1175 \\ (0.0035)} & \makecell{0.1213 \\ (0.0037)} & \makecell{0.1250 \\ (0.0037)} & \makecell{0.1290 \\ (0.0035)} & \makecell{0.1332 \\ (0.0039)} & \makecell{0.1372 \\ (0.0043)} & \makecell{0.1550 \\ (0.0040)} & \makecell{0.1681 \\ (0.0037)} \\
 
 \midrule
 \multirow{6}{*} \makecell{PrivUnitG+PA} 
 & 0 & \makecell{0.1213 \\ (0.0041)} & \makecell{0.1456 \\ (0.0049)} & \makecell{0.1726 \\ (0.0050)} & \makecell{0.2020 \\ (0.0047)} & \makecell{0.2349 \\ (0.0044)} & \makecell{0.2684 \\ (0.0043)} & \makecell{0.3029 \\ (0.0041)} & \makecell{0.3367 \\ (0.0043)} & \makecell{0.3700 \\ (0.0044)} & \makecell{0.4016 \\ (0.0046)} & \makecell{0.5364 \\ (0.0044)} & \makecell{0.6254 \\ (0.0045)} \\
 & 1 & \makecell{0.1174 \\ (0.0025)} & \makecell{0.1358 \\ (0.0029)} & \makecell{0.1570 \\ (0.0028)} & \makecell{0.1787 \\ (0.0035)} & \makecell{0.2017 \\ (0.0041)} & \makecell{0.2253 \\ (0.0045)} & \makecell{0.2491 \\ (0.0052)} & \makecell{0.2723 \\ (0.0050)} & \makecell{0.2938 \\ (0.0052)} & \makecell{0.3154 \\ (0.0051)} & \makecell{0.4018 \\ (0.0059)} & \makecell{0.4619 \\ (0.0051)} \\
 & 2 & \makecell{0.1121 \\ (0.0034)} & \makecell{0.1260 \\ (0.0036)} & \makecell{0.1415 \\ (0.0034)} & \makecell{0.1577 \\ (0.0031)} & \makecell{0.1740 \\ (0.0029)} & \makecell{0.1903 \\ (0.0029)} & \makecell{0.2069 \\ (0.0030)} & \makecell{0.2230 \\ (0.0026)} & \makecell{0.2383 \\ (0.0028)} & \makecell{0.2518 \\ (0.0026)} & \makecell{0.3084 \\ (0.0034)} & \makecell{0.3454 \\ (0.0038)} \\
 & 3 & \makecell{0.1100 \\ (0.0038)} & \makecell{0.1201 \\ (0.0039)} & \makecell{0.1316 \\ (0.0042)} & \makecell{0.1430 \\ (0.0033)} & \makecell{0.1542 \\ (0.0028)} & \makecell{0.1660 \\ (0.0031)} & \makecell{0.1775 \\ (0.0032)} & \makecell{0.1872 \\ (0.0033)} & \makecell{0.1971 \\ (0.0032)} & \makecell{0.2061 \\ (0.0036)} & \makecell{0.2402 \\ (0.0030)} & \makecell{0.2625 \\ (0.0030)} \\
 & 4 & \makecell{0.1081 \\ (0.0027)} & \makecell{0.1172 \\ (0.0029)} & \makecell{0.1271 \\ (0.0034)} & \makecell{0.1366 \\ (0.0035)} & \makecell{0.1467 \\ (0.0036)} & \makecell{0.1560 \\ (0.0038)} & \makecell{0.1650 \\ (0.0040)} & \makecell{0.1735 \\ (0.0036)} & \makecell{0.1822 \\ (0.0036)} & \makecell{0.1894 \\ (0.0032)} & \makecell{0.2166 \\ (0.0019)} & \makecell{0.2345 \\ (0.0028)} \\
 & 5 & \makecell{0.1081 \\ (0.0015)} & \makecell{0.1162 \\ (0.0020)} & \makecell{0.1242 \\ (0.0026)} & \makecell{0.1317 \\ (0.0024)} & \makecell{0.1403 \\ (0.0023)} & \makecell{0.1481 \\ (0.0021)} & \makecell{0.1556 \\ (0.0022)} & \makecell{0.1626 \\ (0.0019)} & \makecell{0.1691 \\ (0.0020)} & \makecell{0.1754 \\ (0.0019)} & \makecell{0.1978 \\ (0.0028)} & \makecell{0.2118 \\ (0.0028)} \\
 
 \midrule
 \multirow{6}{*} \makecell{CW+PA w/o bounding} 
 & 0 & \makecell{0.1039 \\ (0.0024)} & \makecell{0.1077 \\ (0.0025)} & \makecell{0.1120 \\ (0.0026)} & \makecell{0.1161 \\ (0.0026)} & \makecell{0.1205 \\ (0.0026)} & \makecell{0.1251 \\ (0.0027)} & \makecell{0.1298 \\ (0.0026)} & \makecell{0.1346 \\ (0.0025)} & \makecell{0.1397 \\ (0.0026)} & \makecell{0.1446 \\ (0.0028)} & \makecell{0.1716 \\ (0.0027)} & \makecell{0.2008 \\ (0.0026)} \\
 & 1 & \makecell{0.1027 \\ (0.0016)} & \makecell{0.1051 \\ (0.0025)} & \makecell{0.1079 \\ (0.0027)} & \makecell{0.1106 \\ (0.0029)} & \makecell{0.1133 \\ (0.0029)} & \makecell{0.1161 \\ (0.0029)} & \makecell{0.1191 \\ (0.0028)} & \makecell{0.1221 \\ (0.0030)} & \makecell{0.1252 \\ (0.0029)} & \makecell{0.1286 \\ (0.0028)} & \makecell{0.1452 \\ (0.0028)} & \makecell{0.1636 \\ (0.0030)} \\ 
 & 2 & \makecell{0.1022 \\ (0.0034)} & \makecell{0.1041 \\ (0.0034)} & \makecell{0.1059 \\ (0.0036)} & \makecell{0.1077 \\ (0.0036)} & \makecell{0.1098 \\ (0.0035)} & \makecell{0.1119 \\ (0.0037)} & \makecell{0.1138 \\ (0.0037)} & \makecell{0.1159 \\ (0.0038)} & \makecell{0.1180 \\ (0.0036)} & \makecell{0.1201 \\ (0.0037)} & \makecell{0.1309 \\ (0.0038)} & \makecell{0.1429 \\ (0.0039)} \\ 
 & 3 & \makecell{0.1019 \\ (0.0023)} & \makecell{0.1034 \\ (0.0023)} & \makecell{0.1048 \\ (0.0023)} & \makecell{0.1060 \\ (0.0024)} & \makecell{0.1075 \\ (0.0027)} & \makecell{0.1089 \\ (0.0027)} & \makecell{0.1103 \\ (0.0028)} & \makecell{0.1118 \\ (0.0028)} & \makecell{0.1133 \\ (0.0028)} & \makecell{0.1150 \\ (0.0028)} & \makecell{0.1229 \\ (0.0031)} & \makecell{0.1313 \\ (0.0032)} \\
 & 4 & \makecell{0.1021 \\ (0.0027)} & \makecell{0.1033 \\ (0.0028)} & \makecell{0.1043 \\ (0.0026)} & \makecell{0.1055 \\ (0.0026)} & \makecell{0.1065 \\ (0.0026)} & \makecell{0.1077 \\ (0.0027)} & \makecell{0.1089 \\ (0.0026)} & \makecell{0.1102 \\ (0.0024)} & \makecell{0.1114 \\ (0.0023)} & \makecell{0.1126 \\ (0.0024)} & \makecell{0.1191 \\ (0.0030)} & \makecell{0.1263 \\ (0.0032)} \\
 & 5 & \makecell{0.1017 \\ (0.0029)} & \makecell{0.1027 \\ (0.0029)} & \makecell{0.1037 \\ (0.0028)} & \makecell{0.1047 \\ (0.0029)} & \makecell{0.1058 \\ (0.0030)} & \makecell{0.1069 \\ (0.0032)} & \makecell{0.1080 \\ (0.0032)} & \makecell{0.1090 \\ (0.0032)} & \makecell{0.1099 \\ (0.0034)} & \makecell{0.1112 \\ (0.0034)} & \makecell{0.1168 \\ (0.0036)} & \makecell{0.1225 \\ (0.0038)} \\
  \bottomrule
\end{tabular}
}\label{tb:Gaussian}
\end{table}

\clearpage

\subsubsection{Main Results: Accuracy of Linear Regression, Linear Classification, and Nonlinear Classification on CIFAR-10 and MNIST Test Datasets}\label{apdx_activation}
We provide our main results for linear regression, linear classification, and nonlinear classification under $\epsilon$-LDP mechanisms.
In the tables, PA denotes the proposed approach.

\begin{table}[h]
\caption{Main results for $(m=16)$ \textbf{Linear Regression} on \textbf{LHSM} ($m=16$) under $\epsilon$-LDP mechanisms. We evaluate utility by measuring the RMSE using real values. Lower RMSE indicates higher utility.}
\centering
\resizebox{\textwidth}{!}{
\begin{tabular}{l cccccccccccc}
\toprule
& \multicolumn{12}{c}{Privacy Budget ($\epsilon$)} \\
 \cmidrule(l){2-13}
 \textbf{Mechanism} & 0.5 & 1.0 & 1.5 & 2.0 & 2.5 & 3.0 & 3.5 & 4.0 & 4.5 & 5.0 & 7.5 & 10.0 \\
 \midrule
 \makecell[l]{No randomization} & 0.0691 & 0.0691 & 0.0691 & 0.0691 & 0.0691 & 0.0691 & 0.0691 & 0.0691 & 0.0691 & 0.0691 & 0.0691 & 0.0691 \\
 \midrule
 \makecell[l]{Laplace ($\ell_1$)} &
\makecell{17.6293 \\ (0.3513)} & \makecell{8.8154 \\ (0.1761)} & \makecell{5.8778 \\ (0.1177)} & \makecell{4.4093 \\ (0.0885)} & \makecell{3.5284 \\ (0.0710)} & \makecell{2.9413 \\ (0.0593)} & \makecell{2.5221 \\ (0.0509)} & \makecell{2.2078 \\ (0.0447)} & \makecell{1.9635 \\ (0.0398)} & \makecell{1.7682 \\ (0.0359)} & \makecell{1.1832 \\ (0.0242)} & \makecell{0.8921 \\ (0.0184)} \\
 \midrule
 \makecell[l]{Laplace+PA} &
\makecell{1.0520 \\ (0.0245)} & \makecell{0.5387 \\ (0.0122)} & \makecell{0.3729 \\ (0.0086)} & \makecell{0.2935 \\ (0.0075)} & \makecell{0.2482 \\ (0.0074)} & \makecell{0.2197 \\ (0.0077)} & \makecell{0.2006 \\ (0.0080)} & \makecell{0.1872 \\ (0.0084)} & \makecell{0.1774 \\ (0.0088)} & \makecell{0.1700 \\ (0.0091)} & \makecell{0.1511 \\ (0.0101)} & \makecell{0.1439 \\ (0.0105)} \\
 \midrule
 \makecell[l]{PrivUnit2} &
\makecell{3.9709 \\ (0.0484)} & \makecell{2.0035 \\ (0.0248)} & \makecell{1.3492 \\ (0.0144)} & \makecell{1.0252 \\ (0.0131)} & \makecell{0.8348 \\ (0.0108)} & \makecell{0.7097 \\ (0.0087)} & \makecell{0.6222 \\ (0.0093)} & \makecell{0.5577 \\ (0.0077)} & \makecell{0.5083 \\ (0.0082)} & \makecell{0.4710 \\ (0.0064)} & \makecell{0.3619 \\ (0.0062)} & \makecell{0.3155 \\ (0.0058)} \\
 \midrule
 \makecell[l]{PrivUnit2+PA} &
\makecell{0.9048 \\ (0.0155)} & \makecell{0.4776 \\ (0.0081)} & \makecell{0.3437 \\ (0.0064)} & \makecell{0.2822 \\ (0.0057)} & \makecell{0.2490 \\ (0.0053)} & \makecell{0.2282 \\ (0.0062)} & \makecell{0.2150 \\ (0.0066)} & \makecell{0.2058 \\ (0.0068)} & \makecell{0.1991 \\ (0.0072)} & \makecell{0.1945 \\ (0.0072)} & \makecell{0.1826 \\ (0.0080)} & \makecell{0.1780 \\ (0.0083)} \\
 \midrule
 \makecell[l]{PrivUnitG} &
\makecell{3.9645 \\ (0.0625)} & \makecell{1.9924 \\ (0.0252)} & \makecell{1.3471 \\ (0.0171)} & \makecell{1.0290 \\ (0.0136)} & \makecell{0.8425 \\ (0.0100)} & \makecell{0.7192 \\ (0.0092)} & \makecell{0.6333 \\ (0.0077)} & \makecell{0.5707 \\ (0.0075)} & \makecell{0.5235 \\ (0.0070)} & \makecell{0.4865 \\ (0.0061)} & \makecell{0.3831 \\ (0.0059)} & \makecell{0.3365 \\ (0.0053)} \\
 \midrule
 \makecell[l]{PrivUnitG+PA} &
\makecell{0.8996 \\ (0.0106)} & \makecell{0.4736 \\ (0.0069)} & \makecell{0.3406 \\ (0.0055)} & \makecell{0.2801 \\ (0.0048)} & \makecell{0.2471 \\ (0.0051)} & \makecell{0.2269 \\ (0.0054)} & \makecell{0.2136 \\ (0.0061)} & \makecell{0.2044 \\ (0.0065)} & \makecell{0.1981 \\ (0.0072)} & \makecell{0.1936 \\ (0.0071)} & \makecell{0.1822 \\ (0.0081)} & \makecell{0.1765 \\ (0.0084)} \\
 \midrule
 \makecell[l]{CW+PA w/o bounding} &
\makecell{1.1897 \\ (0.0281)} & \makecell{0.5976 \\ (0.0142)} & \makecell{0.4016 \\ (0.0096)} & \makecell{0.3045 \\ (0.0073)} & \makecell{0.2470 \\ (0.0060)} & \makecell{0.2093 \\ (0.0051)} & \makecell{0.1828 \\ (0.0045)} & \makecell{0.1634 \\ (0.0040)} & \makecell{0.1486 \\ (0.0037)} & \makecell{0.1371 \\ (0.0035)} & \makecell{0.1048 \\ (0.0031)} & \makecell{0.0908 \\ (0.0030)} \\
 \midrule
 \makecell[l]{Task-Aware} &
\makecell{0.1920 \\ (0.0093)} & \makecell{0.1920 \\ (0.0093)} & \makecell{0.1918 \\ (0.0093)} & \makecell{0.1916 \\ (0.0094)} & \makecell{0.1913 \\ (0.0094)} & \makecell{0.1910 \\ (0.0094)} & \makecell{0.1906 \\ (0.0094)} & \makecell{0.1901 \\ (0.0094)} & \makecell{0.1896 \\ (0.0093)} & \makecell{0.1890 \\ (0.0093)} & \makecell{0.1852 \\ (0.0093)} & \makecell{0.1802 \\ (0.0092)} \\
 \bottomrule
\end{tabular}
}
\end{table}

\begin{table}[h]
\caption{Main results for $(m=64)$ \textbf{Linear Classification} on \textbf{CIFAR-10} under $\epsilon$-LDP mechanisms. We evaluate utility by measuring the Accuracy. Higher accuracy indicates higher utility.}
\centering
\resizebox{\textwidth}{!}{
\begin{tabular}{l cccccccccccc}
\toprule
& \multicolumn{12}{c}{Privacy Budget ($\epsilon$)} \\
 \cmidrule(l){2-13}
 \textbf{Mechanism} & 0.5 & 1.0 & 1.5 & 2.0 & 2.5 & 3.0 & 3.5 & 4.0 & 4.5 & 5.0 & 7.5 & 10.0 \\
 \midrule
 {No randomization} &
0.9058 & 0.9058 & 0.9058 & 0.9058 & 0.9058 & 0.9058 & 0.9058 & 0.9058 & 0.9058 & 0.9058 & 0.9058 & 0.9058 \\
 \midrule
 \makecell[l]{Laplace ($\ell_1$)} &
\makecell{0.1020 \\ (0.0027)} & \makecell{0.1041 \\ (0.0028)} & \makecell{0.1063 \\ (0.0030)} & \makecell{0.1083 \\ (0.0028)} & \makecell{0.1103 \\ (0.0028)} & \makecell{0.1126 \\ (0.0027)} & \makecell{0.1151 \\ (0.0028)} & \makecell{0.1176 \\ (0.0030)} & \makecell{0.1199 \\ (0.0029)} & \makecell{0.1221 \\ (0.0029)} & \makecell{0.1347 \\ (0.0027)} & \makecell{0.1480 \\ (0.0029)} \\
 \midrule
 \makecell[l]{Laplace+PA} &
\makecell{0.1085 \\ (0.0033)} & \makecell{0.1170 \\ (0.0035)} & \makecell{0.1258 \\ (0.0037)} & \makecell{0.1352 \\ (0.0037)} & \makecell{0.1454 \\ (0.0035)} & \makecell{0.1555 \\ (0.0036)} & \makecell{0.1661 \\ (0.0038)} & \makecell{0.1777 \\ (0.0038)} & \makecell{0.1891 \\ (0.0041)} & \makecell{0.2006 \\ (0.0035)} & \makecell{0.2658 \\ (0.0038)} & \makecell{0.3342 \\ (0.0045)} \\
 \midrule
 \makecell[l]{PrivUnit2} &
\makecell{0.1174 \\ (0.0034)} & \makecell{0.1363 \\ (0.0038)} & \makecell{0.1572 \\ (0.0047)} & \makecell{0.1805 \\ (0.0048)} & \makecell{0.2047 \\ (0.0052)} & \makecell{0.2299 \\ (0.0051)} & \makecell{0.2561 \\ (0.0054)} & \makecell{0.2827 \\ (0.0056)} & \makecell{0.3109 \\ (0.0047)} & \makecell{0.3397 \\ (0.0047)} & \makecell{0.4681 \\ (0.0047)} & \makecell{0.5728 \\ (0.0040)} \\
 \midrule
 \makecell[l]{PrivUnit2+PA} &
\makecell{0.1258 \\ (0.0026)} & \makecell{0.1587 \\ (0.0033)} & \makecell{0.1989 \\ (0.0038)} & \makecell{0.2461 \\ (0.0039)} & \makecell{0.2967 \\ (0.0033)} & \makecell{0.3496 \\ (0.0037)} & \makecell{0.4014 \\ (0.0049)} & \makecell{0.4510 \\ (0.0049)} & \makecell{0.4962 \\ (0.0055)} & \makecell{0.5368 \\ (0.0045)} & \makecell{0.6782 \\ (0.0040)} & \makecell{0.7519 \\ (0.0036)} \\
 \midrule
 \makecell[l]{PrivUnitG} &
\makecell{0.1161 \\ (0.0037)} & \makecell{0.1353 \\ (0.0041)} & \makecell{0.1564 \\ (0.0044)} & \makecell{0.1795 \\ (0.0047)} & \makecell{0.2032 \\ (0.0050)} & \makecell{0.2279 \\ (0.0055)} & \makecell{0.2546 \\ (0.0057)} & \makecell{0.2808 \\ (0.0054)} & \makecell{0.3067 \\ (0.0059)} & \makecell{0.3330 \\ (0.0059)} & \makecell{0.4551 \\ (0.0054)} & \makecell{0.5392 \\ (0.0053)} \\
 \midrule
 \makecell[l]{PrivUnitG+PA} &
\makecell{0.1265 \\ (0.0039)} & \makecell{0.1598 \\ (0.0040)} & \makecell{0.1997 \\ (0.0048)} & \makecell{0.2464 \\ (0.0040)} & \makecell{0.2971 \\ (0.0047)} & \makecell{0.3492 \\ (0.0042)} & \makecell{0.3997 \\ (0.0041)} & \makecell{0.4472 \\ (0.0038)} & \makecell{0.4912 \\ (0.0037)} & \makecell{0.5302 \\ (0.0037)} & \makecell{0.6675 \\ (0.0037)} & \makecell{0.7485 \\ (0.0031)} \\
 \midrule
 \makecell[l]{CW+PA w/o bounding} &
\makecell{0.1057 \\ (0.0030)} & \makecell{0.1114 \\ (0.0032)} & \makecell{0.1170 \\ (0.0032)} & \makecell{0.1227 \\ (0.0034)} & \makecell{0.1286 \\ (0.0035)} & \makecell{0.1349 \\ (0.0033)} & \makecell{0.1414 \\ (0.0033)} & \makecell{0.1480 \\ (0.0033)} & \makecell{0.1546 \\ (0.0030)} & \makecell{0.1615 \\ (0.0034)} & \makecell{0.2001 \\ (0.0038)} & \makecell{0.2426 \\ (0.0039)} \\
 \midrule
 \rowcolor{gray!15} {Task-Aware} &
\makecell{0.1000 \\ (0.0000)} & \makecell{0.1000 \\ (0.0000)} & \makecell{0.1001 \\ (0.0002)} & \makecell{0.1002 \\ (0.0004)} & \makecell{0.1005 \\ (0.0005)} & \makecell{0.1011 \\ (0.0007)} & \makecell{0.1018 \\ (0.0008)} & \makecell{0.1028 \\ (0.0011)} & \makecell{0.1039 \\ (0.0013)} & \makecell{0.1053 \\ (0.0017)} & \makecell{0.1145 \\ (0.0014)} & \makecell{0.1276 \\ (0.0022)} \\
 \bottomrule
\end{tabular}
}
\end{table} \label{tb:4}

\begin{table}[h]
\caption{Main results for \textbf{Nonlinear MLP Classifier with ReLU} $(m=64)$ on \textbf{CIFAR-10} under $\epsilon$-LDP mechanisms. We evaluate utility by measuring the Accuracy. Higher accuracy indicates higher utility.}
\centering
\resizebox{\textwidth}{!}{
\begin{tabular}{ll cccccccccccc}
\toprule
& \multicolumn{12}{c}{Privacy Budget ($\epsilon$)} \\
 \cmidrule(l){2-13}
 \textbf{Mechanism} & 0.5 & 1.0 & 1.5 & 2.0 & 2.5 & 3.0 & 3.5 & 4.0 & 4.5 & 5.0 & 7.5 & 10.0 \\
 \midrule
 \makecell[l]{No randomization} 
 & 0.8937 & 0.8937 & 0.8937 & 0.8937 & 0.8937 & 0.8937 & 0.8937 & 0.8937 & 0.8937 & 0.8937 & 0.8937 & 0.8937 \\
 
 \midrule
 \makecell[l]{Laplace ($\ell_1$)}
 & \makecell{0.1002 \\ (0.0029)} & \makecell{0.1017 \\ (0.0029)} & \makecell{0.1031 \\ (0.0028)} & \makecell{0.1046 \\ (0.0028)} & \makecell{0.1059 \\ (0.0027)} & \makecell{0.1074 \\ (0.0026)} & \makecell{0.1088 \\ (0.0025)} & \makecell{0.1103 \\ (0.0027)} & \makecell{0.1119 \\ (0.0026)} & \makecell{0.1133 \\ (0.0025)} & \makecell{0.1212 \\ (0.0026)} & \makecell{0.1301 \\ (0.0024)} \\
 
 \midrule
 \makecell[l]{Laplace+PA} 
 & \makecell{0.1062 \\ (0.0024)} & \makecell{0.1122 \\ (0.0022)} & \makecell{0.1188 \\ (0.0023)} & \makecell{0.1258 \\ (0.0025)} & \makecell{0.1329 \\ (0.0023)} & \makecell{0.1404 \\ (0.0023)} & \makecell{0.1480 \\ (0.0024)} & \makecell{0.1563 \\ (0.0023)} & \makecell{0.1650 \\ (0.0024)} & \makecell{0.1737 \\ (0.0024)} & \makecell{0.2204 \\ (0.0038)} & \makecell{0.2723 \\ (0.0046)} \\
 
 \midrule
 \makecell[l]{PrivUnit2} 
 & \makecell{0.1107 \\ (0.0032)} & \makecell{0.1221 \\ (0.0032)} & \makecell{0.1348 \\ (0.0032)} & \makecell{0.1481 \\ (0.0031)} & \makecell{0.1620 \\ (0.0028)} & \makecell{0.1762 \\ (0.0029)} & \makecell{0.1913 \\ (0.0038)} & \makecell{0.2063 \\ (0.0044)} & \makecell{0.2229 \\ (0.0044)} & \makecell{0.2383 \\ (0.0040)} & \makecell{0.3185 \\ (0.0042)} & \makecell{0.3978 \\ (0.0035)} \\
 
 \midrule
 \makecell[l]{PrivUnit2+PA} 
 & \makecell{0.1203 \\ (0.0022)} & \makecell{0.1440 \\ (0.0030)} & \makecell{0.1710 \\ (0.0027)} & \makecell{0.2012 \\ (0.0028)} & \makecell{0.2341 \\ (0.0031)} & \makecell{0.2673 \\ (0.0040)} & \makecell{0.3027 \\ (0.0039)} & \makecell{0.3377 \\ (0.0041)} & \makecell{0.3733 \\ (0.0048)} & \makecell{0.4076 \\ (0.0052)} & \makecell{0.5477 \\ (0.0055)} & \makecell{0.6451 \\ (0.0057)} \\
 
 \midrule
 \makecell[l]{PrivUnitG} 
 & \makecell{0.1100 \\ (0.0027)} & \makecell{0.1212 \\ (0.0030)} & \makecell{0.1335 \\ (0.0028)} & \makecell{0.1464 \\ (0.0035)} & \makecell{0.1600 \\ (0.0037)} & \makecell{0.1741 \\ (0.0040)} & \makecell{0.1887 \\ (0.0044)} & \makecell{0.2036 \\ (0.0039)} & \makecell{0.2187 \\ (0.0040)} & \makecell{0.2342 \\ (0.0043)} & \makecell{0.3103 \\ (0.0034)} & \makecell{0.3685 \\ (0.0035)} \\
 
 \midrule
 \makecell[l]{PrivUnitG+PA} 
 & \makecell{0.1213 \\ (0.0041)} & \makecell{0.1456 \\ (0.0049)} & \makecell{0.1726 \\ (0.0050)} & \makecell{0.2020 \\ (0.0047)} & \makecell{0.2349 \\ (0.0044)} & \makecell{0.2684 \\ (0.0043)} & \makecell{0.3029 \\ (0.0041)} & \makecell{0.3367 \\ (0.0043)} & \makecell{0.3700 \\ (0.0044)} & \makecell{0.4016 \\ (0.0046)} & \makecell{0.5364 \\ (0.0044)} & \makecell{0.6254 \\ (0.0045)} \\
 
 \midrule
 \makecell[l]{CW+PA w/o bounding} 
 & \makecell{0.1039 \\ (0.0024)} & \makecell{0.1077 \\ (0.0025)} & \makecell{0.1120 \\ (0.0026)} & \makecell{0.1161 \\ (0.0026)} & \makecell{0.1205 \\ (0.0026)} & \makecell{0.1251 \\ (0.0027)} & \makecell{0.1298 \\ (0.0026)} & \makecell{0.1346 \\ (0.0025)} & \makecell{0.1397 \\ (0.0026)} & \makecell{0.1446 \\ (0.0028)} & \makecell{0.1716 \\ (0.0027)} & \makecell{0.2008 \\ (0.0026)} \\
 
 \bottomrule
\end{tabular}
}
\end{table}

\begin{table}[h]
\caption{Main results for $(m=64)$ \textbf{Nonlinear MLP Classifier with GELU} on \textbf{CIFAR-10} under $\epsilon$-LDP mechanisms. We evaluate utility by measuring the Accuracy. Higher accuracy indicates higher utility.}
\centering
\resizebox{\textwidth}{!}{
\begin{tabular}{l cccccccccccc}
\toprule
& \multicolumn{12}{c}{Privacy Budget ($\epsilon$)} \\
 \cmidrule(l){2-13}
 \textbf{Mechanism} & 0.5 & 1.0 & 1.5 & 2.0 & 2.5 & 3.0 & 3.5 & 4.0 & 4.5 & 5.0 & 7.5 & 10.0 \\
 \midrule
 {No randomization} &
0.9012 & 0.9012 & 0.9012 & 0.9012 & 0.9012 & 0.9012 & 0.9012 & 0.9012 & 0.9012 & 0.9012 & 0.9012 & 0.9012 \\
 \midrule
 \makecell[l]{Laplace ($\ell_1$)} &
\makecell{0.1015 \\ (0.0033)} & \makecell{0.1021 \\ (0.0034)} & \makecell{0.1027 \\ (0.0033)} & \makecell{0.1033 \\ (0.0032)} & \makecell{0.1039 \\ (0.0033)} & \makecell{0.1044 \\ (0.0033)} & \makecell{0.1050 \\ (0.0032)} & \makecell{0.1055 \\ (0.0032)} & \makecell{0.1061 \\ (0.0031)} & \makecell{0.1068 \\ (0.0032)} & \makecell{0.1100 \\ (0.0033)} & \makecell{0.1134 \\ (0.0036)} \\
 \midrule
 \makecell[l]{Laplace+PA} &
\makecell{0.1064 \\ (0.0032)} & \makecell{0.1114 \\ (0.0034)} & \makecell{0.1166 \\ (0.0032)} & \makecell{0.1221 \\ (0.0033)} & \makecell{0.1275 \\ (0.0036)} & \makecell{0.1336 \\ (0.0035)} & \makecell{0.1400 \\ (0.0038)} & \makecell{0.1464 \\ (0.0036)} & \makecell{0.1532 \\ (0.0042)} & \makecell{0.1598 \\ (0.0041)} & \makecell{0.1968 \\ (0.0041)} & \makecell{0.2380 \\ (0.0042)} \\
 \midrule
 \makecell[l]{PrivUnit2} &
\makecell{0.1055 \\ (0.0028)} & \makecell{0.1094 \\ (0.0029)} & \makecell{0.1144 \\ (0.0031)} & \makecell{0.1191 \\ (0.0032)} & \makecell{0.1237 \\ (0.0033)} & \makecell{0.1286 \\ (0.0034)} & \makecell{0.1336 \\ (0.0036)} & \makecell{0.1384 \\ (0.0041)} & \makecell{0.1436 \\ (0.0044)} & \makecell{0.1492 \\ (0.0052)} & \makecell{0.1723 \\ (0.0038)} & \makecell{0.1987 \\ (0.0034)} \\
 \midrule
 \makecell[l]{PrivUnit2+PA} &
\makecell{0.1172 \\ (0.0032)} & \makecell{0.1375 \\ (0.0027)} & \makecell{0.1606 \\ (0.0029)} & \makecell{0.1852 \\ (0.0030)} & \makecell{0.2115 \\ (0.0032)} & \makecell{0.2400 \\ (0.0035)} & \makecell{0.2692 \\ (0.0034)} & \makecell{0.2985 \\ (0.0037)} & \makecell{0.3276 \\ (0.0048)} & \makecell{0.3565 \\ (0.0054)} & \makecell{0.4834 \\ (0.0052)} & \makecell{0.5766 \\ (0.0053)} \\
 \midrule
 \makecell[l]{PrivUnitG} &
\makecell{0.1044 \\ (0.0029)} & \makecell{0.1087 \\ (0.0030)} & \makecell{0.1133 \\ (0.0031)} & \makecell{0.1181 \\ (0.0034)} & \makecell{0.1229 \\ (0.0030)} & \makecell{0.1278 \\ (0.0031)} & \makecell{0.1325 \\ (0.0030)} & \makecell{0.1376 \\ (0.0033)} & \makecell{0.1424 \\ (0.0033)} & \makecell{0.1471 \\ (0.0035)} & \makecell{0.1709 \\ (0.0036)} & \makecell{0.1900 \\ (0.0039)} \\
 \midrule
 \makecell[l]{PrivUnitG+PA} &
\makecell{0.1181 \\ (0.0027)} & \makecell{0.1384 \\ (0.0027)} & \makecell{0.1612 \\ (0.0029)} & \makecell{0.1854 \\ (0.0033)} & \makecell{0.2122 \\ (0.0036)} & \makecell{0.2399 \\ (0.0040)} & \makecell{0.2682 \\ (0.0038)} & \makecell{0.2968 \\ (0.0038)} & \makecell{0.3255 \\ (0.0043)} & \makecell{0.3523 \\ (0.0052)} & \makecell{0.4730 \\ (0.0044)} & \makecell{0.5591 \\ (0.0051)} \\
 \midrule
 \makecell[l]{CW+PA w/o bounding} &
\makecell{0.1045 \\ (0.0032)} & \makecell{0.1076 \\ (0.0031)} & \makecell{0.1109 \\ (0.0032)} & \makecell{0.1142 \\ (0.0031)} & \makecell{0.1176 \\ (0.0033)} & \makecell{0.1211 \\ (0.0033)} & \makecell{0.1248 \\ (0.0033)} & \makecell{0.1283 \\ (0.0035)} & \makecell{0.1321 \\ (0.0033)} & \makecell{0.1357 \\ (0.0035)} & \makecell{0.1567 \\ (0.0039)} & \makecell{0.1801 \\ (0.0046)} \\
 \bottomrule
\end{tabular}
}
\end{table}

\begin{table}[h]
\caption{Main results for $(m=64)$ \textbf{Nonlinear MLP Classifier with Tanh} on \textbf{CIFAR-10} under $\epsilon$-LDP mechanisms. We evaluate utility by measuring the Accuracy. Higher accuracy indicates higher utility.}
\centering
\resizebox{\textwidth}{!}{
\begin{tabular}{l cccccccccccc}
\toprule
& \multicolumn{12}{c}{Privacy Budget ($\epsilon$)} \\
 \cmidrule(l){2-13}
 \textbf{Mechanism} & 0.5 & 1.0 & 1.5 & 2.0 & 2.5 & 3.0 & 3.5 & 4.0 & 4.5 & 5.0 & 7.5 & 10.0 \\
 \midrule
 {No randomization} &
0.9014 & 0.9014 & 0.9014 & 0.9014 & 0.9014 & 0.9014 & 0.9014 & 0.9014 & 0.9014 & 0.9014 & 0.9014 & 0.9014 \\
 \midrule
 \makecell[l]{Laplace ($\ell_1$)} &
\makecell{0.1013 \\ (0.0022)} & \makecell{0.1025 \\ (0.0022)} & \makecell{0.1035 \\ (0.0023)} & \makecell{0.1046 \\ (0.0022)} & \makecell{0.1057 \\ (0.0022)} & \makecell{0.1069 \\ (0.0023)} & \makecell{0.1079 \\ (0.0022)} & \makecell{0.1091 \\ (0.0022)} & \makecell{0.1102 \\ (0.0021)} & \makecell{0.1113 \\ (0.0021)} & \makecell{0.1171 \\ (0.0021)} & \makecell{0.1237 \\ (0.0025)} \\
 \midrule
 \makecell[l]{Laplace+PA} &
\makecell{0.1064 \\ (0.0030)} & \makecell{0.1133 \\ (0.0031)} & \makecell{0.1201 \\ (0.0033)} & \makecell{0.1278 \\ (0.0033)} & \makecell{0.1353 \\ (0.0035)} & \makecell{0.1437 \\ (0.0038)} & \makecell{0.1524 \\ (0.0038)} & \makecell{0.1615 \\ (0.0045)} & \makecell{0.1714 \\ (0.0045)} & \makecell{0.1816 \\ (0.0047)} & \makecell{0.2374 \\ (0.0044)} & \makecell{0.3008 \\ (0.0047)} \\
 \midrule
 \makecell[l]{PrivUnit2} &
\makecell{0.1085 \\ (0.0031)} & \makecell{0.1176 \\ (0.0030)} & \makecell{0.1270 \\ (0.0025)} & \makecell{0.1367 \\ (0.0028)} & \makecell{0.1472 \\ (0.0028)} & \makecell{0.1576 \\ (0.0032)} & \makecell{0.1682 \\ (0.0033)} & \makecell{0.1789 \\ (0.0027)} & \makecell{0.1909 \\ (0.0028)} & \makecell{0.2026 \\ (0.0032)} & \makecell{0.2607 \\ (0.0048)} & \makecell{0.3219 \\ (0.0056)} \\
 \midrule
 \makecell[l]{PrivUnit2+PA} &
\makecell{0.1215 \\ (0.0032)} & \makecell{0.1491 \\ (0.0038)} & \makecell{0.1829 \\ (0.0038)} & \makecell{0.2217 \\ (0.0045)} & \makecell{0.2646 \\ (0.0048)} & \makecell{0.3106 \\ (0.0036)} & \makecell{0.3587 \\ (0.0034)} & \makecell{0.4059 \\ (0.0040)} & \makecell{0.4532 \\ (0.0051)} & \makecell{0.4973 \\ (0.0051)} & \makecell{0.6615 \\ (0.0042)} & \makecell{0.7488 \\ (0.0037)} \\
 \midrule
 \makecell[l]{PrivUnitG} &
\makecell{0.1086 \\ (0.0027)} & \makecell{0.1172 \\ (0.0031)} & \makecell{0.1264 \\ (0.0034)} & \makecell{0.1362 \\ (0.0034)} & \makecell{0.1462 \\ (0.0035)} & \makecell{0.1569 \\ (0.0032)} & \makecell{0.1671 \\ (0.0034)} & \makecell{0.1780 \\ (0.0034)} & \makecell{0.1887 \\ (0.0035)} & \makecell{0.1997 \\ (0.0040)} & \makecell{0.2537 \\ (0.0043)} & \makecell{0.2985 \\ (0.0037)} \\
 \midrule
 \makecell[l]{PrivUnitG+PA} &
\makecell{0.1240 \\ (0.0027)} & \makecell{0.1510 \\ (0.0033)} & \makecell{0.1837 \\ (0.0035)} & \makecell{0.2216 \\ (0.0051)} & \makecell{0.2637 \\ (0.0051)} & \makecell{0.3090 \\ (0.0054)} & \makecell{0.3562 \\ (0.0057)} & \makecell{0.4025 \\ (0.0054)} & \makecell{0.4476 \\ (0.0051)} & \makecell{0.4894 \\ (0.0047)} & \makecell{0.6495 \\ (0.0049)} & \makecell{0.7423 \\ (0.0041)} \\
 \midrule
 \makecell[l]{CW+PA w/o bounding} &
\makecell{0.1039 \\ (0.0029)} & \makecell{0.1080 \\ (0.0029)} & \makecell{0.1124 \\ (0.0029)} & \makecell{0.1168 \\ (0.0031)} & \makecell{0.1217 \\ (0.0032)} & \makecell{0.1263 \\ (0.0032)} & \makecell{0.1316 \\ (0.0032)} & \makecell{0.1367 \\ (0.0032)} & \makecell{0.1419 \\ (0.0032)} & \makecell{0.1473 \\ (0.0033)} & \makecell{0.1773 \\ (0.0039)} & \makecell{0.2119 \\ (0.0044)} \\
 \bottomrule
\end{tabular}
}
\end{table}

\begin{table}[h]
\caption{Main results for \textbf{Nonlinear MLP Classifier with ReLU} on \textbf{MNIST (VAE)} $(m=32)$ under $\epsilon$-LDP mechanisms. We evaluate utility by measuring the classification accuracy using ground-truth labels. Higher accuracy indicates higher utility.}
\centering
\resizebox{\textwidth}{!}{
\begin{tabular}{l cccccccccccc}
\toprule
& \multicolumn{12}{c}{Privacy Budget ($\epsilon$)} \\
 \cmidrule(l){2-13}
 \textbf{Mechanism} & 0.5 & 1.0 & 1.5 & 2.0 & 2.5 & 3.0 & 3.5 & 4.0 & 4.5 & 5.0 & 7.5 & 10.0 \\
 \midrule
 \makecell[l]{No randomization} &
\makecell{0.9829 \\ (0.0000)} & \makecell{0.9829 \\ (0.0000)} & \makecell{0.9829 \\ (0.0000)} & \makecell{0.9829 \\ (0.0000)} & \makecell{0.9829 \\ (0.0000)} & \makecell{0.9829 \\ (0.0000)} & \makecell{0.9829 \\ (0.0000)} & \makecell{0.9829 \\ (0.0000)} & \makecell{0.9829 \\ (0.0000)} & \makecell{0.9829 \\ (0.0000)} & \makecell{0.9829 \\ (0.0000)} & \makecell{0.9829 \\ (0.0000)} \\
 \midrule
 \makecell[l]{Laplace ($\ell_1$)} &
\makecell{0.1050 \\ (0.0032)} & \makecell{0.1078 \\ (0.0034)} & \makecell{0.1108 \\ (0.0034)} & \makecell{0.1139 \\ (0.0035)} & \makecell{0.1171 \\ (0.0034)} & \makecell{0.1202 \\ (0.0035)} & \makecell{0.1234 \\ (0.0037)} & \makecell{0.1267 \\ (0.0039)} & \makecell{0.1302 \\ (0.0040)} & \makecell{0.1336 \\ (0.0039)} & \makecell{0.1524 \\ (0.0039)} & \makecell{0.1732 \\ (0.0041)} \\
 \midrule
 \makecell[l]{Laplace+PA} &
\makecell{0.1092 \\ (0.0030)} & \makecell{0.1178 \\ (0.0032)} & \makecell{0.1271 \\ (0.0029)} & \makecell{0.1369 \\ (0.0029)} & \makecell{0.1474 \\ (0.0032)} & \makecell{0.1584 \\ (0.0035)} & \makecell{0.1704 \\ (0.0037)} & \makecell{0.1828 \\ (0.0040)} & \makecell{0.1956 \\ (0.0039)} & \makecell{0.2091 \\ (0.0037)} & \makecell{0.2846 \\ (0.0037)} & \makecell{0.3666 \\ (0.0041)} \\
 \midrule
 \makecell[l]{PrivUnit2} &
\makecell{0.1152 \\ (0.0044)} & \makecell{0.1307 \\ (0.0043)} & \makecell{0.1474 \\ (0.0041)} & \makecell{0.1659 \\ (0.0047)} & \makecell{0.1855 \\ (0.0044)} & \makecell{0.2063 \\ (0.0046)} & \makecell{0.2289 \\ (0.0046)} & \makecell{0.2526 \\ (0.0051)} & \makecell{0.2767 \\ (0.0054)} & \makecell{0.3021 \\ (0.0054)} & \makecell{0.4360 \\ (0.0047)} & \makecell{0.5729 \\ (0.0027)} \\
 \midrule
 \makecell[l]{PrivUnit2+PA} &
\makecell{0.1250 \\ (0.0025)} & \makecell{0.1536 \\ (0.0032)} & \makecell{0.1875 \\ (0.0038)} & \makecell{0.2263 \\ (0.0038)} & \makecell{0.2696 \\ (0.0040)} & \makecell{0.3174 \\ (0.0036)} & \makecell{0.3674 \\ (0.0039)} & \makecell{0.4203 \\ (0.0038)} & \makecell{0.4720 \\ (0.0041)} & \makecell{0.5214 \\ (0.0043)} & \makecell{0.7252 \\ (0.0036)} & \makecell{0.8427 \\ (0.0032)} \\
 \midrule
 \makecell[l]{PrivUnitG} &
\makecell{0.1139 \\ (0.0034)} & \makecell{0.1290 \\ (0.0034)} & \makecell{0.1462 \\ (0.0037)} & \makecell{0.1640 \\ (0.0032)} & \makecell{0.1822 \\ (0.0029)} & \makecell{0.2014 \\ (0.0028)} & \makecell{0.2217 \\ (0.0031)} & \makecell{0.2424 \\ (0.0035)} & \makecell{0.2632 \\ (0.0038)} & \makecell{0.2838 \\ (0.0043)} & \makecell{0.3862 \\ (0.0054)} & \makecell{0.4682 \\ (0.0053)} \\
 \midrule
 \makecell[l]{PrivUnitG+PA} &
\makecell{0.1242 \\ (0.0027)} & \makecell{0.1528 \\ (0.0025)} & \makecell{0.1862 \\ (0.0035)} & \makecell{0.2236 \\ (0.0040)} & \makecell{0.2646 \\ (0.0044)} & \makecell{0.3083 \\ (0.0047)} & \makecell{0.3538 \\ (0.0051)} & \makecell{0.3989 \\ (0.0059)} & \makecell{0.4432 \\ (0.0066)} & \makecell{0.4861 \\ (0.0057)} & \makecell{0.6578 \\ (0.0042)} & \makecell{0.7647 \\ (0.0037)} \\
 \midrule
 \makecell[l]{CW+PA w/o bounding} &
\makecell{0.1069 \\ (0.0027)} & \makecell{0.1127 \\ (0.0027)} & \makecell{0.1191 \\ (0.0025)} & \makecell{0.1253 \\ (0.0027)} & \makecell{0.1323 \\ (0.0028)} & \makecell{0.1392 \\ (0.0029)} & \makecell{0.1466 \\ (0.0031)} & \makecell{0.1542 \\ (0.0032)} & \makecell{0.1623 \\ (0.0032)} & \makecell{0.1706 \\ (0.0033)} & \makecell{0.2173 \\ (0.0038)} & \makecell{0.2685 \\ (0.0043)} \\
 \bottomrule
\end{tabular}
}
\end{table}

\clearpage

\subsubsection{Supplemental Results: Sensitivity Analysis of the Finite Scale Cap} \label{appdx:finite_cap}

We provide our sensitivity analysis on $\lambda_{\mathrm{max}}$ using CIFAR-10-C (Brightness, Severity Level 5). In the tables, PA denotes the proposed approach.

To provide further clarity, we conducted an additional sensitivity analysis of the finite cap $\lambda_{\max}$ used for near-null directions. For Laplace+PA at $\epsilon=10$, the only visibly lower mean occurs at $\lambda_{\max}=10$. For $\lambda_{\max}\ge 10^2$, accuracy ranges from $0.2376$ to $0.2433$, a difference of $0.0057$, while the standard deviations across 20 seeds range from approximately $0.0035$ to $0.0047$. PrivUnit2+PA and PrivUnitG+PA likewise exhibit little variation across the tested range.

\begin{table}[h]
\caption{Sensitivity analysis on the $\lambda_{\mathrm{max}}$ for \textbf{Laplace+PA}, \textbf{PrivUnit2+PA}, and \textbf{PrivUnitG+PA} on \textbf{CIFAR-10-C (Brightness Corruption At The Highest Severity Level 5)} Test Dataset under $\epsilon$-LDP mechanisms. We evaluate utility by measuring the Accuracy.}
\centering
\resizebox{\textwidth}{!}{
\begin{tabular}{ll cccccccccccc}
\toprule
& \multicolumn{12}{c}{Privacy Budget ($\epsilon$)} \\
\cmidrule(l){3-14}
\textbf{Mechanism} & $\lambda_{\mathrm{max}}$ & 0.5 & 1.0 & 1.5 & 2.0 & 2.5 & 3.0 & 3.5 & 4.0 & 4.5 & 5.0 & 7.5 & 10.0 \\
\midrule
\multirow[t]{5}{*}{\makecell{Laplace+PA}} 
& $10^1$ & \makecell{0.1044 \\ (0.0038)} & \makecell{0.1091 \\ (0.0039)} & \makecell{0.1141 \\ (0.0039)} & \makecell{0.1195 \\ (0.0036)} & \makecell{0.1251 \\ (0.0035)} & \makecell{0.1308 \\ (0.0035)} & \makecell{0.1365 \\ (0.0035)} & \makecell{0.1424 \\ (0.0036)} & \makecell{0.1485 \\ (0.0037)} & \makecell{0.1552 \\ (0.0040)} & \makecell{0.1891 \\ (0.0038)} & \makecell{0.2264 \\ (0.0045)} \\
& $10^2$ & \makecell{0.1045 \\ (0.0022)} & \makecell{0.1095 \\ (0.0021)} & \makecell{0.1152 \\ (0.0022)} & \makecell{0.1208 \\ (0.0018)} & \makecell{0.1267 \\ (0.0019)} & \makecell{0.1327 \\ (0.0024)} & \makecell{0.1390 \\ (0.0025)} & \makecell{0.1454 \\ (0.0027)} & \makecell{0.1517 \\ (0.0027)} & \makecell{0.1589 \\ (0.0027)} & \makecell{0.1963 \\ (0.0036)} & \makecell{0.2376 \\ (0.0035)} \\
& $10^3$ & \makecell{0.1052 \\ (0.0031)} & \makecell{0.1104 \\ (0.0031)} & \makecell{0.1160 \\ (0.0031)} & \makecell{0.1220 \\ (0.0032)} & \makecell{0.1278 \\ (0.0034)} & \makecell{0.1343 \\ (0.0036)} & \makecell{0.1406 \\ (0.0038)} & \makecell{0.1470 \\ (0.0040)} & \makecell{0.1538 \\ (0.0041)} & \makecell{0.1609 \\ (0.0043)} & \makecell{0.1987 \\ (0.0045)} & \makecell{0.2401 \\ (0.0047)} \\
& $10^4$ & \makecell{0.1049 \\ (0.0032)} & \makecell{0.1101 \\ (0.0035)} & \makecell{0.1157 \\ (0.0034)} & \makecell{0.1216 \\ (0.0036)} & \makecell{0.1277 \\ (0.0036)} & \makecell{0.1341 \\ (0.0034)} & \makecell{0.1404 \\ (0.0033)} & \makecell{0.1473 \\ (0.0033)} & \makecell{0.1541 \\ (0.0034)} & \makecell{0.1612 \\ (0.0036)} & \makecell{0.2002 \\ (0.0034)} & \makecell{0.2424 \\ (0.0037)} \\
& $10^5$ & \makecell{0.1056 \\ (0.0016)} & \makecell{0.1106 \\ (0.0018)} & \makecell{0.1163 \\ (0.0021)} & \makecell{0.1222 \\ (0.0022)} & \makecell{0.1283 \\ (0.0020)} & \makecell{0.1342 \\ (0.0022)} & \makecell{0.1405 \\ (0.0028)} & \makecell{0.1476 \\ (0.0031)} & \makecell{0.1543 \\ (0.0028)} & \makecell{0.1608 \\ (0.0031)} & \makecell{0.1990 \\ (0.0035)} & \makecell{0.2409 \\ (0.0037)} \\

\midrule
\multirow[t]{5}{*}{\makecell{PrivUnit2+PA}} 
& $10^1$ & \makecell{0.1189 \\ (0.0036)} & \makecell{0.1404 \\ (0.0035)} & \makecell{0.1653 \\ (0.0038)} & \makecell{0.1922 \\ (0.0036)} & \makecell{0.2218 \\ (0.0034)} & \makecell{0.2522 \\ (0.0039)} & \makecell{0.2832 \\ (0.0040)} & \makecell{0.3146 \\ (0.0046)} & \makecell{0.3438 \\ (0.0047)} & \makecell{0.3733 \\ (0.0042)} & \makecell{0.4983 \\ (0.0051)} & \makecell{0.5851 \\ (0.0053)} \\
& $10^2$ & \makecell{0.1187 \\ (0.0031)} & \makecell{0.1408 \\ (0.0034)} & \makecell{0.1663 \\ (0.0037)} & \makecell{0.1940 \\ (0.0038)} & \makecell{0.2236 \\ (0.0036)} & \makecell{0.2537 \\ (0.0040)} & \makecell{0.2846 \\ (0.0043)} & \makecell{0.3160 \\ (0.0047)} & \makecell{0.3459 \\ (0.0047)} & \makecell{0.3745 \\ (0.0051)} & \makecell{0.4989 \\ (0.0051)} & \makecell{0.5855 \\ (0.0053)} \\
& $10^3$ & \makecell{0.1187 \\ (0.0025)} & \makecell{0.1409 \\ (0.0028)} & \makecell{0.1660 \\ (0.0034)} & \makecell{0.1943 \\ (0.0032)} & \makecell{0.2241 \\ (0.0036)} & \makecell{0.2549 \\ (0.0033)} & \makecell{0.2865 \\ (0.0034)} & \makecell{0.3177 \\ (0.0028)} & \makecell{0.3483 \\ (0.0037)} & \makecell{0.3763 \\ (0.0038)} & \makecell{0.5012 \\ (0.0043)} & \makecell{0.5851 \\ (0.0062)} \\
& $10^4$ & \makecell{0.1205 \\ (0.0032)} & \makecell{0.1423 \\ (0.0029)} & \makecell{0.1680 \\ (0.0027)} & \makecell{0.1953 \\ (0.0030)} & \makecell{0.2246 \\ (0.0037)} & \makecell{0.2548 \\ (0.0036)} & \makecell{0.2857 \\ (0.0038)} & \makecell{0.3161 \\ (0.0030)} & \makecell{0.3457 \\ (0.0027)} & \makecell{0.3745 \\ (0.0029)} & \makecell{0.4986 \\ (0.0057)} & \makecell{0.5852 \\ (0.0043)} \\
& $10^5$ & \makecell{0.1189 \\ (0.0041)} & \makecell{0.1416 \\ (0.0038)} & \makecell{0.1671 \\ (0.0037)} & \makecell{0.1946 \\ (0.0038)} & \makecell{0.2242 \\ (0.0044)} & \makecell{0.2541 \\ (0.0049)} & \makecell{0.2852 \\ (0.0053)} & \makecell{0.3159 \\ (0.0054)} & \makecell{0.3464 \\ (0.0054)} & \makecell{0.3755 \\ (0.0055)} & \makecell{0.4975 \\ (0.0047)} & \makecell{0.5845 \\ (0.0045)} \\

\midrule
\multirow[t]{5}{*}{\makecell{PrivUnitG+PA}} 
& $10^1$ & \makecell{0.1195 \\ (0.0033)} & \makecell{0.1415 \\ (0.0031)} & \makecell{0.1665 \\ (0.0035)} & \makecell{0.1940 \\ (0.0033)} & \makecell{0.2234 \\ (0.0035)} & \makecell{0.2538 \\ (0.0042)} & \makecell{0.2842 \\ (0.0036)} & \makecell{0.3142 \\ (0.0035)} & \makecell{0.3430 \\ (0.0033)} & \makecell{0.3716 \\ (0.0032)} & \makecell{0.4889 \\ (0.0050)} & \makecell{0.5658 \\ (0.0052)} \\
& $10^2$ & \makecell{0.1198 \\ (0.0031)} & \makecell{0.1414 \\ (0.0032)} & \makecell{0.1665 \\ (0.0038)} & \makecell{0.1931 \\ (0.0034)} & \makecell{0.2215 \\ (0.0038)} & \makecell{0.2514 \\ (0.0040)} & \makecell{0.2814 \\ (0.0037)} & \makecell{0.3117 \\ (0.0038)} & \makecell{0.3410 \\ (0.0039)} & \makecell{0.3688 \\ (0.0032)} & \makecell{0.4874 \\ (0.0046)} & \makecell{0.5652 \\ (0.0041)} \\
& $10^3$ & \makecell{0.1192 \\ (0.0031)} & \makecell{0.1411 \\ (0.0032)} & \makecell{0.1666 \\ (0.0032)} & \makecell{0.1935 \\ (0.0030)} & \makecell{0.2221 \\ (0.0027)} & \makecell{0.2522 \\ (0.0030)} & \makecell{0.2822 \\ (0.0038)} & \makecell{0.3118 \\ (0.0044)} & \makecell{0.3410 \\ (0.0045)} & \makecell{0.3690 \\ (0.0045)} & \makecell{0.4882 \\ (0.0040)} & \makecell{0.5667 \\ (0.0043)} \\
& $10^4$ & \makecell{0.1205 \\ (0.0024)} & \makecell{0.1426 \\ (0.0029)} & \makecell{0.1673 \\ (0.0025)} & \makecell{0.1942 \\ (0.0029)} & \makecell{0.2226 \\ (0.0030)} & \makecell{0.2525 \\ (0.0031)} & \makecell{0.2832 \\ (0.0041)} & \makecell{0.3122 \\ (0.0044)} & \makecell{0.3415 \\ (0.0043)} & \makecell{0.3692 \\ (0.0048)} & \makecell{0.4877 \\ (0.0063)} & \makecell{0.5656 \\ (0.0058)} \\
& $10^5$ & \makecell{0.1198 \\ (0.0022)} & \makecell{0.1414 \\ (0.0026)} & \makecell{0.1662 \\ (0.0032)} & \makecell{0.1932 \\ (0.0027)} & \makecell{0.2216 \\ (0.0027)} & \makecell{0.2514 \\ (0.0035)} & \makecell{0.2819 \\ (0.0039)} & \makecell{0.3115 \\ (0.0042)} & \makecell{0.3413 \\ (0.0042)} & \makecell{0.3692 \\ (0.0043)} & \makecell{0.4866 \\ (0.0032)} & \makecell{0.5644 \\ (0.0036)} \\
\bottomrule
\end{tabular}
}
\end{table}

\clearpage



\end{document}